\documentclass[sigconf]{acmart}

\AtBeginDocument{%
	\providecommand\BibTeX{{%
			\normalfont B\kern-0.5em{\scshape i\kern-0.25em b}\kern-0.8em\TeX}}}

\setcopyright{acmcopyright}
\copyrightyear{2018}
\acmYear{2018}
\acmDOI{XXXXXXX.XXXXXXX}

\acmConference[MobiHoc ’22]{Make sure to enter the correct
	conference title from your rights confirmation emai}{October 17--22,
	2022}{Seoul, South Korea}
\acmPrice{15.00}
\acmISBN{978-1-4503-XXXX-X/18/06}

\usepackage{graphicx}
\usepackage{textcomp}
\usepackage{amsmath}

\usepackage{pifont}

\newcommand{\cmark}{\ding{51}}%
\newcommand{\xmark}{\ding{55}}%

\usepackage{xcolor, colortbl} 
\usepackage{natbib}
\usepackage[utf8]{inputenc} 
\usepackage[T1]{fontenc}    
\usepackage{hyperref}       
\usepackage{booktabs}       
\usepackage{nicefrac}       
\usepackage{microtype}      
\usepackage{verbatim}
\usepackage{algorithm}
\usepackage{algpseudocode}
\usepackage{amsmath,amsthm}
\usepackage{subfigure}
\usepackage{caption}
\usepackage{subfigure}
\usepackage{bm}
\usepackage{verbatim}
\usepackage{amsmath,bm}
\newtheorem{remark}{Remark}
\usepackage{float}
\usepackage{wrapfig}
\usepackage{soul}

\newcommand{\algname}{\textsf{INTERACT }}
\newcommand{\algnamens}{\textsf{INTERACT}}
\newcommand{\algtname}{\textsf{SVR-INTERACT }}
\newcommand{\algtnamens}{\textsf{SVR-INTERACT}}

\allowdisplaybreaks

\renewcommand{\algorithmicrequire}{}
\renewcommand{\algorithmicensure}{\textbf{Output:}}

\newcommand{\bx}{\bar{\x}}
\newcommand{\by}{\bar{\y}}

\renewcommand{\d}{\mathbf{d}}
\newcommand{\bd}{\bar{\d}}

\newcommand{\Eb}{\mathbb{E}}

\newcommand{\I}{\mathbf{I}}

\newcommand{\Lc}{\mathcal{L}}

\newcommand{\M}{\mathbf{M}}

\newcommand{\Mt}{\widetilde{\mathbf{M}}}

\newcommand{\Nc}{\mathcal{N}}

\newcommand{\p}{\mathbf{p}}
\newcommand{\bp}{\bar{\p}}

\newcommand{\Rb}{\mathbb{R}}

\renewcommand{\u}{\mathbf{u}}

\renewcommand{\v}{\mathbf{v}}
\newcommand{\bu}{\bar{\mathbf{u}}}
\newcommand{\W}{\mathbf{W}}

\newcommand{\x}{\mathbf{x}}

\newcommand{\y}{\mathbf{y}}

\newcommand{\1}{\mathbf{1}}
\newcommand{\ot}{\1\otimes}

\newtheorem{thm}{Theorem}

\newtheorem{cor}[thm]{Corollary}
\newtheorem{lem}{Lemma}

\newtheorem{defn}{Definition}

\newtheorem{assum}{Assumption}

\renewcommand{\algorithmicrequire}{\text{}}
\renewcommand{\algorithmicensure}{\textbf{Output:}}

\begin{document}


\title[INTERACT: Low Sample and Communication Complexities in Decentralized Bilevel Learning over Networks]{INTERACT: Achieving Low Sample and Communication Complexities in Decentralized Bilevel Learning over Networks}

%
%
%
%
%
%

\author{Zhuqing Liu$^1$, Xin Zhang$^{2}$, Prashant Khanduri$^{1,3}$, Songtao Lu$^{4}$, and Jia Liu$^{1}$
}
\affiliation{
	\institution{$^1$Department of Electrical and Computer Engineering, The Ohio State University}
	\institution{$^2$Department of Statistics, Iowa State University}
	\institution{$^3$Department of Electrical and Computer Engineering, University of Minnesota}
	\institution{$^4$IBM Research AI, IBM Thomas J. Watson Research Center}				
	\country{ }
}

\renewcommand{\shortauthors}{Z. Liu, X. Zhang, P. Khanduri, S. Lu and J. Liu}

\begin{abstract}
In recent years, decentralized bilevel optimization problems have received increasing attention in the networking and machine learning communities thanks to their versatility in modeling decentralized learning problems over peer-to-peer networks (e.g., multi-agent meta-learning, multi-agent reinforcement learning, personalized training, and Byzantine-resilient learning).
However, for decentralized bilevel optimization over peer-to-peer networks with limited computation and communication capabilities, how to achieve low sample and communication complexities are two fundamental challenges that remain under-explored so far.
In this paper, we make the first attempt to investigate the class of decentralized bilevel optimization problems with nonconvex and strongly-convex structure corresponding to the outer and inner subproblems, respectively.
Our main contributions in this paper are two-fold:
i) We first propose a deterministic algorithm called \algname (\ul{in}ner-gradien\ul{t}-d\ul{e}scent-oute\ul{r}-tr\ul{ac}ked-gradien\ul{t}) that requires the sample complexity of $\mathcal{O}(n \epsilon^{-1})$ and communication complexity of $\mathcal{O}(\epsilon^{-1})$ to solve the bilevel optimization problem, where $n$ and $\epsilon > 0$ are the number of samples at each agent and the desired stationarity gap, respectively.
ii) To relax the need for full gradient evaluations in each iteration, we propose a stochastic variance-reduced version of \algname (\algtnamens), which improves the sample complexity to $\mathcal{O}(\sqrt{n} \epsilon^{-1})$ while achieving the same communication complexity as the deterministic algorithm. 
To our knowledge, this work is the first that achieves both low sample and communication complexities for solving decentralized bilevel optimization problems over networks.
Our numerical experiments also corroborate our theoretical findings.

\end{abstract}

\maketitle

\section{Introduction}\label{Section: introduction}

In recent years, fueled by the rise of machine learning and artificial intelligence in edge networks, decentralized bilevel optimization problems have received increasing attention in the networking and machine learning communities.
This is due to the versatility of decentralized bilevel optimization in supporting many decentralized learning paradigms over peer-to-peer networks, such as the {\em multi-agent versions} of meta learning \cite{rajeswaran2019meta,liu2021boml,rajeswaran2019meta}, hyperparameter optimization problem\cite{mackay2019self, okuno2021lp}, area under curve (AUC) problems~\cite{liu2019stochastic,qi2021stochastic}, and reinforcement learning\cite{hong2020two, zhang2020bi}.
%
To date, however, there remain many challenges and open problems in decentralized bilevel learning over peer-to-peer networks.
Two of the most fundamental challenges in decentralized bilevel optimization are how to achieve low sample and communication complexities.
The need for decentralized bilevel optimization with low sample- and communication-complexities is particularly compelling in peer-to-peer edge networks, where the nodes (mobile devices, sensors, UAVS, etc.) are typically limited in their computation and communication capabilities.

Mathematically, decentralized bilevel optimization problems share the same structure as their single-agent counterpart in that such problems contain an outer objective function that is in turn dependent on the optimal parameter values of an inner objective.
Interestingly, the single-agent version of bilevel optimization is a class of challenging optimization problems in its own right due to the inherent non-convexity in these problems, and thus having received a significant amount of attention recently.
To date, many algorithmic ideas have been proposed for single-agent bilevel optimization problems, such as double-loop iterative methods~\cite{ghadimi2018approximation,ghadimi2020single1,ji2021bilevel}, single-level approach for reformulated bilevel problems~\cite{kunapuli2008classification,moore2010bilevel}, and simultaneous upper-lower-level updates~\cite{khanduri2021MSTSA, chen2021single,khanduri2021near,guo2021randomized} (see Section~\ref{Section: related} for more detailed discussions).

Compared to conventional single-agent bilevel optimization problems, designing effective and efficient algorithms for solving decentralized bilevel stochastic optimization problems remains under-explored. 
This is in large part due to the fact that using techniques designed for single-agent bilevel optimization faces several technical challenges when utilized for decentralized bilevel optimization. 
In fact, most of aforementioned algorithmic ideas for single-agent bilevel optimization are not applicable in decentralized multi-agent version of bilevel optimization.
%
%
%
One of the fundamental reasons is that, instead of having only one optimization task in the inner subproblem, there are {\em multiple} inner tasks that need to be handled in the decentralized multi-agent setting, which renders techniques for single-agent bilevel optimization infeasible and often implies the challenges of efficiency and scalability.
%
%
%
%

To address the efficiency and scalability for solving bilevel optimization problem with multiple inner subproblems, a natural idea is to leverage the decentralized network-consensus optimization approach~\cite{zhang2021gt,yau2022docom}, where multiple agents distributively and collaboratively solve a global learning task.
However, in developing network-consensus approaches for decentralized bilevel optimization, two fundamental challenges arise:
First, 
the tightly coupled inner-outer mathematical structure, together with the decentralized nature and the non-convexity of the outer problem, makes the design and theoretical analysis of the algorithms far more challenging compared to solving conventional single-level minimization problems. 
%
Second, due to the potential low-speed network connections and the large variability of stochastic gradient-type information in geographically-dispersed edge networks, network-consensus approaches for decentralized bilevel optimization are very sensitive to communication and sample complexities. 
These challenges motivate us to design effective and efficient network-consensus-based algorithms for decentralized bilevel learning, which strike a good balance between sample  and communication complexities.

The major contribution of this paper is that we propose a series of new algorithmic techniques to overcome the aforementioned challenges and achieve both low sample and communication complexities for decentralized bilevel optimization. Our main technical contributions are summarized below:
\begin{list}{\labelitemi}{\leftmargin=1em \itemindent=-0.09em \itemsep=.2em}
	\item 
	We propose a decentralized bilevel optimization algorithm called \algname (\ul{in}ner-gradien\ul{t}-d\ul{e}scent-oute\ul{r}-tr\ul{ac}ked-gradien\ul{t}). We show that, \algname enjoys a convergence rate of $\mathcal{O}(1/T)$ to achieve an $\epsilon$-stationary point, where $T$ is the maximum number of iterations. 
	As a consequence, \algname achieves $[\mathcal{O}(n\epsilon^{-1})$, $\mathcal{O}(\epsilon^{-1})]$ sample-communication complexity per agent, where $n$ is the total data size at each agent and $\epsilon$ is the desired accuracy.
	To our knowledge, \algname is the first algorithm for decentralized bilevel optimization with such low sample and communication complexities (please see Table~\ref{tab: Comparison}).

	\item To lower the sample complexity of \algname, we propose another enhanced version called \algtname using \underline{s}tochastic \underline{v}aiance \underline{r}eduction techniques, which further reduces the sample complexity of \algname while retaining the same communication complexity order.
	Specifically, we show that \algtname retains the same $\mathcal{O}(\epsilon^{-1})$ communication complexity as of \algname while achieving lower sample complexity of $\mathcal{O}(\sqrt{n}\epsilon^{-1})$.

	\item To examine the performance of algorithm \algname and \algtname and verify our theoretical results, we conduct experiments on meta-learning tasks with the MNIST~\cite{lecun2010mnist} and CIFAR-10~\cite{krizhevsky2009learning}. Our results show that the proposed \algname and \algtname algorithms outperform other baseline algorithms. Also, we numerically show that \algtname enjoys low sample and communication complexities.

\end{list}
The rest of the paper is organized as follows.
In Section~\ref{Section: related}, we review related literature.
%
In Section~\ref{Section: preliminary}, we provide the preliminaries of the decentralized bilevel optimization problems.
In Section~\ref{INTERACT} and Section~\ref{SVR-INTERACT}, we propose two algorithms, namely \algname and \algtname. Also, the convergence rate, communication complexity, and sample complexity of these two algorithms are provided.
%
%
%
Section~\ref{Section: experiment} provides numerical results to verify our theoretical findings, and Section \ref{Section: conclusion} concludes this paper.

\section{Related Work}\label{Section: related}

\begin{table}[t!]
	\caption{Comparisons among algorithms for bilevel optimization problems, where $\epsilon$ denotes the $\epsilon$-stationary point defined in (\ref{stationary point}) and $n$ is the size of dataset at each agent.  Fos BSA, we denote sample complexities as (Outer, Inner) as to achieve the $\epsilon$-staionary point the algorithm requires different number of outer and inner samples. Note that the rest of the algorithms require the same number of outer and inner samples to achieve $\epsilon$-stationary point. 
	}
	\label{tab}
	\begin{center}
		{\scriptsize
			\begin{tabular}{c c c  c c}
				\toprule
		{\bf	Algorithms} & {	\bf Sample  Complex.} & 	{\bf Decentralized} &	 {\bf Multi LL tasks} \\
			\midrule
						BSA \cite{ghadimi2018approximation} &  $(\mathcal{O}(\epsilon^{-2}), \mathcal{O}(\epsilon^{-3}))$ & \xmark& \xmark \\
			\midrule
				TTSA \cite{hong2020two} &  $\mathcal{O}(\epsilon^{-5/2})$ & \xmark& \xmark \\
		\midrule
				stocBiO \cite{ji2021bilevel} &  $\mathcal{O}(\epsilon^{-2})$ & \xmark& \xmark \\
		\midrule
				MSTSA \cite{khanduri2021MSTSA} &  $\mathcal{O}(\epsilon^{-2})$ & \xmark & \xmark \\
			\midrule
					STABLE \cite{chen2021single} &  $\mathcal{O}(\epsilon^{-2})$ & \xmark & \xmark \\
	\midrule
					ALSET \cite{chen2021tighter} &  $\mathcal{O}(\epsilon^{-2})$ & \xmark & \xmark \\
	\midrule
				SUSTAIN \cite{khanduri2021near}  &  {$\mathcal{O}(\epsilon^{-1.5})$} & \xmark& \xmark  \\
		\midrule
				VRBO \cite{yang2021provably}  &  {$\mathcal{O}(\epsilon^{-1.5})$} & \xmark& \xmark  \\
	\midrule
						RSVRB \cite{guo2021randomized}  &  {$\mathcal{O}(\epsilon^{-1.5})$} &  \xmark& \cmark \\
								\midrule
						ITD-BiO/AID-BiO \cite{ji2021bilevel} &  $\mathcal{O}(n \epsilon^{-1})$ & \xmark& \xmark \\
			\midrule
				\cellcolor{blue!8}	{\textbf{\algname (Ours) }} & 	\cellcolor{blue!8} {$\mathcal{O}({n}\epsilon^{-1})$} &	\cellcolor{blue!8} \cmark& 	\cellcolor{blue!8}\cmark  \\
		\midrule
			\cellcolor{blue!8}	{\textbf{\algtname (Ours) }} &	\cellcolor{blue!8} {$\mathcal{O}({\sqrt{n}}K\epsilon^{-1}+n)$} & 	\cellcolor{blue!8}\cmark& 	\cellcolor{blue!8} \cmark \\
				\bottomrule
			\end{tabular}
		}
	\end{center}
	\vspace{-.15in}
	\label{tab: Comparison}
\end{table}

\textbf{1) Bilevel Optimization Approaches:}
Classical approaches to solve bilevel optimization problems include: i) reformulating the bilevel problem to a single-level problem and replacing the inner-level problem by its KKT conditions \cite{shi2005extended}, or ii) utilizing penalty based approaches \cite{mehra2021penalty}.  
%
%
%
Motivated by various machine learning applications, gradient-based techniques have recently become the most popular strategies for solving bilevel optimization problems.
A variety of explicit gradient-based methods have been investigated (e.g., AID-based methods \cite{rajeswaran2019meta, ji2021bilevel} and ITD-based methods \cite{pedregosa2016hyperparameter, ji2021bilevel}). 
Due to large memory requirements and high sample complexities, these methods face challenges when implemented for machine learning applications, especially if the training data sizes are large. 
%
%

To overcome this challenge, a stochastic gradient descent (SGD) based double-loop algorithm called BSA \cite{ghadimi2018approximation} was recently proposed for bilevel optimization. BSA provided the first finite-time convergence guarantees for solving non-convex-strongly-convex bilevel optimization problem. 
The convergence performance of BSA was improved by stocBiO \cite{ji2021bilevel}, however at the cost of evaluating very large batch-size gradients in each iteration for both outer and inner subproblems (see Table \ref{tab: Comparison} for a detailed comparison). 
Recently, motivated by sequential games, many single-loop algorithms have been proposed for solving the bilevel problems, where the outer and inner level problem's iterates are updated simultaneously~\cite{hong2020two, khanduri2021MSTSA, chen2021single,  khanduri2021near, yang2021provably,guo2021randomized}. 
Initially, it was shown in \cite{hong2020two} that TTSA, a vanilla SGD-based algorithm that updates inner and outer level problem's iterates simultaneously, requires $\mathcal{O}(\epsilon^{-5/2})$ samples of both inner and outer-level functions to achieve $\epsilon$-stationarity. 
Later, this rate was improved by MSTSA \cite{khanduri2021MSTSA} and STABLE \cite{chen2021single} to $\mathcal{O}(\epsilon^{-2})$ by utilizing momentum-based gradient estimators for the outer objective. 
Recently, \cite{chen2021tighter} improved the analysis of TTSA and showed that vanilla SGD-based updates for outer- and inner subproblems can in fact achieve $\mathcal{O}(\epsilon^{-2})$ rate. 
%
This performance was further improved by SUSTAIN \cite{khanduri2021near} and RSVRB \cite{guo2021randomized} by utilizing momentum-based gradient estimators for both outer- and inner-level updates. 
It was shown that SUSTAIN and RSVRB match the best complexity bounds of $\mathcal{O}(\epsilon^{-1.5})$ as the optimal SGD algorithm for solving single-level optimization problems.
%
Another approach to accelerate traditional SGD based methods is the family of ``variance-reduced'' (VR) methods, which are referred to as VRBO~\cite{yang2021provably}. 
VRBO adopts recursive gradient estimators to achieve the best known sample complexity of $\mathcal{O}(\epsilon^{-1.5})$.

Orthogonal to the aforementioned approaches, which focus developing bilevel optimization algorithms only for a single agent (centralized), in this paper, we focus on a multi-agent (decentralized) setting to solve a bilevel optimization problem over a network of agents. 
To this end, we propose two novel decentralized bilevel algorithms named \algname and \algtname using gradient descent and variance reduction techniques. 
We show that \algname achieves the sample complexity of $\mathcal{O}(n \epsilon^{-1})$, while \algtname achieves the sample complexity of $\mathcal{O}(\sqrt{n}\epsilon^{-1})$. 
%
%
%
%
%
We summarize and compare the complexity results of the above state-of-the-art algorithms in Table \ref{tab}.

\smallskip
\textbf{2) Decentralized Learning over Networks:}
From a mathematical point of view, conducting decentralized learning over a network amounts to a group of agents collectively solving an optimization problem. 
Most of the existing literature of decentralized learning are focused on the standard loss minimization formulation\cite{zhang2021gt,yau2022docom}, i.e., $\min_{\x \in \mathbb{R}^{d}}f(\x)$, where $f(\cdot)$ is the objective loss function and $\x$ denotes the global model parameters to be learned, and $d$ is the model dimension.
In the literature,  a wide range of machine learning applications can be modeled by the standard decentralized loss minimization formulation (e.g., robotic network  \cite{kober2013reinforcement,polydoros2017survey}, network resource allocation \cite{jiang2018consensus,rhee2012effect}, power networks \cite{callaway2010achieving,glavic2017reinforcement}).
Some recent works, \cite{mateos2015distributed,liu2019decentralized2,liu2020decentralized,zhang2021taming} studied decentralized min-max optimization problems, i.e., $\min_{\x \in  \mathbb{R}^{d_1}} \max_{\y \in  \mathbb{R}^{d_2}} f(\x,\y)$, which are a special case (with same outer and inner level objective) of bilevel optimization problems.
%
Unfortunately, studies on solving general decentralized bilevel optimization problems are still limited. 
Existing works in \cite{baky2009fuzzy,toksari2015interactive} applied the fuzzy goal programming method for linearization of functions based on Jacobian matrix for solving multiobjective bi-level problem. 
However, none of these work tackles the decentralized bilevel optimization problem via the network consensus approach. 

%

\section{System Model and Problem Formulation}\label{Section: preliminary}

In this section, we will describe the system model and the problem formulation first.
Then, we will provide several application examples to further motivate the decentralized nonconvex-strongly-convex bilevel optimization problems studied in this paper.

\smallskip
{\bf 1) Network Consensus Formulation:}
We represent the underlying peer-to-peer communication network with $m$ agents by a graph $\mathcal{G}=(\mathcal{N}, \mathcal{L})$, where $\mathcal{N}$ and $\mathcal{L}$ are the set of agents and edges, respectively, with $|\mathcal{N}|=m$. %
Each agent can share information with its neighbors.
For agent $i \in [m]$, we denote its set of neighbors by $\Nc_i $, i.e., $\Nc_i \triangleq \{j \in \Nc,: (i,j)\in \Lc\}$.
The decentralized bilevel optimization (DBO) problem can be written as:
\begin{align} \label{problem}
	\min _{{\x_i} \in \mathbb{R}^{d_1} } \ell({\x})&=	 \frac{1}{m}\sum_{i=1}^m \ell_i({\x}_i)=\frac{1}{m}\sum_{i=1}^m f_i \left({\x}_i, {\y}_i^{*}({\x_i})\right) \notag\\
	&=\frac{1}{mn}\sum_{i=1}^m\sum_{j=1}^nf_{i} \left({\x}_i, {\y}_i^{*}({\x_i});\xi_{ij}\right), \notag\\
	\text{s.t.} ~~ {\y}_i^{*}({\x_i}) &= \arg\min_{\y_i} g_i({\x}_i, {\y}_i )= \arg\min_{\y_i}  \frac{1}{n}\sum_{j=1}^ng_{i}({\x}_i, {\y}_i ;\xi_{ij}), \notag\\
	\x_i&=\x_j, ~~ \text{if }~~ i, j \in \mathcal{L},
\end{align}
%
%
%
where $\x_i$ and $\y_i$ are the local copies of the outer and inner-level variables at agent $i \in [m]$. 
Each agent has access to a local dataset of size $n$.
The local loss is defined as: $ f_i ({\x}_i, {\y}_i^{*}({\x_i})) \coloneqq \frac{1}{n} \sum_{j = 1}^n f_i({\x}_i, {\y}_i ;\xi_{ij})$.
The equality constraint in \eqref{problem} ensures a ``consensus" among the outer level variables of each individual agents. 
%
Our goal is to solve Problem~\eqref{problem} in a {\em decentralized} manner, where $\ell(x)$ is non-convex while the inner-level funtions $g_i({\x},{\y})$ are strongly-convex with respect to $\y$ for all $i \in [m]$. In the rest of this paper, we will refer to \eqref{problem} as a decentralized  nonconvex-strong-convex bilevel problem. 

Next, we define the notion of $\epsilon$-stationary point. 
We define $\{\x_i,\y_i, \forall i \in [m]\}$ as an $\epsilon$-stationary point if it satisfies:
\begin{align}  \label{stationary point}
	&\underbrace{\frac{1}{m} \sum\nolimits_{i=1}^{m}\left\|\mathbf{x}_{i}-\overline{\mathbf{x}}\right\|^{2}}_{\substack{\mathrm{Consensus}~ \mathrm{Error}}} + \underbrace{\|{\y}^* \!-\! {{{\y}}}\|^2 }_{\substack{\mathrm{Lower-Level} \\ \mathrm{Error}}}+  \underbrace{\| \nabla \ell(\bar\x ) \|^2 }_ {\substack{\mathrm{Stationarity} \\ \mathrm{Error} }} \le \epsilon,
\end{align}
where $\bar{\x} \triangleq \frac{1}{m} \sum_{i=1}^{m} \x_{i}$, ${\y} \triangleq [\y_1^\top,...\y_m^\top]^\top$, and ${\y}^\ast \triangleq [{\y_1^\ast}^\top,...{\y_m^\ast}^\top]^\top$. 
Next, we define the notions of sample and communication complexities \cite{sun2020improving} of a decentralized algorithm to achieve an $\epsilon$-stationary point defined in \eqref{stationary point}. 
\begin{defn}[Sample Complexity]
	The sample complexity is defined as the total number of incremental first-order oracle (IFO) calls required per node for an algorithm to reach an $\epsilon$-stationary point. Note that we define one IFO call as the evaluation of the (stochastic) gradient of upper and lower level problems at node $i \in [m]$.
\end{defn}

\begin{defn}[Communication Complexity]
	We define a communication round as the sharing and receiving of local parameters by each node from all its neighboring nodes. Then, the communication complexity is defined as the total rounds of communications required by an algorithm to achieve an $\epsilon$-stationary point.
\end{defn}

{\bf 2) Motivating Application Examples:} 
Problem~\eqref{problem} can be applied to a number of interesting real-world decentralized machine learning problems.
Here, we provide two examples to further motivate its practical relevance:
\begin{list}{\labelitemi}{\leftmargin=1em \itemindent=-0.09em \itemsep=.2em}
	\item {\em Multi-agent meta-learning~\cite{rajeswaran2019meta}:} 
	Meta-learning (or learning to learn) is a powerful tool for quickly learning new tasks by using the prior experience from related tasks.
	Consider a meta-learning task with $m$ lower level problems.
	There are $m$ agents who collectively solve this meta-learning problem over a network. This problem can be formulated as:
\begin{align}
	&\min _{\mathbf{x} \in \mathbb{R}^{d_1}} \sum_{i=1}^{m} f_{i}\left(\mathbf{x},\mathbf{y}_{i}^{*}(\mathbf{x})\right) \nonumber\\
	&\mathbf{y}_{i}^{*}(\mathbf{x}) \in \arg \min _{\mathbf{y}_{i}} g_{i}\left(\mathbf{x},\mathbf{y}_{i}\right), i=1, \ldots, m,
	\label{Eq: MetaLearning}
\end{align}
where agent $i \in [m]$ has a local dataset with $n$ samples, and $f_{i}$ corresponds to the loss function for the task of agent $i$. 
Parameter $\x \in \mathbb{R}^{d_1}$ is the model parameters shared among all agents, and $\y_{i} \in \mathbb{R}^{d_2}$ denotes task-specific parameters solved by each agent. 


	\item {\em Multi-agent AUPRC optimization~\cite{guo2021randomized}:} Another application of \eqref{problem} is the multi-agent optimization of area under precision-recall curve (AUPRC) \cite{qi2021stochastic}. 
	Multi-agent AUPRC optimization problem can be formulated as a bilevel optimization problem with many lower-level problems as:
	$$
	\begin{aligned}
		&\min _{\mathbf{x} \in \mathbb{R}^{d_1}} \sum_{i=1}^{m} f_{i}\left(\mathbf{y}_{i}^{*}(\mathbf{x})\right) \\
		&\mathbf{y}_{i}^{*}(\mathbf{x})= \arg \min _{\mathbf{y}_{i}} -\y_i^\top g_{i}\left(\mathbf{x}\right) +\frac{1}{2} \| \y_i \|^2, i=1, \ldots, m,
	\end{aligned}
	$$
where each $i \in [n]$ corresponds to a binary classification task and $\x$ denotes the model parameter. 
\end{list}

With the system model and the problem formulation in \eqref{problem}, we are in a position to present our algorithm design in the next section.

\section{The \algname Algorithm }\label{INTERACT}
In this section, we present our \algname algorithm for solving the decentralized bilevel optimization problems and provide its theoretical convergence guarantees.  

\subsection{Algorithm Description}
To solve Problem~\eqref{problem}, we adopt the network consensus approach \cite{nedic2009distributed}. 
In each iteration, every agent shares and receives information from its neighboring nodes. Each agent aggregates the information received from its neighbours through a consensus weight matrix $\M \in \Rb^{m\times m}$. We denote by $[\M]_{ij}$ the element in the $i$-th row and the $j$-th column in $\M$.
We assume that $\M$ satisfies the following:

\begin{list}{\labelitemi}{\leftmargin=2em \itemindent=-0.09em \itemsep=.2em}
\item[(a)] {\em Doubly stochastic:} $\sum_{i=1}^{m} [\mathbf{M}]_{ij}=\sum_{j=1}^{m} [\mathbf{M}]_{ij}=1$;

\item[(b)] {\em Symmetric:} $[\mathbf{M}]_{ij} = [\M]_{ji}$, $\forall i,j \in \mathcal{N}$;

\item[(c)] {\em Network-Defined Sparsity:} $[\M]_{ij} > 0$ if $(i,j)\in \mathcal{L};$ otherwise $[\mathbf{M}]_{ij}=0$, $\forall i,j \in \mathcal{N}$.
\end{list}
The above properties imply that the eigenvalues of $\M$ are real and lie in the interval $(-1,1]$. We sort the eigen values of $\M$ as: $-1 < \lambda_m(\M) \leq \cdots \leq \lambda_2(\M) < \lambda_1(\M) = 1$. We denote the second-largest eigenvalue in magnitude of $\M$ as $\lambda \triangleq \max\{|\lambda_2(\M)|,|\lambda_m(\M)|\}$. 
We note that $\lambda$ plays an important role in the step-size selection for our proposed algorithms (cf. Theorems \ref{thm1} and \ref{thm2}). 
%
Next, using the implicit function theorem, we evaluate the hyper-gradient of $l_i(\x_i) \coloneq f_i ({\x}_i, {\y}_i^{*}({\x_i}))$ at $\x_{i,t}$ and $\y_i^{*}(\x_{i,t})$ as:
\begin{multline}\label{eqs2}
	\nabla \ell_i	(\x_{i,t}) =  \nabla_{\x} f_{i}(\x_{i,t},\y_{i}^*(\x_{i,t})) - \nabla_{\x\y}^2 g_{i}(\x_{i,t},\y_{i}^*(\x_{i,t})) \times \\
	[ \nabla_{\y\y}^2 g_{i}(\x_{i,t},\y_{i}^*(\x_{i,t})) ]^{-1}   \nabla_{\y} f_{i}(\x_{i,t},\y_{i}^*(\x_{i,t})).
\end{multline}
Note that computing the local gradient $\nabla l_i	(\x_{i,t})$, $i \in [m]$ requires knowledge of the lower level problem's optimal solution, $\y_i^*(\x_i)$. However, obtaining $\y_i^*(\x_i)$ in closed form is usually not feasible. Therefore, we replace $\y_i^*(\x_i)$ by an approximate solution $\y_{i,t}$, thereby, defining the following approximate gradient estimate \cite{ghadimi2018approximation}:
\begin{multline}\label{eqs3}
	\bar\nabla f_{i}(\x_{i,t},\y_{i,t})= \nabla_{\x} f_{i}(\x_{i,t},\y_{i,t}) - \nabla_{\x\y}^2 g_{i}(\x_{i,t},\y_{i,t}) \times \\
	[ \nabla_{\y\y}^2 g_{i}(\x_{i,t},\y_{i,t}) ]^{-1}   \nabla_{\y} f_{i}(\x_{i,t},\y_{i,t}).
\end{multline}

The proposed \algname algorithm for each agent $i \in [m]$ is illustrated in Algorithm~\ref{Algorithm_GD}. Note that, to achieve state-of-the-art convergence guarantees, \algname relies on consensus updates along with gradient tracking for solving Problem~\eqref{problem} as follows:
\begin{list}{\labelitemi}{\leftmargin=1em \itemindent=-0.09em \itemsep=.2em}
	\item {\em Step~1 (Consensus update with gradient descent):} The local parameters of each agent $i \in [m]$ are updated using \eqref{consensus_x_GD}. Note that the consensus is established by only sharing the information of the outer model parameters $\x_i$, while the inner-model parameters ${\y}_i$ are updated only locally and the local gradient is computed by just evaluating the local gradient of the lower level objective: 
		\begin{align}\label{consensus_x_GD}
	\x_{i,t} &=\sum_{j \in \mathcal{N}_{i}} [\M]_{ij} \x_{j,t-1} - \alpha\u_{i,t-1},
\end{align}
	\begin{align}\label{consensus_y_GD}
	\y_{i,t} &= \y_{i,t-1} - \beta 	\v_{i,t-1},
\end{align}	

where $\alpha$ and $\beta$ are the step-sizes for updating ${\x}$- and ${\y}$-variables, respectively, and $\v_{i,t}$ is the local gradient result, $\u_{i,t}$ is an auxiliary vector for gradient tracking purposes and will be defined shortly.

\item {\em Step~2 (Local gradient estimate):} 
Each agent $i \in [m]$ evaluates its (full) local gradient $ 	\p_{i}(\x_{i,t},\y_{i,t}) $ and $ 	\d_{i}(\x_{i,t},\y_{i,t}) $ using as:
			\begin{align}\label{Local_gradient_x_GD}
		\p_{i}(\x_{i,t}&,\y_{i,t}) =  \bar\nabla f_{i}(\x_{i,t},\y_{i,t}),
	\end{align}
			\begin{align}\label{Local_gradient_y_GD}
		\v_{i,t}=	\d_{i}(\x_{i,t},\y_{i,t}) &=  \nabla_{\y} g_{i}(\x_{i,t},\y_{i,t}).
			\end{align}
\item {\em Step~3 (Local gradient tracking):} Each agent $i \in [m]$ updates the global gradient estimate $\u_{i,t}$ by averaging over its neighboring global gradient estimates:
\begin{align}\label{gradient_tracking_x_GD}
	\u_{i,t} = \sum_{j \in \mathcal{N}_{i}} [\M]_{ij} \u_{j,t-1}  + \p_{i}(\x_{i,t},\y_{i,t})- \p_{i}(\x_{i,t-1},\y_{i,t-1}).
\end{align}
\end{list}
Note that estimating $\v_{i,t} $ does not require gradient tracking as the inner problem can be solved locally at each agent (cf. Problem~\eqref{problem}).

\begin{algorithm}[t!]
	\caption{The \algname Algorithm.}\label{Algorithm_GD}.
	\begin{algorithmic} 
		\State Set parameter pair $\forall i \in [m], (\x_{i,0},\y_{i,0}) = (\x^0,\y^0)$.
		\State At each agent $i \in [m]$, compute the local gradients:
		{{\begin{align}
					\u_{i,0} =&  \bar\nabla f_{i}(\x_{i,0},\y_{i,0}); ~~
					\v_{i,0} =  \nabla_{\y} g_{i}(\x_{i,0},\y_{i,0});\notag
		\end{align}}}
		\For{$t = 1, \cdots, T$}
		\State Update local parameters using (\ref{consensus_x_GD}) and (\ref{consensus_y_GD});
		\State Compute local gradients using (\ref{Local_gradient_x_GD})and (\ref{Local_gradient_y_GD});%
		\State Track the gradients using (\ref{gradient_tracking_x_GD});
		\EndFor
	\end{algorithmic}
\end{algorithm}

\subsection{Convergence Results of \algname} 
In this subsection, we focus on the theoretical guarantees associated with \algname (Algorithm \ref{Algorithm_GD}). 
Providing convergence guarantees of \algname for solving the decentralized bilevel optimization problem in \eqref{problem} presents two major challenges: 
{\em a) Inner-loop error:} Recall that, in constructing the true gradient estimate \eqref{eqs2}, we utilize a crude approximation $\y_{i,t} \neq \y_{i,t}^*(\x_{i,t}) $. 
This approximation induces an error in the estimated gradient of (cf. \eqref{eqs2}) that presents a major challenge in the analysis of \algname; 
{\em b) Decentralized topology:} Since \algname applies decentralized consensus in outer-loop iterations, the analysis needs to capture the effect of the errors induced by the decentralized topology. 
These two issues make the convergence analysis of \algname difficult. 
In spite of these challenges, we will show that \algname provides state-of-the-art convergence guarantees for solving Problem~\eqref{problem}. 

\smallskip
{\bf 1) Convergence metric:} 
We propose the following {\em new} metric for decentralized bilevel problems, which is the key in determining all convergence results in the following sections: 
\begin{align}\label{Eq: metric1}
	\mathfrak{M}_t \triangleq& \Eb[\left\| \nabla \ell(\bar\x_t ) \right\|^{2} +
	\left\|{\x}_{t} - \1 \otimes \bar{{\x}}_{t}\right\|^{2}+\|{\y}_t^* -  {{{\y}}}_t\|^2],
\end{align}
where  ${\y}_t^*$ is defined in \eqref{stationary point}.
Note that the first term in~\eqref{Eq: metric1} quantifies the convergence of the $\bar\x_t$ to a stationary point of the global objective. 
The  second term in~\eqref{Eq: metric1} measures the consensus error among local copies of the outer variable, while the third term in~\eqref{Eq: metric1} quantifies the (aggregated) error in the inner problem's iterates across all agents. 
Thus, if the iterates generated by an algorithm that achieves $\mathfrak{M}_t \rightarrow 0$, it implies that the algorithm achieves three goals simultaneously: 1) consensus of outer variables, 2) stationary point of Problem~\eqref{problem}, and 3) solution to the inner problem.  

\smallskip
{\bf 2) Technical Assumptions:} Next, before presenting the main convergence result of  \algname, we first state the technical assumptions on the outer and inner objective functions in Problem~\eqref{problem}:  
\begin{assum} \label{assum1}
The inner function $g$ satisfies:
\begin{list}{\labelitemi}{\leftmargin=2.5em \itemindent=-0.09em \itemsep=.2em}
\item[a)] For any $i \in [m]$, and ${\x}_i\in \mathbb{R}^{d_1}$, $g_i(\x_i, \y_i)$ is $\mu_{g}$ -strongly convex with respect to the variable $ \y_i$, i.e., $\mu_{g} I \preceq \nabla_{y}^{2} g({\x}_i, \y_i)$.
	
\item[b)] For any $i \in [m],{\x}_i\in \mathbb{R}^{d_1}, {\y}_i\in \mathbb{R}^{d_2}$, $\nabla g_i(\x_i, \y_i) $ is $L_{g}$-Lipschitz continuous with $L_g>0$.
	
\item[c)]  For any $i \in [m],
{\x}_i\in \mathbb{R}^{d_1}, {\y}_i\in \mathbb{R}^{d_2} \text {, we have }\left\|\nabla_{x y}^{2} g_i(\x_i, \y_i)\right\|^{2} \leq C_{g_{x y}}$ for some $C_{g_{x y}}>0$.

\item[d)]  For any $i \in[m],
{\x}_i\in \mathbb{R}^{d_1}, \y_i \in \mathbb{R}^{d_2},$ we have $\nabla_{x, y}^{2} g_i(\x_i, \y_i)$ and $\nabla_{y y}^{2} g_i(\x_i, \y_i)$ are Lipschitz continuous with constants $L_{g_{x y}}>0$ and $L_{g_{y y}}>0$, respectively.
\end{list}
\end{assum}

\begin{assum} \label{assum2}
The outer function $f$ and $\ell(\x)$ satisfy:
\begin{list}{\labelitemi}{\leftmargin=2.5em \itemindent=-0.09em \itemsep=.2em}
\item[a)] For any $i \in [m], {\x}_i\in \mathbb{R}^{d_1}$, $\nabla_{\x} f_i(\x_i, \y_i), \nabla_{\y} f_i(\x_i, \y_i) $ are Lipschitz continuous with constant $L_{f_x}\geq 0, L_{f_y}\geq 0$.
	
\item[b)] For $i \in [m],{\x}_i\in \mathbb{R}^{d_1}, \y_i \in \mathbb{R}^{d_2}$, we have $\left\|\nabla_{y} f(\x_i, \y_i)\right\| \leq C_{f_{y}}$ for some $C_{f_{y}}\!\geq\! 0$.
	
\item[c)] There exists a finite lower bound $\ell^* = \inf_{\x} \ell(\x) > -\infty$, where $ {\x} \in \mathbb{R}^{d_1}$.
\end{list}
\end{assum}

Assumptions~\ref{assum1} and \ref{assum2} are standard in the literature of bilevel optimization (e.g., \cite{khanduri2021near}\cite{ghadimi2018approximation}). Many problems of practical interest in meta-learning \cite{rajeswaran2019meta}\cite{liu2021boml} and optimization of area under precision-recall curve(AUPRC) \cite{qi2021stochastic}  can be shown to satisfy these assumptions. 

\smallskip
{\bf 3) Supporting Lemmas:} 
Next, we present several lemmas that characterizes the Lipschitz properties of the hypergradient in (\ref{eqs2}), the approximate gradient in (\ref{eqs3}), and the optimal solution of the inner problem under Assumptions~\ref{assum1} and \ref{assum2}.
These lemmas will be useful in proving our main convergence result (proofs directly follow from \cite{ghadimi2018approximation} due to structural similarities). 

\begin{lem} \label{lem:main}
Under Assumptions \ref{assum1}--\ref{assum2}, we have
\begin{align*}
& \|\bar{\nabla} f_i(\x, \y)-\nabla \ell_i(\x)\|^2 \leq L_f\left\|\y^{*}(\x)-\y\right\|^2,\\ 
& \left\|\y_i^{*}\left(\x_{1}\right)-\y_i^{*}\left(\x_{2}\right)\right\| ^2\leq L_{y}\left\|\x_{1}-\x_{2}\right\|^2, \\
& \left\|\nabla \ell_i\left(\x_{1}\right)-\nabla \ell_i \left(\x_{2}\right)\right\|^2 \leq L_{\ell}\left\|\x_{1}-\x_{2}\right\|^2,
\end{align*}
for all $i \in [m]$ and $ \x, \x_{1}, \x_{2} \in \mathbb{R}^{d_1}$ and $\y \in \mathbb{R}^{d_2}$. The above Lipschitz constants are defined as:
	$	L_f\!=\! \Big( L_{f_{x}}\!+\frac{L_{f_{y}} C_{g_{x y}}}{\mu_{g}}+ C_{f_{y}}\Big(\frac{L_{g_{x y}}}{\mu_{g}}\!+\!\frac{L_{g_{y y}} C_{g_{x y}}}{\mu_{g}^{2}}\Big)\Big)^2$,  $
		L_{\ell}=\Big({L_f+\frac{L_f C_{g_{x y}}}{\mu_{g}}} \Big)^2, L_y= \Big(\frac{C_{g_{xy}}}{\mu_{g}}\Big) ^2.$
\end{lem}

	\begin{lem} \label{lem:deterministic} 
	Under Assumption \ref{assum1}-\ref{assum2}, we have
	\begin{align}
	\left\|{\nabla} f_i\left(\x_{1}, \y_1 \right)-{\nabla} f_i\left(\x_{2}, \y_2\right)\right\|^{2}\leq L_{K}^{2} [\left\| \ \x_{1}- \ \x_{2}\right\|^{2} \notag\\
		+\left\| \y_{1}- \y_{2}\right\|^{2}], \forall  \ \x_{1},  \ \x_{2} \in \mathbb{R}^{d_1},\y_{1},  \y_{2} \in \mathbb{R}^{d_2}  .
	\end{align}
	In the above expressions, $\exists  L_{K}\geq L_{K_d}$ and $L_{K_d}$ is defined as:
	\begin{align}
		L_{K_d}^{2}:=&2 L_{f_{x}}^{2}+6C_{g y}^{2} \frac{1}{\mu_g^2} L_{f_{y}}^{2} +6 C_{f_{y}}^{2}\frac{1}{\mu_g^2}  L_{g_{x y}}^{2} +6 C_{g_{x y}}^{2} C_{f_{y}}^{2} L_{g_{y y}}^{2}\frac{1}{\mu_g^4}.
	\end{align}
	\end{lem}


\smallskip
{\bf 4) Convergence Results of \algname:}
We state the main convergence rate result of \algname in the following theorem:
\begin{thm}\label{thm1}
	(Convergence of \algname) Under Assumptions~\ref{assum1}-\ref{assum2}, if the step-sizes satisfy $  \beta\leq\min\{  \frac{3(\mu_g+L_g)}{ \mu_g L_g},\frac{1}{\mu_g+L_g} \}$,
	$\alpha \leq \min\{ \frac{1}{4L_{\ell}}$,
	
	\hspace{-0.14in}$ \frac{1}{4  L_K} \sqrt{\frac{1-\lambda}{2m}},  \frac{1}{m (1-\lambda)}, \frac{(1-\lambda)^2}{32L_K^2 } , \frac{  m(1-\lambda)}{4L_{\ell}}, \frac{9r^2m(1-\lambda) }{32L_y^2(1+1/r)L_f^2} 
	, \frac{(1-r) (1+r) r (1-\lambda)^2}{32 L_y^2  (\mu_g+L_g) L_K^2 \beta }$,
	
		\hspace{-0.14in}$
	 \frac{1-\lambda}{4 L_K} ,     1
	 \}, r=\frac{1}{3}\beta \frac{\mu_g L_g}{\mu_g+L_g }$,	
	then $\{ \x_t,\y_t \}$ generated by \algname satisfy  
	\begin{align}
		\frac{1}{T+1}\sum_{t=0}^{T} \mathfrak{M}_t &\leq \frac{\mathfrak{B}_0- \ell^*}{(T+1)\min\{\frac{1-\lambda}{4},\frac{3  r^2(1-\lambda)}{32(1+r)L_y^2 },\frac{\alpha}{2}\}} \notag = \mathcal{O}\bigg(\frac{1}{T} \bigg),
	\end{align}
	where $ \mathfrak{B}_t=	\ell(\bx_{t}) + \|\y_{t} - \y_{t}^*\|^2  +	\|\x_{t}-\ot\bx_{t}\|^2 +
	\alpha 	\|\u_{t}-\ot\bu_{t}\|^2  $.
\end{thm}

\begin{remark}{\em
Note from the statement of Theorem \ref{thm1} that \algname requires constant step-sizes $\alpha$ and $\beta$ that depend on the network topology, Lipschitz constants, and the number of agents.
In particular, different network topology leads to different network consensus matrix ${\M}$.
Recall that $\lambda < 1$ is the second-largest eigenvalue in magnitude of ${\M}$.
For a dense network, $\lambda$ is close to $0$, which allows larger step-size $\alpha$, and then leads to faster convergence.
}
\end{remark}

Theorem~\ref{thm1} immediately implies the following sample and communication complexity results of \algnamens: 
\begin{cor}\label{Cor: INTERACT}
	Under the conditions of Theorem~\ref{thm1}, to achieve an $\epsilon$-stationary solution, \algname requires: 1) Communication complexity of $\mathcal{O}(\epsilon^{-1})$ , and 2) Sample complexity of $\mathcal{O}(n\epsilon^{-1}))$ . \qed
\end{cor}

\subsection{Proof Sketches of the Convergence Results}
Due to space limitation, we provide a proof sketch of Theorem~\ref{thm1}.
The proof details are provided in the Appendix.
In this section, we organize the proof into several key steps. Our first step is to show
the descent property of our algorithms, which is stated as follows:

\smallskip
{\em Step 1) Descending Inequality for the Outer Function:} 
Under Assumptions~\ref{assum1}-\ref{assum2}, the following inequality holds for Algorithm~\ref{Algorithm_GD}:
\begin{align}\label{descending}
	\ell({\bx}_{t+1}) - \ell(\bx_{t}) \le &-\frac{\alpha}{2} \|\nabla \ell(\bx_t)\|^2 - (\frac{\alpha}{2} - \frac{L_{\ell}\alpha^2}{2})\|\bu_{t}\|^2  \notag\\
	&	\!\!\!\! + \frac{\alpha}{m}\sum_{i=1}^{m} L_{\ell} \| \bx_{t} - \x_{i,t}\|^2 +  \frac{2\alpha}{m} \sum_{i=1}^{m} L_f^2 \| \y_{i,t}^* - \y_{i,t} \|^2\notag\\
& + 2\alpha \bigg\| \frac{1}{m} \sum_{i=1}^{m} \bar\nabla f_i (\x_{i,t} ,\y_{i,t} ) - \bar\u_{ t} \bigg\|^2 ,
\end{align}
where $\y_{i,t}^* = \arg\min_{\y_{i,t}}  g_i(\x_{i,t},\y_{i,t})$.

Eqs.~(\ref{descending}) focuses on the descending upper bound for adjacent iterations of outer-level model parameters.
Unlike traditional decentralized minimization optimization problems, decentralized bilevel optimization problems tackle a composition of an inner problem and an outer problem.
This tightly coupled inner-outer mathematical structure in the bilevel optimization problem, together with the decentralized nature and the non-convexity of the outer problem, imposes significant challenges in the theoretical algorithm analysis.
As a result, Eq.~(\ref{descending}) contains the consensus error of the outer-level model parameters $\| \bx_{t} - \x_{i,t}\|^2$ and the convergence metric of inner model parameter $\| \y_{i,t}^* - \y_{i,t} \|^2$.

\smallskip
{\em  Step 2) Error Bound on $\y^*(\x)$:} 
Under Assumptions~\ref{assum1} and \ref{assum2}, the following inequality holds for Algorithm \ref{Algorithm_GD}:
\begin{align}
&	\|\y_{i, t+1} - \y_{i, t+1}^*\|^2  
	\le(1+r)^2 (1-2\beta \frac{\mu_g L_g}{\mu_g+L_g }) \|\y_{i, t} - \y_{i, t}^*\|^2 \notag\\&+(\delta -1) (1+r)^2 (2\beta \frac{1}{\mu_g+L_g } -\beta^2)\| \v_{i,t} \|^2+(1+\frac{1}{r}) L_y^2 \|\x_{i, t+1}-\x_{i, t}\|^2
	\notag\\&+[(\frac{1}{\delta}-1 ) (1+r)^2 (2\beta \frac{1}{\mu_g+L_g } -\beta^2)+(1+r)(1+1/r) \beta^2] \notag\\&\cdot\|  \nabla_{\y} g_{i}(\x_{i,t},\y_{i,t}) -\v_{i,t} \|^2.
\end{align}
In Step 2, we continue evaluating the fourth element in Eq.~(\ref{descending}). 
From the result of Step 2, we can see that $\|\y_{i, t} - \y_{i, t}^*\|^2$ shrinks if $2(1+r) (1-2\beta \frac{\mu_g L_g}{\mu_g+L_g })\leq 1$.

\smallskip
{\em Step 3) Iterates Contraction:}
The following contraction properties of the iterates hold:
\begin{align}
	\|\x_t-\ot\bx_t&\|^2 \le  (1+c_1)\lambda^2\|\x_{t-1} -\ot\bx_{t-1} \|^2  \notag\\&+ \Big(1+\frac{1}{c_1} \Big) \alpha^2\|\u_{t-1}-\ot \bu_{t-1}\|^2, \\
	\|\u_t-\ot\bu_t&\|^2 \le 
	(1+ c_2)\lambda^2\|\u_{t-1} - \ot \bu_{t-1}\|^2  \notag\\&+  (1 + \frac{1}{ c_2})\|\p_t-\p_{t-1} \|^2,
\end{align}
where $c_1$ and $c_2$ are some positive constants.

Step 3 is important in analyzing the convergence performance of the our proposed \algname.
The key step is to define $\Mt = \M \otimes \I_m$. 
Then, we have $\|\Mt\x_{t} -\ot\bx_{t} \|^2 = \|\Mt(\x_{t} -\ot\bx_{t}) \|^2 \le \lambda^2\|\x_{t} -\ot\bx_{t}\|^2$.
This is because $\x_{t} -\ot\x_{t}$ is orthogonal to $\1$, which is the eigenvector corresponding to the largest eigenvalue of $\Mt,$ and $\lambda = \max\{|\lambda_2|,|\lambda_m|\}.$

\smallskip
{\em Step 4) Potential Function:}
With results from Steps~1-3, we have:
\begin{align}\label{step4}
	&\ell(\bx_{T+1}) - \ell(\bx_{0}) +\frac{1-\lambda}{32(1+1/r)L_y^2 } \big[\|\by_{T+1} - \y_{T+1}^*\|^2  - \|\y_0^* - \by_0\|^2\big] \notag\\
&	\le-\frac{\alpha}{2} \sum_{t=0}^{T} \|\nabla \ell(\bx_t)\|^2 - (\frac{\alpha}{2} - \frac{L_{\ell}\alpha^2}{2}- 4(1+\frac{1}{r}) L_y^2\alpha^2m \frac{1-\lambda}{32(1+1/r)L_y^2 } )\notag\\&\cdot \sum_{t=0}^{T} \|\bu_{t}\|^2  + (\frac{\alpha L_{\ell}  }{m}  +(1+\frac{1}{r}) L_y^2 \frac{1-\lambda}{4(1+1/r)L_y^2 } )\sum_{t=0}^{T} \| \x_{t} - \ot\bx_{t}  \|^2 \notag\\
	&+  [\frac{2\alpha}{m}  L_f^2  +   \frac{1-\lambda}{32(1+1/r)L_y^2 } 3r -  \frac{1-\lambda}{32(1+1/r)L_y^2 }(1+3r)( 2\beta \frac{\mu_g L_g}{\mu_g+L_g }) ]\notag\\&\cdot \sum_{t=0}^{T}\|\y_{t} - \y_{t}^*\|^2+ 2\alpha \sum_{t=0}^{T} \| \frac{1}{m} \sum_{i=1}^{m} \bar\nabla f_i (\x_{i,t} ,\y_{i,t} ) - \bar\u_{ t} \|^2 \notag\\
	&+ \frac{1-\lambda}{32(1+1/r)L_y^2 }(\delta-1)(1+r)^2 (2\beta \frac{1}{\mu_g+L_g } -\beta^2)\sum_{t=0}^{T}\| \v_{t} \|^2\notag\\&+ 4(1+\frac{1}{r}) L_y^2\alpha^2 \frac{1-\lambda}{32(1+1/r)L_y^2 } (\sum_{t=0}^{T}\|\u_{t}- \ot\bu_{t}\|^2 ) \notag\\
	&+ \frac{1-\lambda}{32(1+1/r)L_y^2 }\cdot [\frac{(1+r)^2\beta^2}{r} +(\frac{1}{\delta}-1 ) (1+r)^2 (2\beta \frac{1}{\mu_g+L_g } -\beta^2)] \notag\\
	&\cdot \sum_{i=1}^m \sum_{t=0}^{T} \| \v_{i,t} - \nabla_{\y} g_{i}(\x_{i,t},\y_{i,t})\|^2.
\end{align}
Note that the coefficients of $ \|\bu_{t}\|^2 $ and $\| \v_{t} \|^2$ can be made negative when step-sizes are chosen properly and $ \| \frac{1}{m} \sum_{i=1}^{m} \bar\nabla f_i (\x_{i,t} ,\y_{i,t} ) - \bar\u_{ t} \|^2 $ can be decomposed later, we define our potential function as:
\begin{align}
 \mathfrak{B}_t=&	\ell(\bx_{t}) \!+\!  \frac{(1-\lambda)\|\y_{t} - \y_{t}^*\|^2}{32(1+1/r)L_y^2 }  \! +\! 	\|\x_{t}\! -\! \ot\bx_{t}\|^2\!  +\! \alpha \|\u_{t}-\ot\bu_{t}\|^2.  \notag
\end{align}

\smallskip
{\em Step 5) Main proof of Theorem \ref{thm1}:}
Note that we use the full gradient estimator in INTERACT, thus we have $ \| \frac{1}{m} \sum_{i=1}^{m} \bar\nabla f_i (\x_{i,t} ,\y_{i,t} ) - \bar\u_{ t} \|^2 =0$.    
Also, we have
\begin{align}
	\|\p_t-\p_{t-1}\|^2 
	{\le}
	L_K^2 m (\|\x_{t}- \x_{t-1} \|^2+\beta^2 \|\v_{t-1} \|^2),
\end{align}
where  $\p_t=\left[\p_1(\x_{1,t},\y_{1,t})^{\top}, \cdots, \p_m(\x_{m,t},\y_{m,t})^{\top}\right]^{\top}$.

Combing the above results and choosing $c_1=c_2=\frac{1}{\lambda}-1$ yield:
\begin{align}
	&\ell(\bx_{T+1}) - \ell(\bx_{0}) + \frac{1-\lambda}{32(1+1/r)L_y^2 } \big[\|\by_{T+1} - \y_{T+1}^*\|^2  - \|\y_0^* - \by_0\|^2\big] \notag\\&+	[\|\x_{T+1}-\ot\bx_{T+1}\|^2-	\|\x_{0}-\ot\bx_{0}\|^2 ]\notag\\&+
	\alpha [	\|\u_{T+1}-\ot\bu_{T+1}\|^2  -	\|\u_{0}-\ot\bu_{0}\|^2 ]
	\notag\\
	&\le -\frac{\alpha}{2} \sum_{t=0}^{T} \|\nabla \ell(\bx_t)\|^2 + C_1\sum_{t=0}^{T} \|\bu_{t}\|^2  
	+C_2\sum_{t=0}^{T} \| \x_{t} - \ot\bx_{t}  \|^2 \notag\\&+  C_3 \sum_{t=0}^{T}\|\y_{t} - \y_{t}^*\|^2 +C_4\sum_{t=0}^{T}\| \v_{t} \|^2+ C_5\sum_{t=0}^{T}\|\u_{t}- \ot\bu_{t}\|^2,
\end{align}
where the constants are relegated to our Appendix.
With appropriately chosen parameters to ensure $C_1,C_4,C_5 \leq 0, C_2\leq-\frac{1-\lambda}{4},  C_3\leq -\frac{3  r^2(1-\lambda)}{32(1+r)L_y^2 }$, we have the following convergence results:
\begin{align}
	\frac{1}{T+1}\sum_{t=0}^{T} \mathfrak{M}_t \leq \frac{\mathfrak{B}_0- \ell^*}{(T+1)\min\{\frac{1-\lambda}{4}, \frac{3  r^2(1-\lambda)}{32(1+r)L_y^2 },\frac{\alpha}{2}\}}.
\end{align}
This completes the proof of Theorem~\ref{thm1}.

\section{ The \algtname Algorithm}\label{SVR-INTERACT}
Note that, in the \algname algorithm (Algorithm \ref{Algorithm_GD}), each agent needs to evaluate the full local gradient in each iteration, which might not be feasible for many large-scale problems. 
This motivates us to develop a stochastic version of \algname that circumvents the need to compute full local gradients in each iteration. 
To this end, we  leverage the variance reduce methods to design \algtnamens, a stochastic version of \algname that achieves low sample complexity for solving decentralized bilevel optimization problem.
In this section, we first present the \algtname method, and then provide its convergence guarantees.

In particular, we further define an (approximate) stochastic gradient estimate of the local gradient in \eqref{eqs3}. We define the stochastic gradient of $\bar\nabla f_{i}(\x_{i,t},\y_{i,t})$ by  $\bar\nabla f_{i}(x_{i,t}, y_{i,t}; \bar\xi_i)$ as
\begin{align}\label{eqs4}
	\bar\nabla f_{i}&(\x_{i,t},\y_{i,t}; \bar\xi_i)=   \nabla_{\x} f_{i}(\x_{i,t},\y_{i,t};\xi_i^0) -\frac{K}{L_g } \nabla_{\x\y}^2 g_{i}(\x_{i,t},\y_{i,t};\zeta_i^0)\notag\\&\cdot \prod_{j=1}^{\mathrm{k}(K)}( I - \frac{\nabla_{\y\y}^2 g_{i}(\x_{i,t},\y_{i,t};\zeta_i^j) }{L_g})  \nabla_{\y} f_{i}(\x_{i,t},\y_{i,t};\xi_i^0),
\end{align}
where $\mathrm{k}(K) \sim \mathcal{U}\{0, \ldots, K-1\}$ denotes a random variable uniformly chosen from $\{0, \ldots, K-1\}$. 
Note that the stochastic gradient estimator collects $K+2\in \mathbb{N}$ independent samples $\bar{\xi}_i:=\left\{\xi_i', \xi_i^{0}, \ldots, \xi_i^{K}, \mathrm{k}(K)\right\}$, where $\xi_i',\xi_i^j \sim \mu, \zeta_i^{j} \sim \pi, j=0, \ldots, K$ denote the samples of the outer and inner objectives. 
Next, we define the convergence metric that we will utilize for solving the decentralized bilevel learning problem.

\subsection{Algorithm Description} 
In \algtnamens, we use the same network consensus approach as in \algnamens. 
%

The proposed \algtname algorithm is illustrated in Algorithm~\ref{Algorithm_VR}. 
In each iteration, every agent $i \in [m]$ estimates its full gradients every $q$ steps. 
For other iterations, when $mod(t,q)\neq0$, \algtname estimates local gradients $\p_{i}(\x_{i,t},\y_{i,t}) $ and $\d_{i}(\x_{i,t},\y_{i,t}) $ using the following gradient estimators:
	    	\begin{align}\label{Local_gradient_x_VR}
	\p_t \!= \!	\p_{t-1}\!+\!	\frac{1}{| \mathcal{S}|}\sum_{i=1}^{\mathcal{S}} \big[\bar\nabla f_{i}&(\x_{i,t},\y_{i,t}; \bar\xi_i )\!-\!\bar\nabla f_{i}(\x_{i,t-1},\y_{i,t-1};  \bar\xi_i )\big];
\end{align}
\vspace{-0.1in}
		\begin{align}\label{Local_gradient_y_VR}
	\d_t \!= \!	\d_{t-1}\!+	\!\frac{1}{ |\mathcal{S}|}\sum_{i=1}^{\mathcal{S}} \big[	\nabla g_{i}(\x_{i,t},\y_{i,t}; \bar{\xi_i'} )\!-\!\nabla g_{i}(\x_{i,t-1},\y_{i,t-1}; \bar{\xi_i'})\big],
\end{align}

where we define:
\begin{align}
	\p_t \triangleq \left[\p_1(\x_{1,t},\y_{1,t})^{\top}, \cdots, \p_m(\x_{m,t},\y_{m,t})^{\top}\right]^{\top};\notag\\ \d_t \triangleq \big[\d_1(\x_{1,t},\y_{1,t})^{\top}, \cdots, \d_m(\x_{m,t},\y_{m,t})^{\top}\big]^{\top}.
\end{align}

Note that in contrast to \algnamens, \algtname utilizes a variance-reduced gradient estimator. Moreover, \algtname makes use of gradient tracking and consensus updates similar to the \algname algorithm. 
Next, we will present the convergence rate results of \algtnamens. 
\begin{algorithm}[t!]
	\caption{The \algtname Algorithm.}\label{Algorithm_VR}.
	\begin{algorithmic} 
		\State Set parameter pair $\forall i \in [m], (\x_{i,0},\y_{i,0}) = (\x^0,\y^0)$.
		\State At each agent $i \in [m]$, compute the local gradients:
		{{\begin{align}
					\u_{i,0} =&  \bar\nabla f_{i}(\x_{i,0},\y_{i,0}); 
					\v_{i,0} =  \nabla_{\y} g_{i}(\x_{i,0},\y_{i,0});
		\end{align}}}
		\For{$t = 1, \cdots, T$}
	\State Update local parameters using (\ref{consensus_x_GD}) and (\ref{consensus_y_GD});
	 \If{$mod(t,q)=0$}
	\State Compute local gradients using (\ref{Local_gradient_x_GD}) and (\ref{Local_gradient_y_GD});
	\Else
		\State Compute local gradients using (\ref{Local_gradient_x_VR}) and (\ref{Local_gradient_y_VR});
       	   \EndIf
       	   	\State Track the gradient using (\ref{gradient_tracking_x_GD});
		\EndFor
	\end{algorithmic}
\end{algorithm}

\subsection{Convergence Results of \algtname}
\smallskip
For \algtnamens, we need an extra assumption on the outer and inner objective functions in Problem~\eqref{problem} due to the stochastic gradient estimator we use.  

\begin{assum}[Stochastic Functions] \label{assum3} Assumptions 1 and 2 hold for $f_i(\x_i, \y_i ; \xi)$ and $g_i(x_i, y_i ; \xi)$, for all $\xi \in$ $\operatorname{supp}\left(\pi\right)$, where $\operatorname{supp}(\pi)$ represents the support of $\pi$.
\end{assum}

\begin{lem} \label{lem:stochastic_gradients}
	Under Assumptions \ref{assum1}--\ref{assum2}, $\forall i\in [m], (\x, \y) \in \mathbb{R}^{d_1} \times \mathbb{R}^{d_2}$, the stochatic gradient estimator in (\ref{eqs4}) satisfies:
	\begin{align*}
		\left\|\bar{\nabla} f_i(\x, \y)-\mathbb{E}_{\bar{\xi}_i}[\bar{\nabla} f(\x, \y ; \bar{\xi}_i)]\right\| \leq \frac{C_{g_{x y}} C_{f_{y}}}{\mu_{g}}\left(1-\frac{\mu_{g}}{L_{g}}\right)^{K}.
	\end{align*}
\end{lem}

	Besides, we have an additional lemma to characterizes the bias of the stochastic estimator shown in eqs.(\ref{eqs4}), and it also holds for the deterministic estimator shown in eqs.(\ref{eqs3}).
	\begin{lem} \label{lem:stocahstic} 
		Under Assumption \ref{assum1}-\ref{assum3}, we have
		\begin{multline}
			\mathbb{E}_{\bar{\xi}} \left\|{\nabla} f_i\left(\x_{1}, \y_1 ; \bar{\xi}\right)-{\nabla} f_i\left(\x_{2}, \y_2 ; \bar{\xi}\right)\right\|^{2}\leq L_{K}^{2} [\left\| \ \x_{1}- \ \x_{2}\right\|^{2} \\
			+\left\| \y_{1}- \y_{2}\right\|^{2}], \forall  \ \x_{1},  \ \x_{2} \in \mathbb{R}^{d_1},\y_{1},  \y_{2} \in \mathbb{R}^{d_2}  .
		\end{multline}
		In the above expressions, $\exists  L_{K}\geq L_{K_s}$ and $L_{K_s}$ is defined as:
		\begin{align}
			L_{K_s}^{2}=2 L_{f_{x}}^{2}++6 C_{g_{x y}}^{2} L_{f_{y}}^{2}\left(\frac{K}{2 \mu_{g} L_{g}-\mu_{g}^{2}}\right)+6 C_{f_{y}}^{2} L_{g_{x y}}^{2}\left(\frac{K}{2 \mu_{g} L_{g}-\mu_{g}^{2}}\right) \notag\\+ 6 C_{g_{x y}}^{2} C_{f_{y}}^{2} \frac{K^4}{L_{g}^{2}}   \frac{1}{L_g^2}L_{g_{yy}}^{2}.\notag
		\end{align}
	\end{lem}

To establish the convergence of \algtnamens, we use the same convergence metric as that of \algname as shown in (\ref{Eq: metric1}).

\smallskip
{\bf 1) Main Theorem of \algtnamens:} We state the main convergence rate result of \algtname in the following theorem:

\begin{thm}	(Convergence of \algtnamens)
	\label{thm2}
	Under Assumptions \ref{assum1}-\ref{assum3}, if the step-sizes satisfy 
	$  \beta\leq \min\{ \frac{(1-\lambda)\mu_g L_g}{768 L_K^2 (\mu_g+L_g)} , \frac{(1-\lambda) (\mu_g+L_g)}{4096 L_K^2 } $,
	
	$\frac{3 (\mu_g+L_g) }{\mu_gL_g},  \frac{ L_y^2 \mu_g L_g}{24 L_K^2(\mu_g+L_g) }, \frac{(1-\lambda)(\mu_g+L_g)}{512L_K^2} ,  \frac{1}{2(\mu_g+L_g)},  \frac{16}{(1-\lambda)(\mu_g+L_g)} \} $,
	
 $ \alpha\leq\{  \frac{1}{8L_{\ell}},    \frac{r}{16m L_y^2(r+1)},\frac{1}{8L_K\sqrt{m}} ,\frac{1}{m(1-\lambda)}  , \frac{(1-\lambda)^2}{128L_K^2 } , \frac{  (1-\lambda)}{4}(\frac{m}{L_{\ell}+16L_K^2m})$,
 
$ \frac{1}{16 L_K} \sqrt{\frac{1-\lambda}{m}}  ,  \frac{288 r (1+r) m L_y^2 }{(1-\lambda)L_f^2},  {\frac{r (1+r) (1-\lambda) }{256L_y^2 (\mu_g+L_g) L_K^2 \beta }}, \frac{r(1+r)(1-\lambda)^2}{512L_y^2 (\mu_g+L_g)(L_K^2\beta)} $,

$  \frac{\sqrt{ 1-\lambda}}{8L_K},   \frac{(1-\lambda)^2}{4}, \frac{32L_y^2 }{16(1+\frac{1}{r})L_y^2},\sqrt{\frac{1-\lambda}{64L_K^2} }, \frac{32L_y^2}{(1-\lambda)}  \},   r= \frac{1}{3}\beta \frac{\mu_g L_g}{\mu_g+L_g }$,
	then the iterates $\{\x_t,\y_t\}$ generated by \algtname satisfy
\begin{align}
\frac{1}{T+1}\sum_{t=0}^{T} \mathfrak{M}_t &\leq  \frac{\mathfrak{B}_0- \ell^*}{(T+1)\min\{\frac{1-\lambda}{4}, \frac{3  r^2(1-\lambda)}{32(1+r)L_y^2 },\frac{\alpha}{2}\}} + C_{\mathrm{bias}} \nonumber\\
&= \mathcal{O}(\frac{1}{T+1})+ C_{\mathrm{bias}},
\end{align}
where $ \mathfrak{B}_t=	\ell(\bx_{t}) + \|\by_{t} - \y_{t}^*\|^2  +	\|\x_{t}-\ot\bx_{t}\|^2 +
\alpha 	\|\u_{t}-\ot\bu_{t}\|^2, C_{\mathrm{bias}} \triangleq \frac{2\alpha  ( \frac{C_{g_{x y}} C_{f_{y}}}{\mu_{g}}\left(1-\frac{\mu_{g}}{L_{g}}\right)^{K})^2}{\min\{\frac{1-\lambda}{4}, \frac{3  r^2(1-\lambda)}{32(1+r)L_y^2 },\frac{\alpha}{2}\}}$.
\end{thm}

\begin{remark} {\em
	From the statement of Theorem \ref{thm2}, it can be seen that the step-sizes $\alpha$ and $\beta$ for \algtname depend on the network topology, the Lipschitz constants, and the number of agents.
	Additionally, the convergence performance of \algtname is affected by the constant $C_{bias}$, which is the bias term affected by the stochastic gradient estimator.  
}
\end{remark}

Theorem~\ref{thm2} immediately implies the following sample and communication complexity results of \algtnamens: 
\begin{cor}\label{Cor: SVR_INTERACT} 
Under the conditions of Theorem~\ref{thm2}, to achieve an $\epsilon$-stationary solution (cf. definition in Eq.~\eqref{stationary point}), with the number of samples for the gradient estimator in \eqref{eqs4} chosen as $K=\mathcal{O}(\log(1/\epsilon))$, \algtname requires: 1) communication complexity: $\mathcal{O}(\epsilon^{-1})$; and 2) sample complexity: $\mathcal{O}(\sqrt{n}\epsilon^{-1}))$.
\end{cor}

\begin{remark} {\em
Corollary \ref{Cor: SVR_INTERACT} indicates that \algtname has the same communication complexity as that of \algnamens, but improves the sample complexity of \algnamens.
}
\end{remark}

\subsection{Proof Sketch of the Convergence Results}
The proof of Theorem~\ref{thm2} follows the same structure as that of Theorem~\ref{thm1}, with Step~1)--Step 4) being identical.
Thus, we omit the first four steps for brevity and only focus on the last step in this paper.
In Algortihm \ref{Algorithm_VR}, we can show the following relations:
$\| \frac{1}{m} \sum_{i=1}^{m} \bar\nabla f_i (\x_{i,t} ,\y_{i,t} ) - \bar\u_{ t} \|^2=0$ every $q$ iterations. Let $n_{t}$ denotes the largest positive integer that satisfies $\left(n_{t}-1\right) q \leq t$.
With $t \in\left(\left(n_{t}-1\right) q, n_{t} q-1\right] \cap \mathbb{Z}$, we have
$
\Big\| \frac{1}{m} \sum_{i=1}^{m} \bar\nabla f_i (\x_{i,t} ,\y_{i,t} ) - \bar\u_{ t} \Big\|^2 
\leq  \frac{1}{m} \sum_{i=1}^{m} \|\bar\nabla f_i (\x_{i,t} ,\y_{i,t} ) - \p_i(\x_{i,t},\y_{i,t} ) \|^2.
$
We note that the next step is one of the most crucial steps in our proofs. 
The bias term has to be eliminated from the $ \|\bar\nabla f_i (\x_{i,t} ,\y_{i,t} ) - \p_i(\x_{i,t},\y_{i,t} ) \|^2$ in advance; otherwise, we will not be able to use the mean variance theorem later.
From the algorithm update of \algtnamens, we have
\begin{align} 
	& \mathbb{E}_t \|\bar\nabla f_i (\x_{i,t} ,\y_{i,t} ) - \p_i(\x_{i,t},\y_{i,t} ) \|^2
	\notag\\
	\leq& \mathbb{E}_t \| \p_i(\x_{i,t},\y_{i,t} )-\mathbb{E}_{\bar\xi_i}[ \p_i(\x_{i,t},\y_{i,t} )] \|^2 + \bigg(\frac{C_{g_{x y}} C_{f_{y}}}{\mu_{g}}\left(1-\frac{\mu_{g}}{L_{g}}\right)^{K} \bigg)^2.\notag
\end{align}
Next, telescoping over $t$ from ($(n_t-1)q+1$ to $t$ and noting $ |\mathcal{S}|=q$ implies that 
\begin{align}  \mathbb{E}_t \| \p_{i}(\x_{i,(n_t-1)q},\y_{i,(n_t-1)q}) -\mathbb{E}_{\bar\xi_i}[ \p_i(\x_{i,(n_t-1)q},\y_{i,(n_t-1)q} )] \|^2=0,\notag
\end{align}
Then, we have
\begin{align*} 
&\sum_{t=0}^T \| \frac{1}{m} \sum_{i=1}^{m} \bar\nabla f_i (\x_{i,t} ,\y_{i,t} ) - \bar\u_{ t} \|^2 \leq \!  \bigg[ \frac{C_{g_{x y}} C_{f_{y}}}{\mu_{g}}\left(1 \!-\! \frac{\mu_{g}}{L_{g}}\right)^{K} \bigg]^2 (T+1) \\
&+L_K^2\sum_{t=0}^T \mathbb{E}_t(\| \x_{t}-\x_{t-1} \|^2+\| \y_{t}-\y_{t-1} \|^2) \!.\notag
\end{align*}
Combing the above inequality and results in Steps~1)--4) and choosing $c_1=c_2=\frac{1}{\lambda}-1$, we have:

\begin{align}
&\ell(\bx_{T+1}) \!-\! \ell(\bx_{0}) \!+ \! \frac{1-\lambda}{32(1+1/r)L_y^2 } \big[\|\by_{T+1} \!-\! \y_{T+1}^*\|^2  - \|\y_0^* - \by_0\|^2\big] \notag\\&+	[\|\x_{T+1}\!-\!\ot\bx_{T+1}\|^2\!-\!	\|\x_{0}\!-\!\ot\bx_{0}\|^2 ] \!+\!\alpha [	\|\u_{T+1}\!-\!\ot\bu_{T+1}\|^2  \notag\\
&-	\|\u_{0}-\ot\bu_{0}\|^2 ] \le -\frac{\alpha}{2} \sum_{t=0}^{T} \|\nabla \ell(\bx_t)\|^2 + C_1'\sum_{t=0}^{T} \|\bu_{t}\|^2 \notag\\
&+C_2'\sum_{t=0}^{T} \| \x_{t} - \ot\bx_{t}  \|^2  + C_3' \sum_{t=0}^{T}\|\y_{t} - \y_{t}^*\|^2 +C_4'\sum_{t=0}^{T}\| \v_{t} \|^2 \notag\\
&+ C_5'\sum_{t=0}^{T}\|\u_{t} \!-\! \ot\bu_{t}\|^2 \!+\! 2\alpha \bigg[ \frac{C_{g_{x y}} C_{f_{y}}}{\mu_{g}}\left(1 \!-\! \frac{\mu_{g}}{L_{g}}\right)^{K} \bigg]^2 (T\!+\!1),
\end{align}

where the definitions of the constants can be found in our Appendix.
With appropriately chosen parameters to ensure $C_1',C_4',C_5' \leq 0, C_2'\leq-\frac{1-\lambda}{4},  C_3'\leq - \beta \frac{\mu_g L_g}{\mu_g+L_g }$, we have the following convergence results:
\begin{align}
\frac{1}{T+1}\sum_{t=0}^{T} \mathfrak{M}_t \leq \frac{\mathfrak{B}_0- \ell^*}{(T+1)\min\{\frac{1-\lambda}{4}, \frac{3  r^2(1-\lambda)}{32(1+r)L_y^2 },\frac{\alpha}{2}\}} + C_{\mathrm{bias}},
\end{align}
where $C_{\mathrm{bias}} \triangleq \frac{2\alpha  ( \frac{C_{g_{x y}} C_{f_{y}}}{\mu_{g}}\left(1-\frac{\mu_{g}}{L_{g}}\right)^{K})^2}{\min\{\frac{1-\lambda}{4}, \frac{3  r^2(1-\lambda)}{32(1+r)L_y^2 },\frac{\alpha}{2}\}}$.
This completes the proof.

\section{Experimental evaluation}\label{Section: experiment}
\begin{figure}[t!]
	\centering
	\subfigure[A five-agent network.]{
		\includegraphics[width=0.13\textwidth]{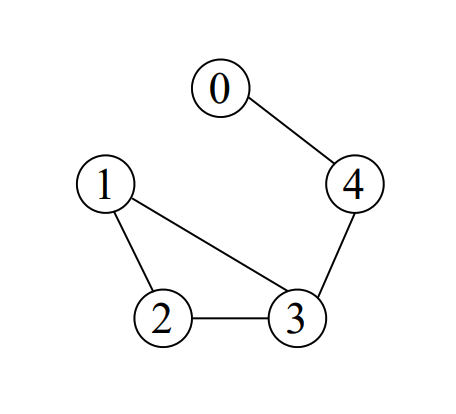}
		\label{Fig:5agent}
	}
	\subfigure[A 10-agent network.]{
		\includegraphics[width=0.13\textwidth]{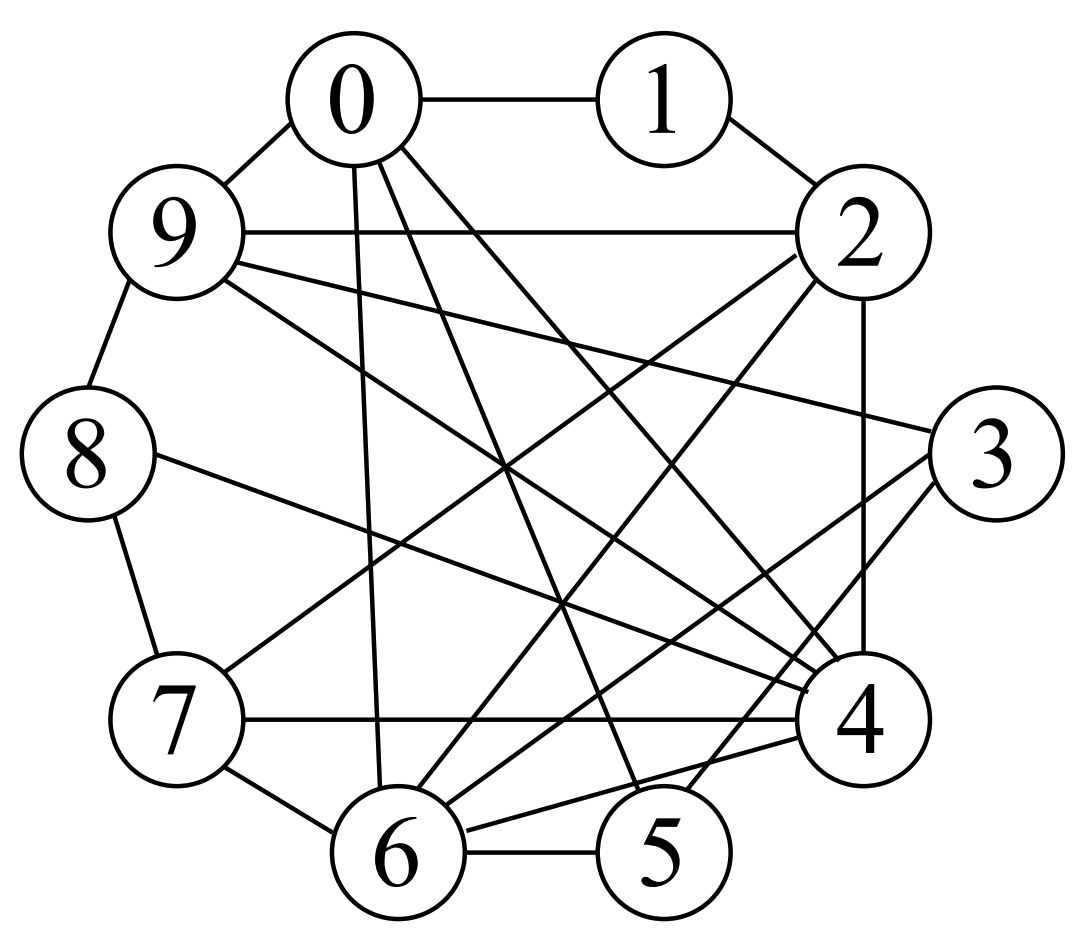}
		\label{Fig:10agent}
	}
\vspace{-0.15in}
	\caption{Different network topology.}
	\label{Fig:p_c_topology}
	\vspace{-0.3in}
\end{figure}

\begin{figure*}[t!]
	\centering
	\subfigure[The MNIST dataset.]{
		\includegraphics[width=0.215\textwidth]{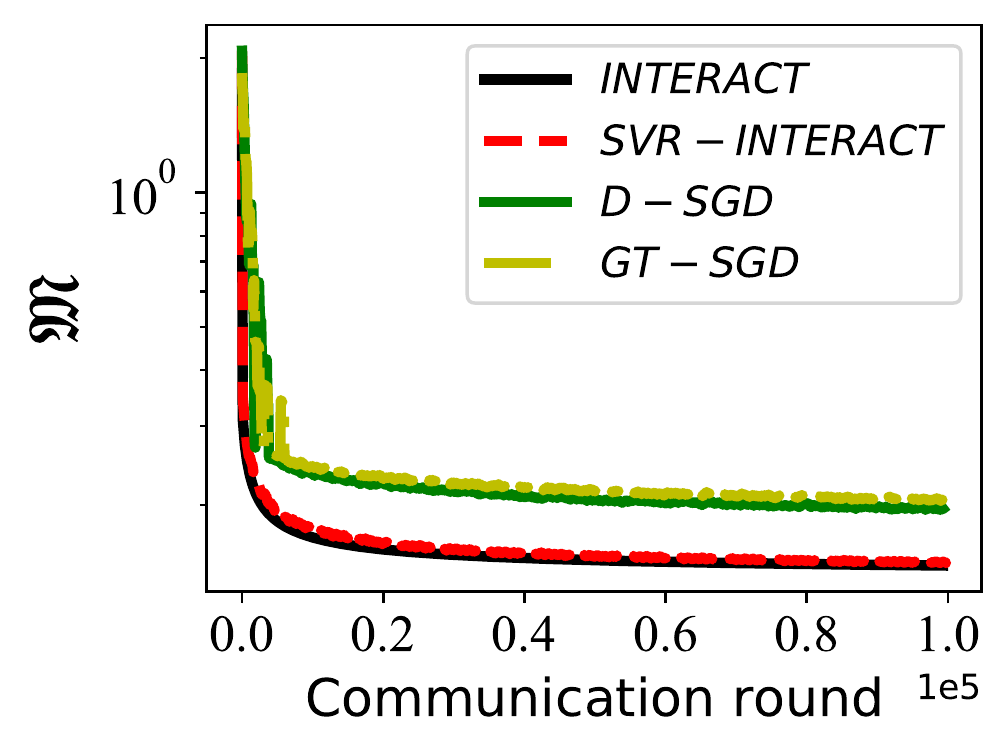}
		\includegraphics[width=0.215\textwidth]{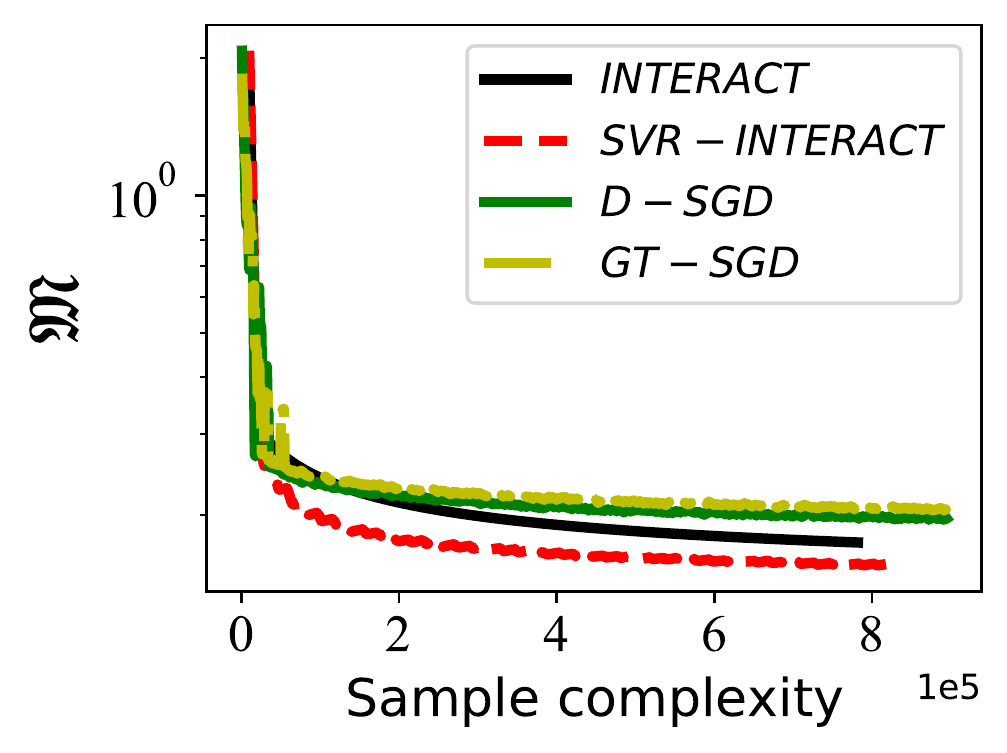}
	}
	\subfigure[The CIFAR-10 dataset.]{
		\includegraphics[width=0.22\textwidth]{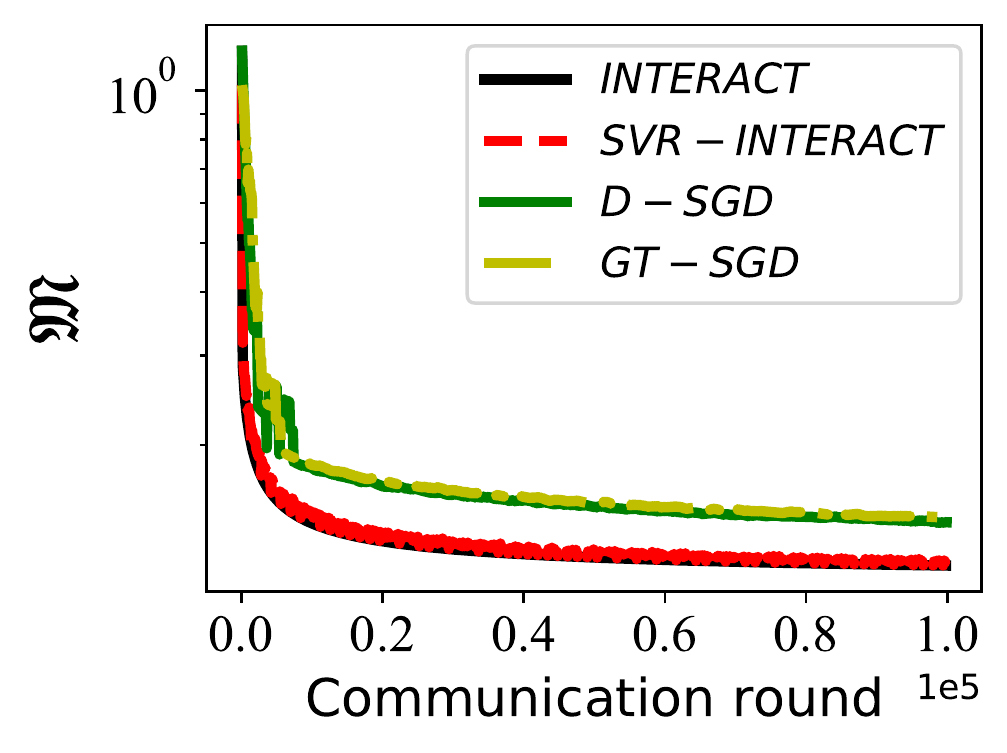}
		\includegraphics[width=0.22\textwidth]{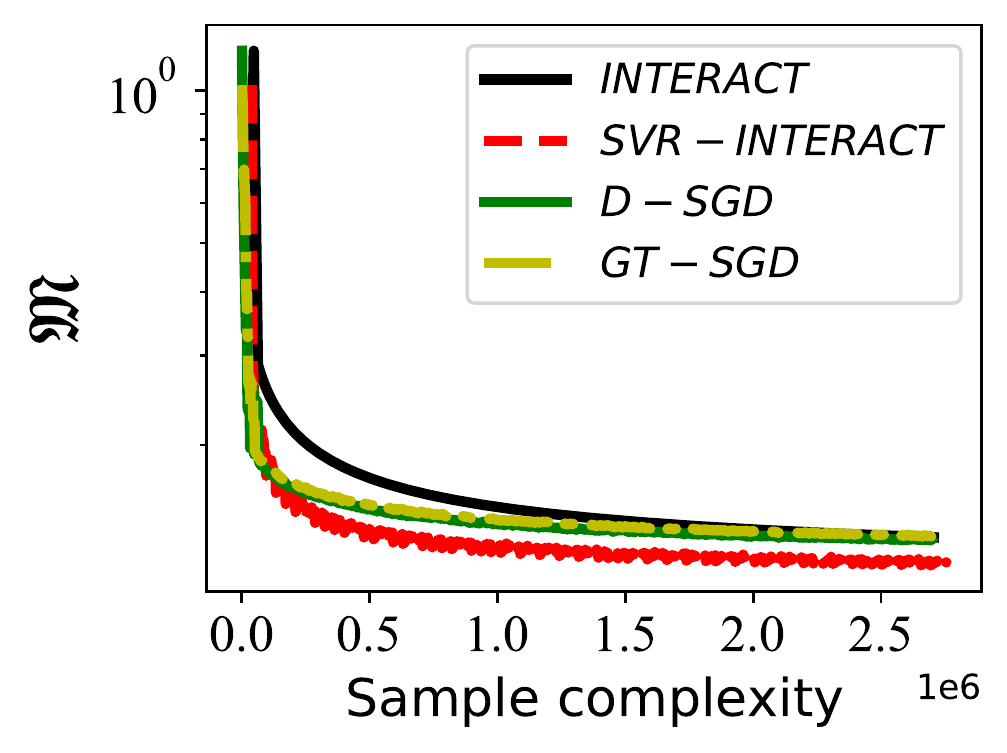}
	}
	\caption{Convergence performance comparisons on the five-agent network.}\label{Fig:performance_5agents}
\end{figure*}

\begin{figure*}[htbp]
	\centering
	\begin{minipage}[t]{0.45\textwidth}
		\centering
		\subfigure[]{
			\includegraphics[width=0.47\textwidth]{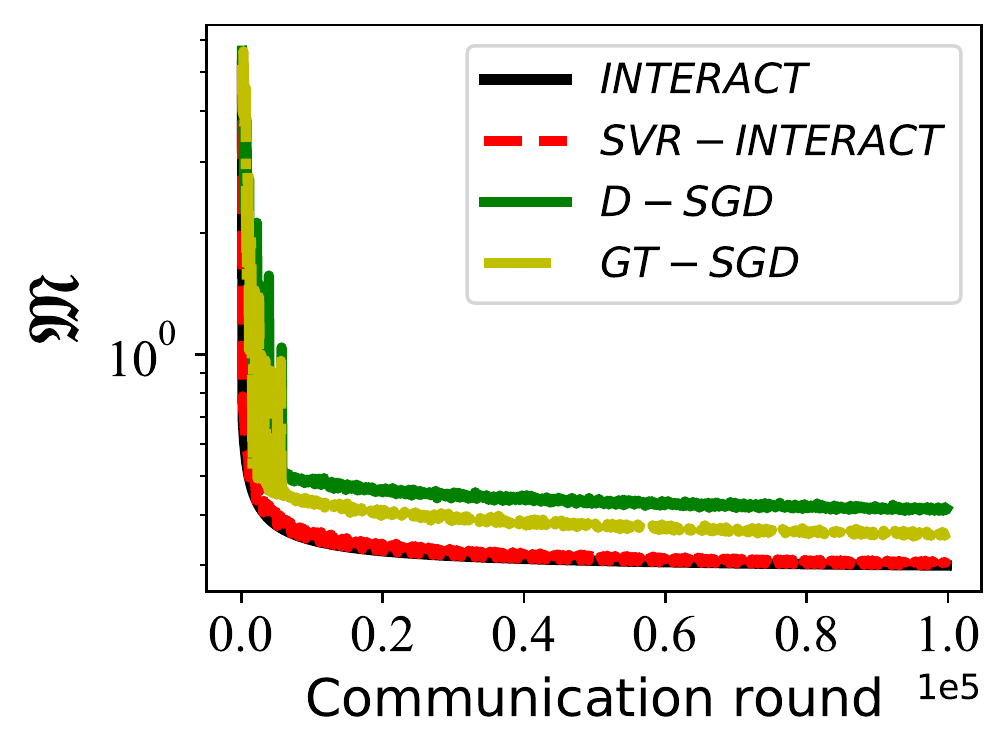}
		}
		\subfigure[]{
			\includegraphics[width=0.48\textwidth]{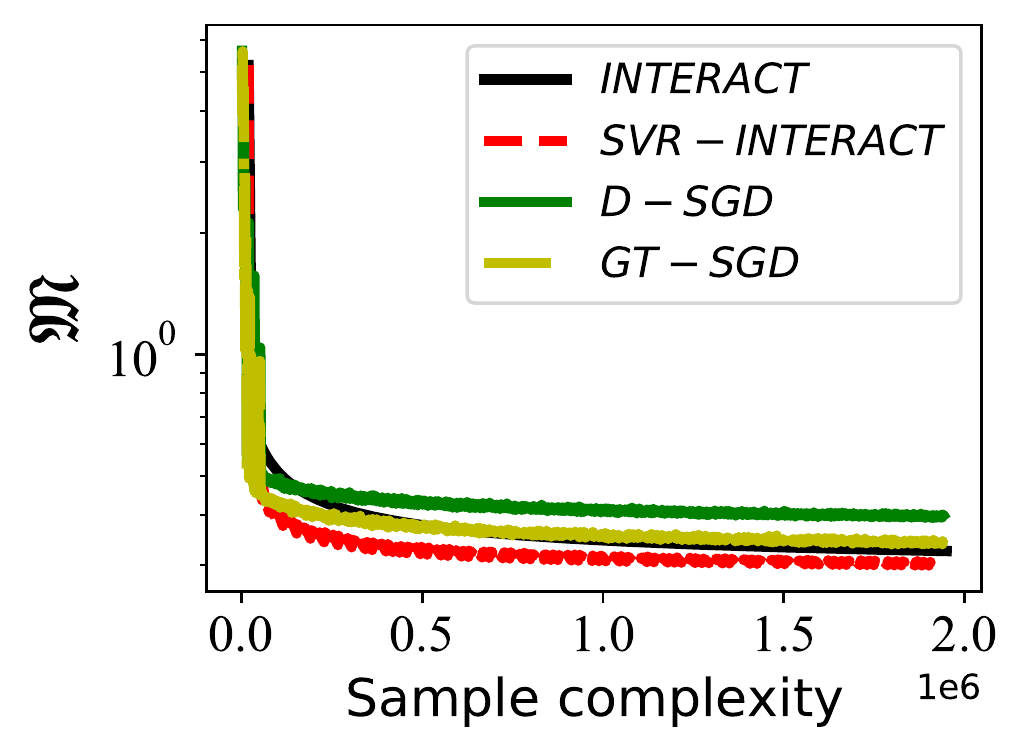}
		}
		\caption{Convergence performance comparisons on the 10-agent network.}\label{Fig:performance_10agents}
	\end{minipage}
	\hspace{.01\textwidth}
	\begin{minipage}[t]{0.45\textwidth}
		\subfigure[The \algname algorithm.]{
			\includegraphics[width=0.47\textwidth]{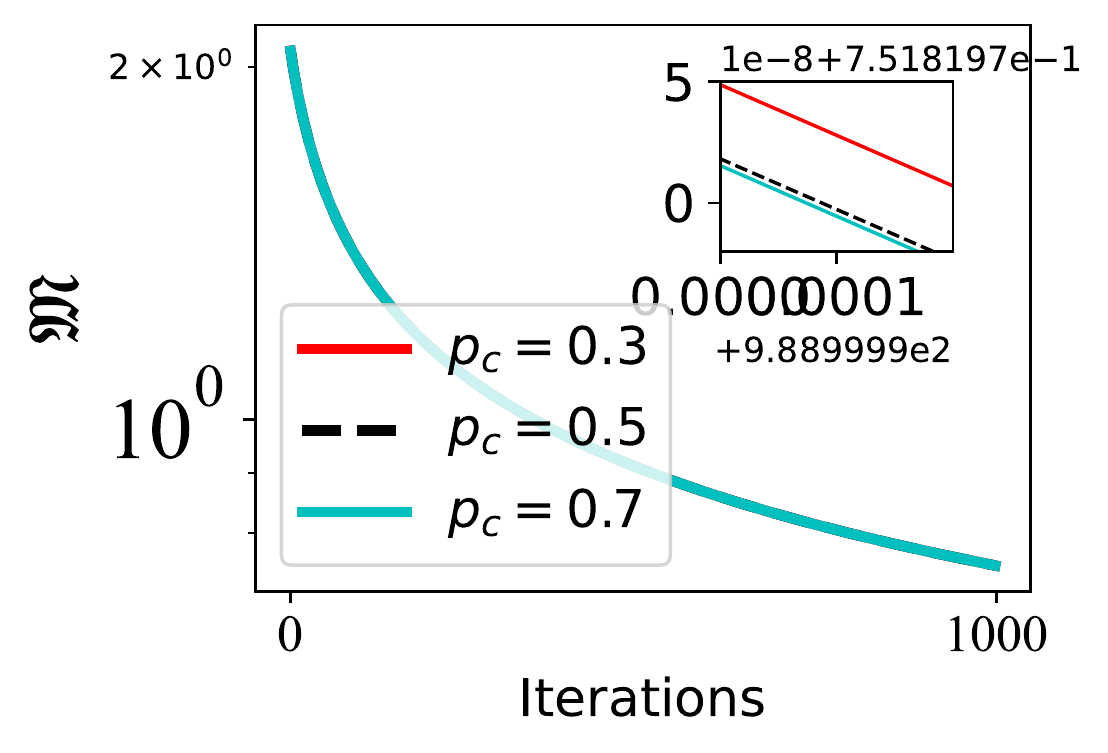}
		}
		\subfigure[The \algtname algorithm.]{
			\includegraphics[width=0.48\textwidth]{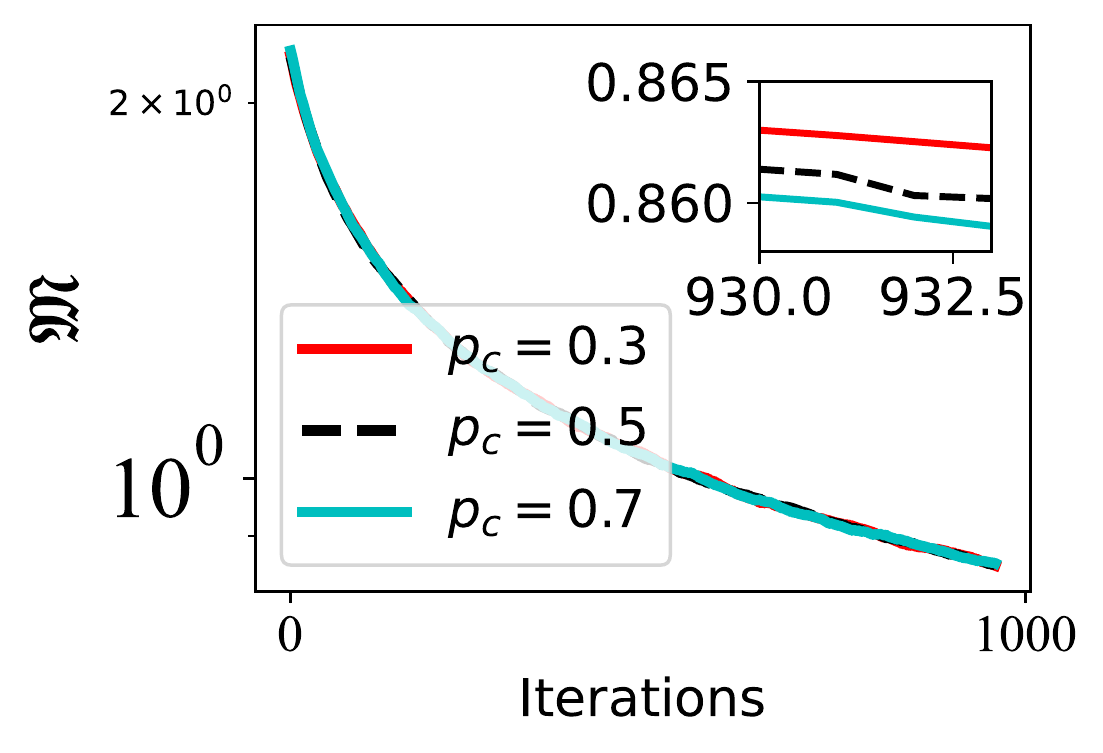}
		}
		\caption{Convergence performance with different edge connectivity probability.}\label{Fig:p_c}
	\end{minipage}
\end{figure*}
\begin{figure*}[t!]
	\centering
	\subfigure[The \algname algorithm.]{
		\includegraphics[width=0.22\textwidth]{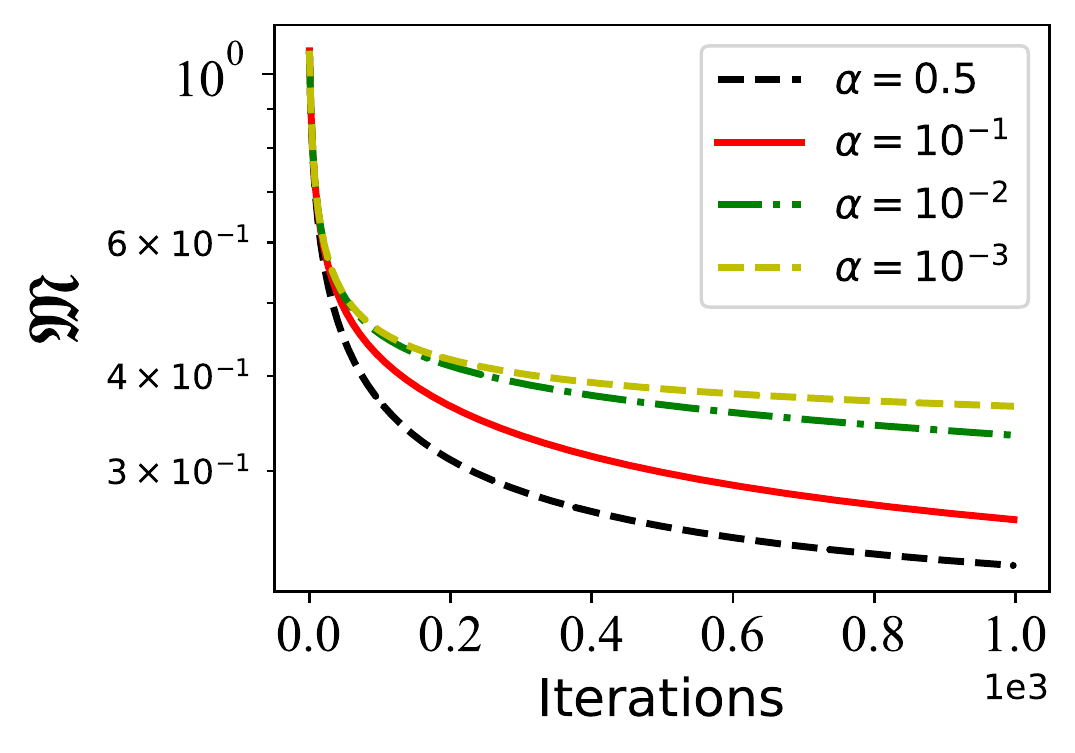}
		\includegraphics[width=0.22\textwidth]{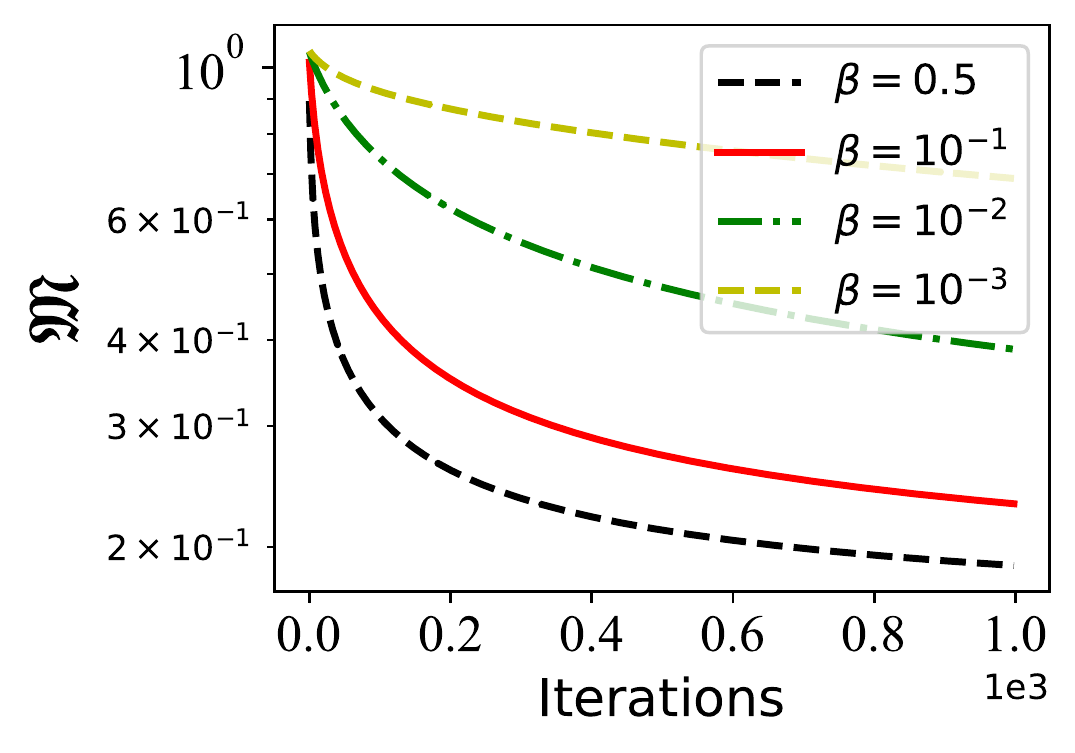}
	}
	\subfigure[The \algtname algorithm.]{
		\includegraphics[width=0.22\textwidth]{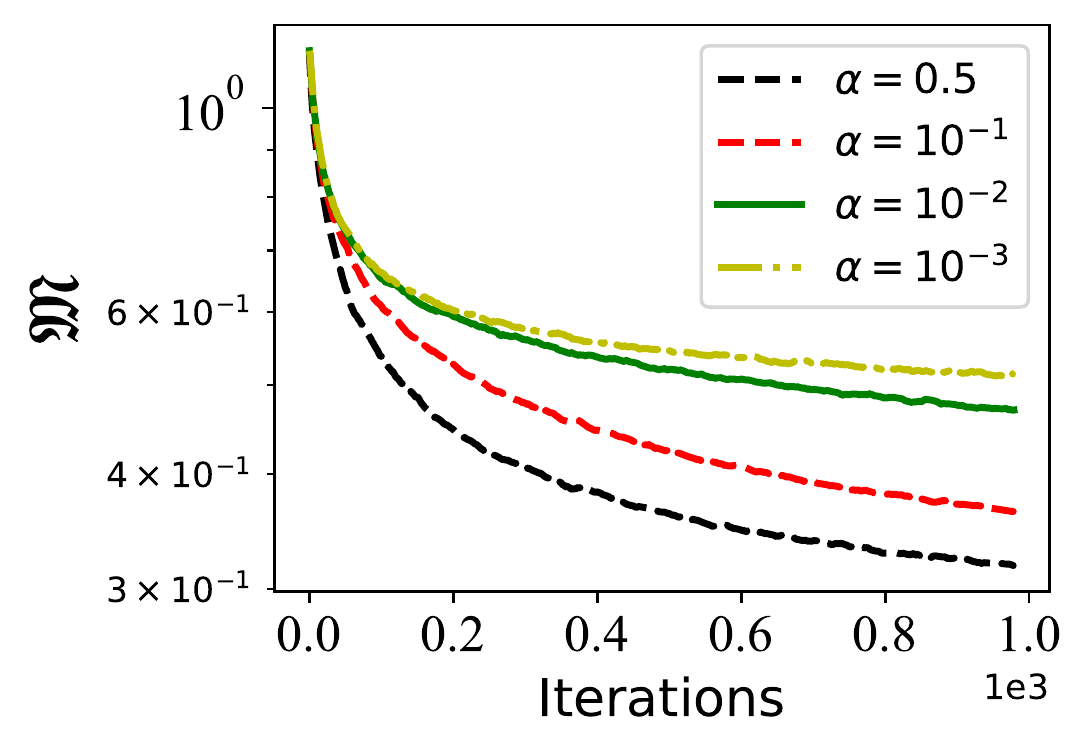}
		\includegraphics[width=0.22\textwidth]{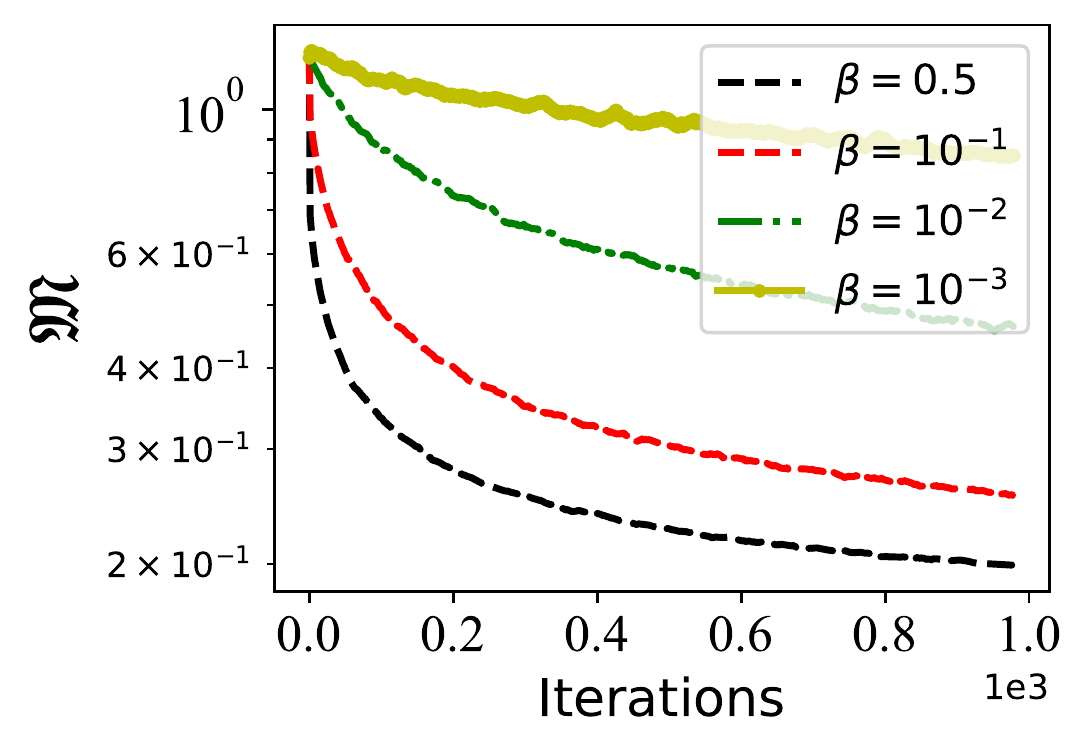}
	}
	\caption{Convergence performance with different learning rate on the MNIST dataset.}\label{Fig:step-size}
\end{figure*}

In this section, we conduct numerical experiments to demonstrate the performance of our \algname and \algtname algorithms on a meta-learning problem.
In particular, we evaluate and compare the performance of the proposed algorithms against two stochastic algorithms discussed below as the baselines:
\begin{list}{\labelitemi}{\leftmargin=1em \itemindent=-0.09em \itemsep=.2em}
\item {\em Decentralized Gradient-Tracking Stochastic Gradient Descent (GT-DSGD):} 
This algorithm can be viewed as a “stripped-down” version of \algname utilizing stochastic gradients instead of full gradients.
It is motivated by the GT-SGD algorithm~\cite{lu2019gnsd}. 
The algorithm performs local updates using gradients $\p_{i}(\x_{i,t},\y_{i,t}) $ and $\d_{i}(\x_{i,t},\y_{i,t})$ evaluated using the following stochastic gradient estimators at each agent:
$\p_{i}(\x_{i,t},\y_{i,t})\! = \!	\frac{1}{ \mathcal{S}}\sum_{i=1}^{\mathcal{S}} \!\big[\bar\nabla f_{i}(\x_{i,t},\y_{i,t}; \bar\xi_i )\big]$, and $\d_{i}(\x_{i,t},\y_{i,t}) \!\!= \!\!		\frac{1}{ \mathcal{S}}\sum_{i=1}^{\mathcal{S}} \big[	\nabla g_{i}(\x_{i,t},\y_{i,t}; \bar{\xi_i} )\big]$.

\item {\em Decentralized Stochastic Gradient Descent (D-SGD):}  This algorithm can be viewed as a simplified version of GT-SGD without utilizing gradient tracking \cite{jiang2017collaborative}. 
For D-SGD, each agent updates its local parameters as: $\x_{i,t} \!=\!\sum_{j \in \mathcal{N}_{i}} [\M]_{ij} \x_{j,t-1} - \alpha \frac{1}{ \mathcal{S}}\sum_{i=1}^{\mathcal{S}} \big[\bar\nabla f_{i}(\x_{i,t},\y_{i,t}; \bar\xi_i )\big]$ and 
$\y_{i,t} = \y_{i,t-1} - \beta 	\frac{1}{\mathcal{S}}\sum_{i=1}^{\mathcal{S}} \big[ \nabla g_{i}(\x_{i,t},\y_{i,t}; \bar{\xi_i} )\big]$. 
\end{list}

{\bf 1) Meta-learning Model and Datasets:} We evaluate the performance of \algname and \algtname across $m$ agents that collectively aim to solve the meta-learning problem with $m$ tasks $\{\mathcal{T}_{i}, i \in [m]\}$. 
We allocate each agent one task with $f_{i}$ defining the loss function of the $i$-th agent. 
For the loss function $f_i \left(\x, \y_{i} \right)$ corresponding to task $\mathcal{T}_{i}$, the parameter $x$ defines the model parameters shared among all agents, while $y_{i}$ denotes task specific parameters.   
The meta-learning problem aims at finding common parameters $\x$ for all tasks, and each task then adapts its own parameters $\y_{i}$ by minimizing local lower level losses. 
The objective function is $\min \frac{1}{m} \sum_{i=1}^{m} f_i\left(\x, \y_{i}^{*} \right)$, subject to $\y_i^{*}={\arg \min }_{\y_i}g_i(\x, \y_i)$. 
In practice, $\y_{i}$ denotes the parameters of the linear layer of a neural network, while $\x$ are the parameters of the remaining layers of the neural network. 
The inner problem includes a strongly convex regularizer ensuring that the inner function $g_i(\x, \y_i)$ is strongly-convex with respect to $\y_i$ and the outer function $f_i (\x, \y_{i}^{*} )$ is nonconvex with respect to $\x$ .

We evaluate the proposed algorithms on the MNIST~\cite{lecun2010mnist} and CIFAR-10~\cite{krizhevsky2009learning} datasets using a two-hidden-layer neural network
with 20 hidden units. 
The network topology $\mathcal{G}$ is generated by the Erd$\ddot{\text{o}}$s-R$\grave{\text{e}}$nyi random graph approach.
The consensus matrix is chosen as $\W = \I - \frac{2\mathbf{L}}{3\lambda_{\text{max}}(\mathbf{L})}$, where $\mathbf{L}$ is the Laplacian matrix of $\mathcal{G}$ and $\lambda_{\text{max}}(\mathbf{L})$ denotes the largest eigenvalue of $\mathbf{L}$. 

\textbf{2) Performance Comparison:} We set the constant learning rates $\alpha=\beta=0.5$, edge connectivity probability $p_c=0.5$, number of agents $m=5$, and each agent has $n=1000$ data samples and mini-batch size $q=\lceil\sqrt{n}\rceil$.
The network topology is shown in Fig.~\ref{Fig:5agent}.
In Fig.~\ref{Fig:performance_5agents}, we compare the performance of \algname and \algtname with the GT-SGD and SGD algorithms using the convergence metric $\mathfrak{M}$ on the MNIST and CIFAR-10 datasets. 
We note that both \algname and \algtname outperform the baseline algoirthms GT-SGD and SGD. 
In Fig.~\ref{Fig:performance_10agents}, we notice a similar behavior when the number of tasks (and agents) are increased to 10.  
More importantly, we note that in all experiment settings, the \algtname algorithm has the lowest sample complexity, which corroborates our theoretical findings. 
%

\textbf{3) Impact of the Edge Connectivity:} We conduct experiments using five-node network systems and compare the performance of the algorithms on three different network topologies. 
For the three settings, we choose the value of the edge connectivity probability from the discrete set $\{0.3, 0.5, 0.7\}$, while the rest of the settings stay the same as previous experiments.
As shown in Fig.~\ref{Fig:p_c}, we can observe that the convergence metric $\mathfrak{M}$ is relatively insensitive to the change of network topologies. 
Fig.~\ref{Fig:p_c} shows that the convergence metric $\mathfrak{M}$ slightly increases as the edge connectivity probability $p_c$ decreases (i.e., the network becomes sparser). 

\textbf{4) Impact of the Learning Rate:} 
In this experiment, we choose the learning rates $\alpha$ and $\beta$ from the discrete set $\{0.5, 0.1, 0.01, 0.001\}$.
We fix the number of agents to five and set $p_c = 0.5$.
The rest of the experiment settings stay the same as in previous experiments. 
As shown in Fig.~\ref{Fig:step-size}, larger values of learning rates $\alpha$ or $\beta$ lead to faster convergence rates for both \algname and \algtnamens algorithms.
 \vspace{-0.1in}
\section{Conclusion}\label{Section: conclusion}
In this paper, we developed two algorithms called \algname and \algtname for solving  decentralized non-convex-strongly-convex bilevel optimization problems. 
We showed that, to achieve an $\epsilon$-stationary point, \algname and \algtname have the communication complexity of $\mathcal{O}(\epsilon^{-1})$ and sample complexities of $\mathcal{O}(n \epsilon^{-1})$ and $\mathcal{O}(\sqrt{n} \epsilon^{-1})$, respectively.
Our numerical studies corroborate the theoretical performance of our proposed algorithms.
We note that our paper is the first to explore decentralized bilevel learning, which opens up several interesting directions for future research. 
For instance, it would be of interest to develop differentially-private algorithms for decentralized bilevel learning over networks.
Also, one can adopt compression techniques to further reduce communication costs, especially for large-scale deep learning models.

\bibliographystyle{acm}
\bibliography{Zhuqing}

\onecolumn
\appendix
\section{Proof}

	We note that 
	\begin{align}
	\overline{\mathbf{x}}_{t}&=\frac{1}{m} \sum_{i=1}^{m} \mathbf{x}_{i, t},\notag\\ \mathbf{x}_{t}&=\left[\mathbf{x}_{1, t}^{\top}, \cdots, \mathbf{x}_{m, t}^{\top}\right]^{\top},\notag\\
\p_t&=\left[\p_1(\x_{1,t},\y_{1,t})^{\top}, \cdots, \p_m(\x_{m,t},\y_{m,t})^{\top}\right]^{\top},\notag\\
\d_t&=\left[\d_1(\x_{1,t},\y_{1,t})^{\top}, \cdots, \d_m(\x_{m,t},\y_{m,t})^{\top}\right]^{\top},\notag\\
\bp_t&=\frac{1}{m} \sum_{i=1}^{m} \p_i(\x_{i,t},\y_{i,t}),\notag\\
\bd_t&=\frac{1}{m} \sum_{i=1}^{m} \d_i(\x_{i,t},\y_{i,t}),
\end{align}

{\em {Step 1:}} 
\begin{lem}[Descending Inequality for upper function]
Under asumptions \ref{assum1}-\ref{assum2}, the following descending inequality holds for both with Algorithm \ref{Algorithm_GD} and Algorithm \ref{Algorithm_VR}:
\begin{align}
\ell({\bx}_{t+1}) - \ell(\bx_{t}) \le &-\frac{\alpha}{2} \|\nabla \ell(\bx_t)\|^2 - (\frac{\alpha}{2} - \frac{L_{\ell}\alpha^2}{2})\|\bu_{t}\|^2  \notag\\&+ \frac{\alpha}{m}\sum_{i=1}^{m} L_{\ell} \| \bx_{t} - \x_{i.t}\|^2 +  \frac{2\alpha}{m} \sum_{i=1}^{m} L_f^2 \| \y_{i,t}^* - \y_{i,t} \|^2+ 2\alpha \| \frac{1}{m} \sum_{i=1}^{m} \bar\nabla f_i (\x_{i,t} ,\y_{i,t} ) - \bar\u_{ t} \|^2 ,
\end{align}
where $\y_{i,t}^* = \arg\max_{\y} \ell(\x_{i,t},\y)$.
\end{lem}

\begin{proof}

\begin{align}
&
\ell(\bx_{t+1}) - \ell(\bx_{t}) 
\stackrel{(a)}{\le}
\langle \nabla \ell(\bx_t), \bx_{t+1} - \bx_{t} \rangle + \frac{L_{\ell}}{2}\|\bx_{t+1} - \bx_{t}\|^2 \notag\\
\stackrel{(b)}{=} 
&
-\alpha \langle \nabla \ell(\bx_t), \bu_{t} \rangle + \frac{L_{\ell}\alpha^2}{2}\|\bu_{t}\|^2 \notag\\
= 
&
-\frac{\alpha }{2} \|\nabla \ell(\bx_t)\|^2 - (\frac{\alpha}{2} - \frac{L_{\ell}\alpha^2}{2})\|\bu_{t}\|^2  + \frac{\alpha}{2} \|\nabla \ell(\bx_t) - \bu_{t}\|^2 \notag\\
= 
&
-\frac{\alpha}{2} \|\nabla \ell(\bx_t)\|^2 - (\frac{\alpha}{2} - \frac{L_{\ell}\alpha^2}{2})\|\bu_{t}\|^2 + \frac{\alpha}{2} \|\nabla \ell(\bx_t) - \frac{1}{m}\sum_{i=1}^{m}\nabla l_i (\x_{i,t}) +  \frac{1}{m}\sum_{i=1}^{m}\nabla l_i(\x_{i,t}) - \bu_{t}\|^2 \notag\\
\stackrel{(c)}{\le}
&
-\frac{\alpha}{2} \|\nabla \ell(\bx_t)\|^2 - (\frac{\alpha}{2} - \frac{L_{\ell}\alpha^2}{2})\|\bu_{t}\|^2  + \frac{\alpha}{m}\sum_{i=1}^{m} \|\nabla l_i(\bx_{t}) - \nabla l_i (\x_{i,t})\|^2+  {\alpha}\| \frac{1}{m} \sum_{i=1}^{m}\nabla l_i (\x_{i,t})- \bar\u_{ t}\|^2 \notag\\
=
&
-\frac{\alpha}{2} \|\nabla \ell(\bx_t)\|^2 - (\frac{\alpha}{2} - \frac{L_{\ell}\alpha^2}{2})\|\bu_{t}\|^2  + \frac{\alpha}{m}\sum_{i=1}^{m} L_{\ell} \| \bx_{t} - \x_{i.t}\|^2+  \alpha \| \frac{1}{m} \sum_{i=1}^{m}\bar\nabla f_i (\x_{i,t} ,\y_{i,t}^*  )- \frac{1}{m} \sum_{i=1}^{m}\bar\nabla f_i (\x_{i,t} ,\y_{i,t}  )+\frac{1}{m} \sum_{i=1}^{m}\bar\nabla f_i (\x_{i,t} ,\y_{i,t} ) - \bar\u_{ t} \|^2  \notag\\
\stackrel{(d)}{\le}
&
-\frac{\alpha}{2} \|\nabla \ell(\bx_t)\|^2 - (\frac{\alpha}{2} - \frac{L_{\ell}\alpha^2}{2})\|\bu_{t}\|^2  + \frac{\alpha}{m}\sum_{i=1}^{m} L_{\ell} \| \bx_{t} - \x_{i.t}\|^2 +  \frac{2\alpha}{m} \sum_{i=1}^{m}\|\bar\nabla f_i (\x_{i,t} ,\y_{i,t}^*  )-\bar \nabla f_i (\x_{i,t},\y_{i,t}) \|^2+2 \alpha \| \frac{1}{m} \sum_{i=1}^{m} \bar\nabla f_i (\x_{i,t} ,\y_{i,t} ) - \bar\u_{ t} \|^2 \notag\\
\stackrel{(e)}{\le}&
-\frac{\alpha}{2} \|\nabla \ell(\bx_t)\|^2 - (\frac{\alpha}{2} - \frac{L_{\ell}\alpha^2}{2})\|\bu_{t}\|^2  + \frac{\alpha}{m}\sum_{i=1}^{m} L_{\ell} \| \bx_{t} - \x_{i.t}\|^2 +  \frac{2\alpha}{m} \sum_{i=1}^{m} L_f^2 \| \y_{i,t}^* - \y_{i,t} \|^2 +2 \alpha \| \frac{1}{m} \sum_{i=1}^{m} \bar\nabla f_i (\x_{i,t} ,\y_{i,t} ) - \bar\u_{ t} \|^2,
\end{align}
where (a) is because of Lipschitz continuous gradients of $\ell$, see Lemma \ref{lem:main}. (b) is because of the updating rules. (c), (d) is from the triangle inequality. (e) is from Lemma\ref{lem:main}.
\end{proof}

{\em {Step 2:}} 
\begin{lem}[Error Bound on $\y^*(\x)$]
Under assumptions \ref{assum1}-\ref{assum2}, the following inequality holds for both with Algorithm \ref{Algorithm_GD} and Algorithm \ref{Algorithm_VR}:
\begin{align}
\|\y_{i, t+1} - \y_{i, t+1}^*\|^2  
\le & (1+r)^2 (1-2\beta \frac{\mu_g L_g}{\mu_g+L_g }) \|\y_{i, t} - \y_{i, t}^*\|^2 +(\delta -1) (1+r)^2 (2\beta \frac{1}{\mu_g+L_g } -\beta^2)\| \v_{i,t} \|^2\notag\\&+(\frac{1}{\delta}-1 ) (1+r)^2 (2\beta \frac{1}{\mu_g+L_g } -\beta^2)\|  \nabla_{\y} g_{i}(\x_{i,t},\y_{i,t}) -\v_{i,t} \|^2\notag\\&+ (1+\frac{1}{r}) L_y^2 \|\x_{i, t+1}-\x_{i, t}\|^2
+ (1+r)(1+1/r) \beta^2 \| \v_{i,t} - \nabla_{\y} g_{i}(\x_{i,t},\y_{i,t})\|^2
\end{align}
\end{lem}
\begin{proof}

\begin{align}
&\|\y_{i, t+1} - \y_{i, t+1}^*\|^2  
= \|\y_{i, t+1} - \y_{i, t}^*+ \y_{i, t}^*- \y_{i, t+1}^*\|^2 \notag\\
\stackrel{(a)}{\le} &  (1+r)  \|\y_{i, t+1} - \y_{i, t}^*\|^2 + (1+\frac{1}{r}) \|\y_{i, t}^*- \y_{i, t+1}^*\|^2\notag\\
\stackrel{(b)}{=}  & (1+r)  \|\y_{i, t} -\beta\v_{i,t} +\beta \nabla_{\y} g_{i}(\x_{i,t},\y_{i,t}) -\beta \nabla_{\y} g_{i}(\x_{i,t},\y_{i,t}) - \y_{i, t}^*\|^2 + (1+\frac{1}{r}) \|\y_{i, t}^*- \y_{i, t+1}^*\|^2\notag\\
\stackrel{(c)}{=}  & (1+r) (1+r)  \|\y_{i, t}  -\beta \nabla_{\y} g_{i}(\x_{i,t},\y_{i,t}) - \y_{i, t}^*\|^2+ (1+1/r) (1+r)\beta^2 \| \v_{i,t} - \nabla_{\y} g_{i}(\x_{i,t},\y_{i,t})\|^2+ (1+\frac{1}{r}) \|\y_{i, t}^*- \y_{i, t+1}^*\|^2\notag\\
\stackrel{(d)}{\leq}   &  (1+r)^2 (1-2\beta \frac{\mu_g L_g}{\mu_g+L_g }) \|\y_{i, t} - \y_{i, t}^*\|^2 -  (1+r)^2 (2\beta \frac{1}{\mu_g+L_g } -\beta^2)\| \v_{i,t}+ \nabla_{\y} g_{i}(\x_{i,t},\y_{i,t}) -\v_{i,t} \|^2+ (1+\frac{1}{r}) \|\y_{i, t}^*- \y_{i, t+1}^*\|^2\notag\\& + (1+r)(1+1/r) \beta^2 \| \v_{i,t} - \nabla_{\y} g_{i}(\x_{i,t},\y_{i,t})\|^2\notag\\
\stackrel{(e)}{\leq}  & (1+r)^2 (1-2\beta \frac{\mu_g L_g}{\mu_g+L_g }) \|\y_{i, t} - \y_{i, t}^*\|^2 - (1+r)^2 (2\beta \frac{1}{\mu_g+L_g } -\beta^2)\| \v_{i,t}+ \nabla_{\y} g_{i}(\x_{i,t},\y_{i,t}) -\v_{i,t} \|^2+ (1+\frac{1}{r}) L_y^2 \|\x_{i, t+1}- \x_{i, t}\|^2\notag\\&+ (1+1/r) (1+r)\beta^2 \| \v_{i,t} - \nabla_{\y} g_{i}(\x_{i,t},\y_{i,t})\|^2\notag\\
\stackrel{(f)}{=} & (1+r)^2(1-2\beta \frac{\mu_g L_g}{\mu_g+L_g }) \|\y_{i, t} - \y_{i, t}^*\|^2 - (1+r)^2 (2\beta \frac{1}{\mu_g+L_g } -\beta^2)\| \v_{i,t}+ \nabla_{\y} g_{i}(\x_{i,t},\y_{i,t}) -\v_{i,t} \|^2\notag\\&+ (1+\frac{1}{r}) L_y^2 \|\x_{i, t+1}-\x_{i, t}\|^2+ (1+r) (1+1/r)\beta^2 \| \v_{i,t} - \nabla_{\y} g_{i}(\x_{i,t},\y_{i,t})\|^2\notag\\
\stackrel{(g)}{\leq} & (1+r)^2 (1-2\beta \frac{\mu_g L_g}{\mu_g+L_g }) \|\y_{i, t} - \y_{i, t}^*\|^2 - (1+r)^2 (2\beta \frac{1}{\mu_g+L_g } -\beta^2)\| \v_{i,t}+ \nabla_{\y} g_{i}(\x_{i,t},\y_{i,t}) -\v_{i,t} \|^2+ (1+\frac{1}{r}) L_y^2  \|\x_{i, t+1}-\x_{i, t}\|^2 \notag\\&
+ (1+r)(1+1/r) \beta^2 \| \v_{i,t} - \nabla_{\y} g_{i}(\x_{i,t},\y_{i,t})\|^2\notag\\
\stackrel{(h)}{\leq} & (1+r)^2 (1-2\beta \frac{\mu_g L_g}{\mu_g+L_g }) \|\y_{i, t} - \y_{i, t}^*\|^2 +(\delta -1) (1+r)^2 (2\beta \frac{1}{\mu_g+L_g } -\beta^2)\| \v_{i,t} \|^2+(\frac{1}{\delta}-1 ) (1+r)^2 (2\beta \frac{1}{\mu_g+L_g } -\beta^2)\|  \nabla_{\y} g_{i} \x_{i,t},\y_{i,t}) -\v_{i,t} \|^2\notag\\&+ (1+\frac{1}{r}) L_y^2 \|\x_{i, t+1}-\x_{i, t}\|^2
+ (1+r)(1+1/r) \beta^2 \| \v_{i,t} - \nabla_{\y} g_{i}(\x_{i,t},\y_{i,t})\|^2, 
\end{align}
where (a) and (h) are because of Young's inequality.
(b) and (f) are because of the updating rules of $\y_{i,t}$ and $\x_{i,t}$. (d) is due to the coercive property of lower level function $g_i(\x_i,\y_i)$ and $g_i(\x_i,\y_i^*)=0$. (e) is from Lemma \ref{lem:main}. (c) and (g) are because of the triangle inequality.

\end{proof}

{\em{Step 3:}}
\begin{lem}[Iterates Contraction]
The following contraction properties of the iterates hold:
\begin{align}
\|\x_t-\ot\bx_t\|^2 \le & (1+c_1)\lambda^2\|\x_{t-1} -\ot\bx_{t-1} \|^2 + (1+\frac{1}{c_1}) \alpha^2\|\u_{t-1}-\ot \bu_{t-1}\|^2, \\
\|\u_t-\ot\bu_t\|^2 \le& 
(1+ c_2)\lambda^2\|\u_{t-1} - \ot \bu_{t-1}\|^2 + (1 + \frac{1}{ c_2})\|\p_t-\p_{t-1} \|^2,
\end{align}
where $c_1$ and $c_2$ are arbitrary positive constants.
Additionally, we have 
\begin{align}\label{eqs30}
\|\x_t-\x_{t-1}\|^2 
& \le
8\|(\x_{t-1} - \ot\bx_{t-1}) \|^2 + 4\alpha^2 \|\u_{t-1} - \ot\bu_{t-1}\|^2 + 4\alpha^2m\|\bu_{t-1}\|^2.\notag\\
\|\y_t-\y_{t-1}\|^2 &
\le		 \beta^2\|\v_{t-1}\|^2,
\end{align}
\end{lem}
\begin{proof}
 Define $\Mt = \M \otimes \I_m$. First for the iterates $\x_t$, we have the following contraction:
\begin{align}\label{eqs32}
\|\Mt\x_{t} -\ot\bx_{t} \|^2 = \|\Mt(\x_{t} -\ot\bx_{t}) \|^2 \le \lambda^2\|\x_{t} -\ot\bx_{t}\|^2.
\end{align}
This is because $\x_{t} -\ot\x_{t}$ is orthogonal $\1,$ which is the eigenvector corresponding to the largest eigenvalue of $\Mt,$ and $\lambda = \max\{|\lambda_2|,|\lambda_m|\}.$
Recall that $\bx_t = \bx_{t-1} - \alpha\bu_{t-1},$ hence,
\begin{align}
&
\|\x_t-\ot\bx_t\|^2
=
\|\Mt\x_{t-1} -\alpha\u_{t-1}-\ot(\bx_{t-1} -\alpha\bu_{t-1})\|^2 \notag\\
&
\stackrel{(a)}{\leq}
(1+c_1)\|\Mt\x_{t-1} -\ot\bx_{t-1} \|^2 + (1+\frac{1}{c_1}) \alpha^2\|\u_{t-1}-\ot \bu_{t-1}\|^2 \notag\\
&
\stackrel{(b)}{\leq}
(1+c_1)\lambda^2\|\x_{t-1} -\ot\bx_{t-1} \|^2 + (1+\frac{1}{c_1}) \alpha^2\|\u_{t-1}-\ot \bu_{t-1}\|^2, 
\end{align}
where (a) is becase of triangle inequality and (b) is from eqs.(\ref{eqs32}).

For $\u_t$, we have
\begin{align}
&\|\u_t-\ot\bu_t\|^2 
\notag\\=&
\|\Mt \u_{t-1} + \p_t-\p_{t-1}  - \ot \big(\bu_{t-1} +\bp_t-\bp_{t-1} \big)\|^2 \notag\\
\le&
(1+ c_2)\lambda^2\|\u_{t-1} - \ot \u_{t-1}\|^2+ (1 + \frac{1}{ c_2})\|\p_t-\p_{t-1}-\ot\big(\bp_t-\bp_{t-1} \big)\|^2\notag\\
\le&
(1+ c_2)\lambda^2\|\u_{t-1} - \ot \bu_{t-1}\|^2 + (1 + \frac{1}{ c_2})\|\big(\I - \frac{1}{n}(\1\1^\top)\otimes\I\big)\big(\p_t-\p_{t-1}\big)\|^2\notag\\
\stackrel{(a)}{\le}&
(1+ c_2)\lambda^2\|\u_{t-1} - \ot \bu_{t-1}\|^2 + (1 + \frac{1}{ c_2})\|\p_t-\p_{t-1} \|^2\notag\\
\end{align}
where (a) is due to $\|\I-\frac{1}{m}(\1\1^\top)\otimes \I \|\le 1.$ 

According to the updating
\begin{align}
&
\|\x_t-\x_{t-1}\|^2 
= 
\|\Mt \x_{t-1} -\alpha \u_{t-1} - \x_{t-1}\|^2 \notag\\
=
&
\|(\Mt -\I)\x_{t-1} - \alpha\u_{t-1}\|^2
\le 2\|(\Mt -\I)\x_{t-1} \|^2 + 2\alpha^2 \|\u_{t-1}\|^2 \notag\\
=
&
2\|(\Mt -\I)(\x_{t-1} - \ot\bx_{t-1}) \|^2 + 2\alpha^2 \|\u_{t-1}\|^2 \notag\\
\stackrel{}{\le}
&
8\|(\x_{t-1} - \ot\bx_{t-1}) \|^2 + 4\alpha^2 \|\u_{t-1} - \ot\bu_{t-1}\|^2 + 4\alpha^2m\|\bu_{t-1}\|^2 
\end{align}
And also, 
\begin{align}
	&
	\|\y_t-\y_{t-1}\|^2 
\le \beta^2\|\v_{t-1}\|^2
.
\end{align}

\end{proof}

{\em{Step 4:}}
With the results from Step 1, we have
\begin{align}\label{l1}
\ell({\bx}_{t+1}) - \ell(\bx_{t}) \le &-\frac{\alpha}{2} \|\nabla \ell(\bx_t)\|^2 - (\frac{\alpha}{2} - \frac{L_{\ell}\alpha^2}{2})\|\bu_{t}\|^2  \notag\\&+ \frac{\alpha}{m}\sum_{i=1}^{m} L_{\ell} \| \bx_{t} - \x_{i.t}\|^2 +  \frac{2\alpha}{m} \sum_{i=1}^{m} L_f^2 \| \y_{i,t}^* - \y_{i,t} \|^2+ 2\alpha \| \frac{1}{m} \sum_{i=1}^{m} \bar\nabla f_i (\x_{i,t} ,\y_{i,t} ) - \bar\u_{ t} \|^2\notag\\
\le &-\frac{\alpha}{2} \|\nabla \ell(\bx_t)\|^2 - (\frac{\alpha}{2} - \frac{L_{\ell}\alpha^2}{2})\|\bu_{t}\|^2  \notag\\&+ \frac{\alpha L_{\ell}  }{m}\| \x_{t} - \ot\bx_{t}  \|^2 +  \frac{2\alpha}{m}  L_f^2 \| \y_{t}^* - \y_{t} \|^2+ 2\alpha \| \frac{1}{m} \sum_{i=1}^{m} \bar\nabla f_i (\x_{i,t} ,\y_{i,t} ) - \bar\u_{ t} \|^2
\end{align}

With the results from Step 2, we have
\begin{align}\label{l2}
	\|\y_{ t+1} &- \y_{  t+1}^*\|^2  -\|\y_{ t} - \y_{ t}^*\|^2  
	\le -[1-(1+r)^2 (1-2\beta \frac{\mu_g L_g}{\mu_g+L_g }) ]\|\y_{t} - \y_{t}^*\|^2 +(\frac{1}{\delta}-1 ) (1+r)^2 (2\beta \frac{1}{\mu_g+L_g } -\beta^2)\|  \nabla_{\y} g_{i}(\x_{i,t},\y_{i,t}) -\v_{i,t} \|^2\notag\\&+(\delta-1)(1+r)^2 (2\beta \frac{1}{\mu_g+L_g } -\beta^2)\| \v_{t} \|^2+  (1+\frac{1}{r}) L_y^2 \|\x_{ t+1}-\x_{t}\|^2+(1+r)(1+1/r) \beta^2\sum_{i=1}^m \| \v_{i,t} - \nabla_{\y} g_{i}(\x_{i,t},\y_{i,t})\|^2
\end{align}

Combing \ref{l1} and \ref{l2} and telescoping the inequality and with $r\in(0,1]$, we have
\begin{align}\label{eqs5}
&\ell(\bx_{T+1}) - \ell(\bx_{0}) + \frac{1-\lambda}{32(1+1/r)L_y^2 } \big[\|\y_{T+1} - \y_{T+1}^*\|^2  - \|\y_0^* - \y_0\|^2\big] \notag\\
	\le
&-\frac{\alpha}{2} \sum_{t=0}^{T} \|\nabla \ell(\bx_t)\|^2 - (\frac{\alpha}{2} - \frac{L_{\ell}\alpha^2}{2})\sum_{t=0}^{T} \|\bu_{t}\|^2  + \frac{\alpha L_{\ell}  }{m}\sum_{t=0}^{T} \| \x_{t} - \ot\bx_{t}  \|^2 +2 \alpha \sum_{t=0}^{T} \| \frac{1}{m} \sum_{i=1}^{m} \bar\nabla f_i (\x_{i,t} ,\y_{i,t} ) - \bar\u_{ t} \|^2\notag\\&+  \frac{2\alpha}{m}  L_f^2 \sum_{t=0}^{T} \| \y_{t}^* - \y_{t} \|^2
  - \frac{1-\lambda}{32(1+1/r)L_y^2 } [1-(1+3r)(1-2\beta \frac{\mu_g L_g}{\mu_g+L_g }) ] \sum_{t=0}^{T} \|\y_{t} - \y_{t}^*\|^2 \notag\\&+ \frac{1-\lambda}{32(1+1/r)L_y^2 }(\delta-1)(1+r)^2 (2\beta \frac{1}{\mu_g+L_g } -\beta^2)\sum_{t=0}^{T} \| \v_{t} \|^2\notag\\
  &+  \frac{1-\lambda}{32(1+1/r)L_y^2 }(1+\frac{1}{r}) L_y^2\sum_{t=0}^T \|\x_{ t+1}-\x_{t}\|^2 \notag\\
  &+ \frac{1-\lambda}{32(1+1/r)L_y^2 }(1+r)(1+1/r) \beta^2\sum_{i=1}^m \sum_{t=0}^{T}\| \v_{i,t} - \nabla_{\y} g_{i}(\x_{i,t},\y_{i,t})\|^2\notag\\
  & + \frac{1-\lambda}{32(1+1/r)L_y^2 }(\frac{1}{\delta}-1 ) (1+r)^2 (2\beta \frac{1}{\mu_g+L_g } -\beta^2)\sum_{t=0}^{T} \|  \nabla_{\y} g_{i}(\x_{i,t},\y_{i,t}) -\v_{i,t} \|^2\notag\\
\overset{(a)}{\leq}
&-\frac{\alpha}{2} \sum_{t=0}^{T} \|\nabla \ell(\bx_t)\|^2 - (\frac{\alpha}{2} - \frac{L_{\ell}\alpha^2}{2})\sum_{t=0}^{T} \|\bu_{t}\|^2  + \frac{\alpha L_{\ell}  }{m}\sum_{t=0}^{T} \| \x_{t} - \ot\bx_{t}  \|^2 +2 \alpha \sum_{t=0}^{T} \| \frac{1}{m} \sum_{i=1}^{m} \bar\nabla f_i (\x_{i,t} ,\y_{i,t} ) - \bar\u_{ t} \|^2\notag\\&+  \frac{2\alpha}{m}  L_f^2 \sum_{t=0}^{T} \| \y_{t}^* - \y_{t} \|^2
- \frac{1-\lambda}{32(1+1/r)L_y^2 } [1-(1+3r)(1-2\beta \frac{\mu_g L_g}{\mu_g+L_g }) ] \sum_{t=0}^{T} \|\y_{t} - \y_{t}^*\|^2 \notag\\&+ \frac{1-\lambda}{32(1+1/r)L_y^2 }(\delta-1)(1+r)^2 (2\beta \frac{1}{\mu_g+L_g } -\beta^2)\sum_{t=0}^{T} \| \v_{t} \|^2\notag\\
&+  \frac{1-\lambda}{32(1+1/r)L_y^2 }(1+\frac{1}{r}) L_y^2\sum_{t=0}^T(8\|(\x_{t} - \ot\bx_{t}) \|^2 + 4\alpha^2 \|\u_{t} - \ot\bu_{t}\|^2 + 4\alpha^2m\|\bu_{t}\|^2 ) \notag\\
&+ \frac{1-\lambda}{32(1+1/r)L_y^2 }(1+r)(1+1/r) \beta^2\sum_{i=1}^m \sum_{t=0}^{T}\| \v_{i,t} - \nabla_{\y} g_{i}(\x_{i,t},\y_{i,t})\|^2\notag\\
& + \frac{1-\lambda}{32(1+1/r)L_y^2 }(\frac{1}{\delta}-1 ) (1+r)^2 (2\beta \frac{1}{\mu_g+L_g } -\beta^2)\sum_{t=0}^{T} \|  \nabla_{\y} g_{i}(\x_{i,t},\y_{i,t}) -\v_{i,t} \|^2\notag\\
 =&-\frac{\alpha}{2} \sum_{t=0}^{T} \|\nabla \ell(\bx_t)\|^2 - (\frac{\alpha}{2} - \frac{L_{\ell}\alpha^2}{2}- 4(1+\frac{1}{r}) L_y^2\alpha^2m \frac{1-\lambda}{32(1+1/r)L_y^2 } )\sum_{t=0}^{T} \|\bu_{t}\|^2  \notag\\&+ (\frac{\alpha L_{\ell}  }{m}  +(1+\frac{1}{r}) L_y^2 \frac{1-\lambda}{4(1+1/r)L_y^2 } )\sum_{t=0}^{T} \| \x_{t} - \ot\bx_{t}  \|^2 \notag\\
 &+  [\frac{2\alpha}{m}  L_f^2  +   \frac{1-\lambda}{32(1+1/r)L_y^2 } 3r -  \frac{1-\lambda}{32(1+1/r)L_y^2 }(1+3r)( 2\beta \frac{\mu_g L_g}{\mu_g+L_g }) ]\sum_{t=0}^{T}\|\y_{t} - \y_{t}^*\|^2\notag\\&+ 2\alpha \sum_{t=0}^{T} \| \frac{1}{m} \sum_{i=1}^{m} \bar\nabla f_i (\x_{i,t} ,\y_{i,t} ) - \bar\u_{ t} \|^2 \notag\\
 &+ \frac{1-\lambda}{32(1+1/r)L_y^2 }(\delta-1)(1+r)^2 (2\beta \frac{1}{\mu_g+L_g } -\beta^2)\sum_{t=0}^{T}\| \v_{t} \|^2+ 4(1+\frac{1}{r}) L_y^2\alpha^2 \frac{1-\lambda}{32(1+1/r)L_y^2 } (\sum_{t=0}^{T}\|\u_{t}- \ot\bu_{t}\|^2 )\notag\\
& + \frac{1-\lambda}{32(1+1/r)L_y^2 }[(1+1/r)(1+r)\beta^2+(\frac{1}{\delta}-1 ) (1+r)^2 (2\beta \frac{1}{\mu_g+L_g } -\beta^2)]\sum_{i=1}^m\sum_{t=0}^{T} \| \v_{i,t} - \nabla_{\y} g_{i}(\x_{i,t},\y_{i,t})\|^2,
\end{align}
where (a) follows from Eqs.\eqref{eqs30}.

\textbf{Proof of Theorem \ref{thm1}:}
For Algortihm \ref{Algorithm_GD}, we have $ \| \frac{1}{m} \sum_{i=1}^{m} \bar\nabla f_i (\x_{i,t} ,\y_{i,t} ) - \bar\u_{ t} \|^2 =0 $, $ \| \v_{i,t} - \nabla_{\y} g_{i}(\x_{i,t},\y_{i,t})\|^2=0$, and 
\begin{align}
	\|\p_t-\p_{t-1} \|^2 {\le}&
	L_K^2  (\|\x_{t}- \x_{t-1} \|^2+\beta^2 \|\v_{t-1} \|^2).
\end{align}

With the results from Step 3, we have
\begin{align}	\|\x_t-\ot\bx_t\|^2& \le  (1+c_1)\lambda^2\|\x_{t-1} -\ot\bx_{t-1} \|^2 + (1+\frac{1}{c_1}) \alpha^2\|\u_{t-1}-\ot \bu_{t-1}\|^2,\notag\\
	\|\u_t-\ot\bu_t\|^2 & \le((1+c_2)\lambda^2 )\|\u_{t-1} -\ot\bu_{t-1} \|^2 \notag\\&+ (1 + \frac{1}{c_2})L_K^2  (\|\x_{t}- \x_{t-1} \|^2+\beta^2 \|\v_{t-1} \|^2). 
\end{align}

Then, we have
\begin{align}\label{l3}
	&	\|\x_{T+1}-\ot\bx_{T+1}\|^2-	\|\x_{0}-\ot\bx_{0}\|^2 \notag\\ \le & ((1+c_1)\lambda^2-1) \sum_{t=1}^{T+1} \|\x_{t-1} -\ot\bx_{t-1} \|^2 + (1+\frac{1}{c_1}) \alpha^2\sum_{t=1}^{T+1}\|\u_{t-1}-\ot \bu_{t-1}\|^2.
\end{align}
\begin{align}\label{l4}
&	\|\u_{T+1}-\ot\bu_{T+1}\|^2  -	\|\u_{0}-\ot\bu_{0}\|^2 
\notag\\
\le& ((1+c_2)\lambda^2 -1)\sum_{t=1}^{T+1}\|\u_{t-1} -\ot\bu_{t-1} \|^2 +(1 + \frac{1}{c_2})L_K^2 \sum_{t=1}^{T+1} (\|\x_{t}- \x_{t-1} \|^2+\beta^2 \|\v_{t-1} \|^2). 
\end{align}

Combing \eqref{l3},\eqref{l4} and \eqref{eqs5}, with $\delta=r$, $ \| \frac{1}{m} \sum_{i=1}^{m} \bar\nabla f_i (\x_{i,t} ,\y_{i,t} ) - \bar\u_{ t} \|^2 =0 $, $ \| \v_{i,t} - \nabla_{\y} g_{i}(\x_{i,t},\y_{i,t})\|^2=0$, we have
\begin{align}
	&
	\ell(\bx_{T+1}) - \ell(\bx_{0}) +  \frac{1-\lambda}{32(1+1/r)L_y^2 } \big[\|\y_{T+1} - \y_{T+1}^*\|^2  - \|\y_0^* - \y_0\|^2\big] \notag\\&+	[\|\x_{T+1}-\ot\bx_{T+1}\|^2-	\|\x_{0}-\ot\bx_{0}\|^2 ]+
	\alpha [	\|\u_{T+1}-\ot\bu_{T+1}\|^2  -	\|\u_{0}-\ot\bu_{0}\|^2 ]
	\notag\\
	\leq 
	&-\frac{\alpha}{2} \sum_{t=0}^{T} \|\nabla \ell(\bx_t)\|^2 - (\frac{\alpha}{2} - \frac{L_{\ell}\alpha^2}{2}- 4(1+\frac{1}{r}) L_y^2\alpha^2m \frac{1-\lambda}{32(1+1/r)L_y^2 } )\sum_{t=0}^{T} \|\bu_{t}\|^2  \notag\\&+ (\frac{\alpha L_{\ell}  }{m}  +(1+\frac{1}{r}) L_y^2 \frac{1-\lambda}{4(1+1/r)L_y^2 } )\sum_{t=0}^{T} \| \x_{t} - \ot\bx_{t}  \|^2 \notag\\
	&+  [\frac{2\alpha}{m}  L_f^2  +   \frac{1-\lambda}{32(1+1/r)L_y^2 } 3r -  \frac{1-\lambda}{32(1+1/r)L_y^2 }(1+3r)( 2\beta \frac{\mu_g L_g}{\mu_g+L_g }) ]\sum_{t=0}^{T}\|\y_{t} - \y_{t}^*\|^2\notag\\
	&+ \frac{1-\lambda}{32(1+1/r)L_y^2 }(r-1)(1+r)^2 (2\beta \frac{1}{\mu_g+L_g } -\beta^2)\sum_{t=0}^{T}\| \v_{t} \|^2+ 4(1+\frac{1}{r}) L_y^2\alpha^2 \frac{1-\lambda}{32(1+1/r)L_y^2 } (\sum_{t=0}^{T}\|\u_{t}- \ot\bu_{t}\|^2 )
	\notag\\
&	+((1+c_1)\lambda^2-1) \sum_{t=1}^{T+1} \|\x_{t-1} -\ot\bx_{t-1} \|^2 + (1+\frac{1}{c_1}) \alpha^2\sum_{t=1}^{T+1}\|\u_{t-1}-\ot \bu_{t-1}\|^2\notag\\
&+\alpha((1+c_2)\lambda^2 -1)\sum_{t=1}^{T+1}\|\u_{t-1} -\ot\bu_{t-1} \|^2 +\alpha(1 + \frac{1}{c_2})L_K^2 \sum_{t=1}^{T+1} (\|\x_{t}- \x_{t-1} \|^2+\beta^2 \|\v_{t-1} \|^2)\notag\\
\overset{(a)}{\leq} &	-\frac{\alpha}{2} \sum_{t=0}^{T} \|\nabla \ell(\bx_t)\|^2 - (\frac{\alpha}{2} - \frac{L_{\ell}\alpha^2}{2}- 4(1+\frac{1}{r}) L_y^2\alpha^2m \frac{1-\lambda}{32(1+1/r)L_y^2 } )\sum_{t=0}^{T} \|\bu_{t}\|^2  \notag\\&+ (\frac{\alpha L_{\ell}  }{m}  +(1+\frac{1}{r}) L_y^2 \frac{1-\lambda}{4(1+1/r)L_y^2 } )\sum_{t=0}^{T} \| \x_{t} - \ot\bx_{t}  \|^2 \notag\\
&+  [\frac{2\alpha}{m}  L_f^2  +   \frac{1-\lambda}{32(1+1/r)L_y^2 } 3r -  \frac{1-\lambda}{32(1+1/r)L_y^2 }(1+3r)( 2\beta \frac{\mu_g L_g}{\mu_g+L_g }) ]\sum_{t=0}^{T}\|\y_{t} - \y_{t}^*\|^2\notag\\
&+ \frac{1-\lambda}{32(1+1/r)L_y^2 }(r-1)(1+r)^2 (2\beta \frac{1}{\mu_g+L_g } -\beta^2)\sum_{t=0}^{T}\| \v_{t} \|^2+ 4(1+\frac{1}{r}) L_y^2\alpha^2 \frac{1-\lambda}{32(1+1/r)L_y^2 } (\sum_{t=0}^{T}\|\u_{t}- \ot\bu_{t}\|^2 )
\notag\\
&	+((1+c_1)\lambda^2-1) \sum_{t=1}^{T+1} \|\x_{t-1} -\ot\bx_{t-1} \|^2 + (1+\frac{1}{c_1}) \alpha^2\sum_{t=1}^{T+1}\|\u_{t-1}-\ot \bu_{t-1}\|^2\notag\\
&+\alpha((1+c_2)\lambda^2 -1)\sum_{t=1}^{T+1}\|\u_{t-1} -\ot\bu_{t-1} \|^2 \notag\\
&+\alpha(1 + \frac{1}{c_2})L_K^2 \sum_{t=1}^{T+1} (8\|(\x_{t-1} - \ot\bx_{t-1}) \|^2 + 4\alpha^2 \|\u_{t-1} - \ot\bu_{t-1}\|^2 + 4\alpha^2m\|\bu_{t-1}\|^2+\beta^2 \|\v_{t-1} \|^2)\notag\\
= &	-\frac{\alpha}{2} \sum_{t=0}^{T} \|\nabla \ell(\bx_t)\|^2 - (\frac{\alpha}{2} - \frac{L_{\ell}\alpha^2}{2}- 4(1+\frac{1}{r}) L_y^2\alpha^2m \frac{1-\lambda}{32(1+1/r)L_y^2 }-\alpha(1 + \frac{1}{c_2})L_K^2 4 \alpha^2 m )\sum_{t=0}^{T} \|\bu_{t}\|^2  \notag\\&+ (8\alpha(1 + \frac{1}{c_2})L_K^2+((1+c_1)\lambda^2-1)+ \frac{\alpha L_{\ell}  }{m}  +(1+\frac{1}{r}) L_y^2 \frac{1-\lambda}{4(1+1/r)L_y^2 } )\sum_{t=0}^{T} \| \x_{t} - \ot\bx_{t}  \|^2 \notag\\
&+  [ \frac{2\alpha}{m}  L_f^2  +   \frac{1-\lambda}{32(1+1/r)L_y^2 } 3r -  \frac{1-\lambda}{32(1+1/r)L_y^2 }(1+3r)( 2\beta \frac{\mu_g L_g}{\mu_g+L_g }) ]\sum_{t=0}^{T}\|\y_{t} - \y_{t}^*\|^2\notag\\
&+[ \frac{1-\lambda}{32(1+1/r)L_y^2 }(r-1)(1+r)^2 (2\beta \frac{1}{\mu_g+L_g } -\beta^2)+\alpha(1 + \frac{1}{c_2})L_K^2\beta^2 ] \sum_{t=0}^{T}\| \v_{t} \|^2\notag\\
&+[ 4(1+\frac{1}{r}) L_y^2\alpha^2 \frac{1-\lambda}{32(1+1/r)L_y^2 }+4\alpha^2 \alpha(1 + \frac{1}{c_2})L_K^2+ \alpha((1+c_2)\lambda^2 -1) +(1+\frac{1}{c_1}) \alpha^2] (\sum_{t=0}^{T}\|\u_{t}- \ot\bu_{t}\|^2 ),
\end{align}
where (a) is from eqs.(\ref{eqs30}). Next, choosing $c_1=c_2=\frac{1}{\lambda}-1$, we have
\begin{align}\label{l5}
	&
	\ell(\bx_{T+1}) - \ell(\bx_{0}) +\frac{1-\lambda}{32(1+1/r)L_y^2 } \big[\|\y_{T+1} - \y_{T+1}^*\|^2  - \|\y_0^* - \y_0\|^2\big] \notag\\&+	[\|\x_{T+1}-\ot\bx_{T+1}\|^2-	\|\x_{0}-\ot\bx_{0}\|^2 ]+
	\alpha [	\|\u_{T+1}-\ot\bu_{T+1}\|^2  -	\|\u_{0}-\ot\bu_{0}\|^2 ]
	\notag\\
	\le
	&-\frac{\alpha}{2} \sum_{t=0}^{T} \|\nabla \ell(\bx_t)\|^2 + C_1\sum_{t=0}^{T} \|\bu_{t}\|^2  
	+C_2\sum_{t=0}^{T} \| \x_{t} - \ot\bx_{t}  \|^2 +  C_3 \sum_{t=0}^{T}\|\y_{t} - \y_{t}^*\|^2 \notag\\
	&+C_4\sum_{t=0}^{T}\| \v_{t} \|^2+ C_5\sum_{t=0}^{T}\|\u_{t}- \ot\bu_{t}\|^2,
\end{align}
where the constants are
\begin{align}
&C_1=-(\frac{\alpha}{2} -4\alpha^2m \frac{1}{1-\lambda} L_K^2	\alpha - \frac{L_{\ell}\alpha^2}{2}- 4(1+\frac{1}{r}) L_y^2\alpha^2m  \frac{1-\lambda}{32(1+1/r)L_y^2 } ),\\
&C_2=\big[8\frac{1}{1-\lambda} L_K^2	\alpha +\lambda-1+ \frac{\alpha L_{\ell}  }{m} +(1+\frac{1}{r}) L_y^2 \frac{1-\lambda}{4(1+1/r)L_y^2 }\big],\\
&C_3=[\frac{2\alpha}{m}  L_f^2  +   \frac{1-\lambda}{32(1+1/r)L_y^2 }3r -  \frac{1-\lambda}{32(1+1/r)L_y^2 }(1+3r)( 2\beta \frac{\mu_g L_g}{\mu_g+L_g }) ],\\
&C_4=\big[\frac{1}{1-\lambda} L_K^2	\alpha \beta^2+(r-1)(1+r)^2 (2\beta \frac{1}{\mu_g+L_g } -\beta^2)  \frac{1-\lambda}{32(1+1/r)L_y^2 } \big],\\
&C_5=\big[4\alpha^2 \frac{1}{1-\lambda}L_K^2 	\alpha +4(1+\frac{1}{r}) L_y^2\alpha^2 \frac{1-\lambda}{32(1+1/r)L_y^2 } +	\alpha(\lambda -1)+ \frac{1}{1-\lambda} \alpha^2\big].
\end{align}	

To ensure $C_1\leq0$, we have

\begin{align}
C_1=&	-(\frac{\alpha}{2} -4\alpha^2m \frac{1}{1-\lambda} L_K^2	\alpha - \frac{L_{\ell}\alpha^2}{2}- 4(1+\frac{1}{r}) L_y^2\alpha^2m  \frac{1-\lambda}{32(1+1/r)L_y^2 } )\notag\\
=&-	\frac{\alpha}{2} + \frac{4\alpha^3 L_K^2 m}{1-\lambda} + \frac{L_{\ell}\alpha^2}{2}+  \alpha^2m  \frac{1-\lambda}{8 } \notag\\
\overset{(a)}{\leq}&-\frac{\alpha}{2} +\frac{\alpha}{8}+\frac{\alpha}{8}+\frac{\alpha}{8} < 0,
\end{align}	
where $(a)$ follows from $\alpha \leq \min\{ \frac{1}{4  L_K} \sqrt{\frac{1-\lambda}{2m}}, \frac{1}{4L_{\ell}},    \frac{1}{m (1-\lambda)}\} $.

To ensure $C_2\leq-\frac{1-\lambda}{2}$, we have
\begin{align}
C_2=\big[8\frac{1}{1-\lambda} L_K^2	\alpha m+\lambda-1+ \frac{\alpha L_{\ell}  }{m} +(1+\frac{1}{r}) L_y^2 \frac{1-\lambda}{4(1+1/r)L_y^2 }\big]\overset{(a)}{\leq} \frac{1-\lambda}{4} -(1-\lambda)+ \frac{1-\lambda}{4}+ \frac{1-\lambda}{4}\leq -\frac{1-\lambda}{4},
\end{align}	
where $(a)$ follows from $\alpha \leq \min\{ \frac{(1-\lambda)^2}{32L_K^2 } , \frac{  m(1-\lambda)}{4L_{\ell}}\} $.

To ensure $C_3\leq- \frac{3  r^2(1-\lambda)}{32(1+r)L_y^2 }$, we have
\begin{align}
	C_3&=[\frac{2\alpha}{m}  L_f^2  +   \frac{1-\lambda}{32(1+1/r)L_y^2 }3r -  \frac{1-\lambda}{32(1+1/r)L_y^2 }(1+3r)( 2\beta \frac{\mu_g L_g}{\mu_g+L_g }) ]\notag\\&\overset{(a)}{\leq} 3r \frac{1-\lambda}{32(1+1/r)L_y^2 }( 2\beta \frac{\mu_g L_g}{\mu_g+L_g })+\frac{1-\lambda}{32(1+1/r)L_y^2 }( \beta \frac{\mu_g L_g}{\mu_g+L_g })-\frac{1-\lambda}{32(1+1/r)L_y^2 }(1+3r)( 2\beta \frac{\mu_g L_g}{\mu_g+L_g })\notag\\ \leq& -\frac{1-\lambda}{32(1+r)L_y^2 }\beta r \frac{\mu_g L_g}{\mu_g+L_g } =-\frac{3  r^2(1-\lambda)}{32(1+r)L_y^2 } ,
\end{align}	
where $(a)$ follows from $\alpha \leq \frac{(1-\lambda) }{32L_y^2(1+1/r)} \frac{3rm}{ L_f^2} ( \beta \frac{\mu_g L_g}{\mu_g+L_g })=  \frac{(1-\lambda) }{32L_y^2(1+1/r)} \frac{9r^2m}{ L_f^2} , \beta\leq \frac{3(\mu_g+L_g)}{ \mu_g L_g}, r= \frac{1}{3}\beta \frac{\mu_g L_g}{\mu_g+L_g } $.

To ensure $C_4\leq0$, we have
\begin{align}
	C_4&=\big[\frac{1}{1-\lambda} L_K^2	\alpha \beta^2+(r-1)(1+r)^2 (2\beta \frac{1}{\mu_g+L_g } -\beta^2)  \frac{1-\lambda}{32(1+1/r)L_y^2 } \big]\notag\\&\overset{(a)}{\leq}   \frac{1-\lambda}{32(1+1/r)L_y^2 }(1-r) (1+r)^2\beta \frac{1}{\mu_g+L_g } -  \frac{1-\lambda}{32(1+1/r)L_y^2 }(1-r) (1+r)^2 2\beta \frac{1}{\mu_g+L_g } +  \frac{1-\lambda}{32(1+1/r)L_y^2 }(1-r) (1+r)^2\beta \frac{1}{\mu_g+L_g } =0,
\end{align}	
where $(a)$ follows from $\alpha \leq  \frac{(1-r) (1+r) r (1-\lambda)^2}{32 L_y^2  (\mu_g+L_g) L_K^2 \beta }, \beta\leq \frac{1}{\mu_g+L_g}$.

To ensure $C_5\leq0$, we have
\begin{align}
C_5&=\big[4\alpha^2 \frac{1}{1-\lambda}L_K^2 	\alpha +4(1+\frac{1}{r}) L_y^2\alpha^2 \frac{1-\lambda}{32(1+1/r)L_y^2 } +	\alpha(\lambda -1)+ \frac{1}{1-\lambda} \alpha^2\big]
\notag\\&=\big[4\alpha^2 \frac{1}{1-\lambda}L_K^2 	\alpha +\alpha^2 \frac{1-\lambda}{8} +	\alpha(\lambda -1)+ \frac{1}{1-\lambda} \alpha^2\big]
\notag\\&\overset{(a)}{\leq} 
\frac{\alpha(1-\lambda )}{4}+\frac{\alpha(1-\lambda )}{8}- \alpha(1-\lambda )+\frac{\alpha(1-\lambda )}{8}< 0
\end{align}	
where $(a)$ follows from $ \alpha \leq \min\{ \frac{1-\lambda}{4 L_K} ,     1, \frac{(1-\lambda)^2}{8} \}
  $.

  With the above conditions, we have
  \begin{align}
  	&
  	\ell(\bx_{T+1}) - \ell(\bx_{0}) + \frac{1-\lambda}{32(1+1/r)L_y^2 }  \big[\|\y_{T+1} - \y_{T+1}^*\|^2  - \|\y_0^* - \y_0\|^2\big] \notag\\&+	[\|\x_{T+1}-\ot\bx_{T+1}\|^2-	\|\x_{0}-\ot\bx_{0}\|^2 ]+
  	\alpha [	\|\u_{T+1}-\ot\bu_{T+1}\|^2  -	\|\u_{0}-\ot\bu_{0}\|^2 ]
  	\notag\\
  	\le
  	&-\frac{\alpha}{2} \sum_{t=0}^{T} \|\nabla \ell(\bx_t)\|^2-\frac{1-\lambda}{4}\sum_{t=0}^{T} \| \x_{t} - \ot\bx_{t}  \|^2 - \frac{3  r^2(1-\lambda)}{32(1+r)L_y^2 } \sum_{t=0}^{T}\|\y_{t} - \y_{t}^*\|^2.
  \end{align}

Thus,  we have the following convergence results:
  \begin{align}
  \frac{1}{T+1}\sum_{t=0}^{T} \| \mathfrak{M}_t\|^2 \leq & 	\frac{
  1}{(T+1)\min\{\frac{1-\lambda}{4},\frac{3  r^2(1-\lambda)}{32(1+r)L_y^2 },\frac{\alpha}{2}\}}[ \mathfrak{B}_0- \mathfrak{B}_{T+1}]\notag\\
& \leq  	\frac{
	1}{(T+1)\min\{\frac{1-\lambda}{4}, \frac{3  r^2(1-\lambda)}{32(1+r)L_y^2 },\frac{\alpha}{2}\}}[ \mathfrak{B}_0- \ell^*],
\end{align}
where $ \mathfrak{B}_t=	\ell(\bx_{t}) +  \frac{1-\lambda}{32(1+1/r)L_y^2 }  \|\by_{t} - \y_{t}^*\|^2  +	\|\x_{t}-\ot\bx_{t}\|^2 +
\alpha 	\|\u_{t}-\ot\bu_{t}\|^2  $.

\textbf{Proof of Theorem \ref{thm2}:}
For Algortihm \ref{Algorithm_VR}, we have

\begin{align} 
	\| \frac{1}{m} \sum_{i=1}^{m} \bar\nabla f_i (\x_{i,t} ,\y_{i,t} ) - \bar\u_{ t} \|^2
	=	\| \frac{1}{m} \sum_{i=1}^{m} \bar\nabla f_i (\x_{i,t} ,\y_{i,t} ) - \bar\p(\x_{i,t},\y_{i,t} ) \|^2 \leq  \frac{1}{m} \sum_{i=1}^{m} \|\bar\nabla f_i (\x_{i,t} ,\y_{i,t} ) - \p_i(\x_{i,t},\y_{i,t} ) \|^2
\end{align}

From the algortihm update, we have
\begin{align} 
	& \mathbb{E}_t \|\bar\nabla f_i (\x_{i,t} ,\y_{i,t} ) - \p_i(\x_{i,t},\y_{i,t} ) \|^2
	\notag\\
	=	& \mathbb{E}_t \|\bar\nabla f_i (\x_{i,t} ,\y_{i,t} ) - \p_i(\x_{i,t},\y_{i,t} )+\mathbb{E}_{\bar\xi_{i,t}}[ \p_i(\x_{i,t},\y_{i,t} )]-\mathbb{E}_{\bar\xi_{i,t}}[ \p_i(\x_{i,t},\y_{i,t} )] \|^2
	\notag\\
= &\mathbb{E}_t \| \p_i(\x_{i,t},\y_{i,t} )-\mathbb{E}_{\bar\xi_{i,t}}[ \p_i(\x_{i,t},\y_{i,t} )] \|^2+  \mathbb{E}_t \|\bar\nabla f_i (\x_{i,t} ,\y_{i,t} ) -\mathbb{E}_{\bar\xi_{i,t}}[ \p_i(\x_{i,t},\y_{i,t} )] \|^2\notag\\
\leq& \mathbb{E}_t \| \p_i(\x_{i,t},\y_{i,t} )-\mathbb{E}_{\bar\xi_{i,t}}[ \p_i(\x_{i,t},\y_{i,t} )] \|^2+  (\frac{C_{g_{x y}} C_{f_{y}}}{\mu_{g}}\left(1-\frac{\mu_{g}}{L_{g}}\right)^{K})^2,
\end{align}
where the last in equality os 
Moreover, with $t \in\left(\left(n_{t}-1\right) q, n_{t} q-1\right] \cap \mathbb{Z}$, we have
\begin{align} 
&	\mathbb{E}_t \| \p_i(\x_{i,t},\y_{i,t} )-\mathbb{E}_{\bar\xi_{i,t}}[ \p_i(\x_{i,t},\y_{i,t} )] \|^2\notag\\
=& \mathbb{E}_t \| \p_{i}(\x_{i,t-1},\y_{i,t-1})+	\frac{1}{|\mathcal{S}|} \sum_{i=1}^{\mathcal{S}} \big[\bar\nabla f_{i}(\x_{i,k},\y_{i,k}; {\bar\xi_{i,t}} )-\bar\nabla f_{i}(\x_{i,k-1},\y_{i,k-1};  {\bar\xi_{i,t}} )\big]-\mathbb{E}_{\bar\xi_{i,t}}[ \p_i(\x_{i,t},\y_{i,t} )]\notag\\&+\mathbb{E}_{\bar\xi_{i,t}}[ \p_i(\x_{i,t-1},\y_{i,t-1} )] -\mathbb{E}_{\bar\xi_{i,t}}[ \p_i(\x_{i,t-1},\y_{i,t-1} )]  \|^2\notag\\
=& \mathbb{E}_t \| \p_{i}(\x_{i,t-1},\y_{i,t-1}) -\mathbb{E}_{\bar\xi_{i,t}}[ \p_i(\x_{i,t-1},\y_{i,t-1} )] \|^2+\|	\frac{1}{|\mathcal{S}|} \sum_{i=1}^{\mathcal{S}} \big[\bar\nabla f_{i}(\x_{i,k},\y_{i,k}; {\bar\xi_{i,t}} )-\bar\nabla f_{i}(\x_{i,k-1},\y_{i,k-1};  {\bar\xi_{i,t}} )\big]\notag\\&-\mathbb{E}_{\bar\xi_{i,t}}[ \p_i(\x_{i,t},\y_{i,t} )]+\mathbb{E}_{\bar\xi_{i,t}}[ \p_i(\x_{i,t-1},\y_{i,t-1} )] \|^2\notag\\
\overset{(b}{\leq})& \mathbb{E}_t \| \p_{i}(\x_{i,t-1},\y_{i,t-1}) -\mathbb{E}_{\bar\xi_{i,t}}[ \p_i(\x_{i,t-1},\y_{i,t-1} )] \|^2+\frac{1}{|\mathcal{S}|}  L_K^2 \mathbb{E}_t(\| \x_{i,t}-\x_{i,t-1} \|^2+\| \y_{i,t}-\y_{i,t-1} \|^2),
\end{align}
where the last inequality use the mean variance theorem.

Telescoping over $t$ from ($(n_t-1)q+1$ to $t$, where $t\leq n_tq-1$, we obtain that
\begin{align} \label{eqs70}
	&	\mathbb{E}_t \| \p_i(\x_{i,t},\y_{i,t} )-\mathbb{E}_{\bar\xi_{i,t}}[ \p_i(\x_{i,t},\y_{i,t} )] \|^2\notag\\
	\leq& \mathbb{E}_t \| \p_{i}(\x_{i,(n_t-1)q},\y_{i,(n_t-1)q}) -\mathbb{E}_{\bar\xi_{i,t}}[ \p_i(\x_{i,(n_t-1)q},\y_{i,(n_t-1)q} )] \|^2+\frac{1}{|\mathcal{S}|}  L_K^2\sum_{t=(n_t-1)q}^{t-1} \mathbb{E}_t(\| \x_{i,t}-\x_{i,t-1} \|^2+\| \y_{i,t}-\y_{i,t-1} \|^2)\notag\\
\end{align}

Since $ \mathbb{E}_t \| \p_{i}(\x_{i,(n_t-1)q},\y_{i,(n_t-1)q}) -\mathbb{E}_{\bar\xi_{i,t}}[ \p_i(\x_{i,(n_t-1)q},\y_{i,(n_t-1)q} )] \|^2=0, |\mathcal{S}|=q$, we can conclude that 

\begin{align} \label{73}
&\sum_{t=0}^T	\| \frac{1}{m} \sum_{i=1}^{m} \bar\nabla f_i (\x_{i,t} ,\y_{i,t} ) - \bar\u_{ t} \|^2\notag\\
\leq& L_K^2\sum_{t=0}^T (\| \x_{t+1}-\x_{t} \|^2+\| \y_{t+1}-\y_{t} \|^2) + ( \frac{C_{g_{x y}} C_{f_{y}}}{\mu_{g}}\left(1-\frac{\mu_{g}}{L_{g}}\right)^{K})^2\cdot (T+1).
\end{align}

Also, we have 
	
\begin{align} \label{74}
	&\sum_{i=1}^m \sum_{t=0}^T 	\|  \bar\nabla_{\y} g_i (\x_{i,t} ,\y_{i,t} ) - \v_{ i, t} \|^2
	\leq L_K^2\sum_{t=0}^T (\| \x_{t+1}-\x_{t} \|^2+\| \y_{t+1}-\y_{t} \|^2).
\end{align}

Next, for $	\|\p_t-\p_{t-1} \|^2$, we have\textsf{}
Case 1: $t \in\left(\left(n_{t}-1\right) q, n_{t} q-1\right] \cap \mathbb{Z}$: 
\begin{align}\label{72}
	&\mathbb{E}\|\p_t-\p_{t-1} \|^2=\sum_{i=1}^{m} \mathbb{E}\left\|\frac{1}{\left|\mathcal{S}_{i, t}\right|} \sum_{j \in \mathcal{S}_{i, t}} \nabla_{\boldsymbol{\x}} f_{i}\left(\boldsymbol{\x}_{i, t}, \boldsymbol{\y}_{i, t}; {\bar\xi_{j,t}} \right)-\nabla_{\boldsymbol{\x}} f_{i}\left(\boldsymbol{\x}_{i, t-1}, \boldsymbol{\y}_{i, t-1}; {\bar\xi_{j,t}} \right)\right\|^{2} \\
		&\leq \frac{1}{\left|\mathcal{S}_{i, t}\right|^{2}} \sum_{i=1}^{m} \sum_{j \in \mathcal{S}_{i, t}} \mathbb{E}\left\|\nabla_{\boldsymbol{\x}} f_{i}\left(\boldsymbol{\x}_{i, t}, \boldsymbol{\y}_{i, t}; {\bar\xi_{j,t}}\right)-\nabla_{\boldsymbol{\x}} f_{i}\left(\boldsymbol{\x}_{i, t-1}, \boldsymbol{\y}_{i, t-1};{\bar\xi_{j,t}} \right) \right\|^{2} \\
		&\leq L_{K}^{2} \sum_{i=1}^{m} \mathbb{E}\left\|\boldsymbol{\x}_{i, t-1}-\boldsymbol{\x}_{i, t}\right\|^{2}+L_{K}^{2} \sum_{i=1}^{m} \mathbb{E}\left\|\boldsymbol{\y}_{i, t-1}-\boldsymbol{\y}_{i, t}\right\|^{2}\notag\\
&	 {\le}
	L_K^2  (\|\x_{t}- \x_{t-1} \|^2+\beta^2 \|\v_{t-1} \|^2).
\end{align}

Case 2: $t= n_{t} q $: 


\begin{align}
	&	\mathbb{E}\|\p_t-\p_{t-1} \|^2\notag\\
	=&	\mathbb{E}\|\p_t-\p_{t-1}-\mathbb{E}_{\bar\xi_{i,t}}\p_i(\x_{i,t},\y_{i,t} )+\mathbb{E}_{\bar\xi_{i,t}}\p_i(\x_{i,t},\y_{i,t} )-\mathbb{E}_{\bar\xi_{i,t}}[ \p_i(\x_{i,t-1},\y_{i,t-1} )+\mathbb{E}_{\bar\xi_{i,t}} \p_i(\x_{i,t-1},\y_{i,t-1} ) \|^2\notag\\
	\le&	3\mathbb{E}\|\p_t-\mathbb{E}_{\bar\xi_{i,t}}[ \p_i(\x_{i,t},\y_{i,t} ) \|^2+3\mathbb{E}\|\p_{t-1}-\mathbb{E}_{\bar\xi_{i,t}}[ \p_i(\x_{i,t-1},\y_{i,t-1} )\|^2+3L_K^2\mathbb{E}(\|\x_{t}- \x_{t-1} \|^2+\beta^2 \|\v_{t-1} \|^2)\notag\\
	\overset{(a)}{\le}&
	3 \mathbb{E}\left\|\p_{n_tq}-\mathbb{E}_{\bar\xi_{i,t}}\p_i(\x_{i,n_tq},\y_{i,n_tq} ) \right\|^{2}+3 \mathbb{E}\left\|\p(\x_{i,(n_t-1)q},\y_{i,(n_t-1)q}-\mathbb{E}_{\bar\xi_{i,t}}\p_i(\x_{i,(n_t-1)q},\y_{i,(n_t-1)q} )\right\|^{2}  \notag\\
	&+\sum_{r'=\left(n_{t}-1\right) q+1}^{n_{t} q-1} \frac{3 L_{K}^{2}}{|\mathcal{S}|} \mathbb{E}\left(\left\|\x_{r'}-\x_{r'-1}\right\|^{2}+\left\|\y_{r'}-\y_{r'-1}\right\|^{2}\right) +3 L_{K}^{2} \mathbb{E}\left\|\x_{n_{t} q-1}-\x_{n_{t} q}\right\|^{2}+3 L_{K}^{2} \beta^2 \mathbb{E}\left\|\v_{n_{t} q-1}\right\|^{2},
\end{align}
where (a) is from (\ref{eqs70}) and set $t=n_tq$.

Telescoping from $r=\left(n_{t}-1\right) q+1$  to $n_{t} q $ and set $|\mathcal{S}|=q$, we have
\begin{align}
&	\sum_{r=\left(n_{t}-1\right) q+1}^{n_{t} q}\mathbb{E}\|\p_r-\p_{r-1} \|^2\notag\\
	\leq &	3(q+1) \mathbb{E}\left\|\p_{n_tq}-\mathbb{E}_{\bar\xi_{i,t}}\p_i(\x_{i,n_tq},\y_{i,n_tq})\right\|^{2}+3 (q+1) \mathbb{E}\left\|\p_{(n_t-1)q}-\mathbb{E}_{\bar\xi_{i,t}}\p_i(\x_{i,(n_t-1)q},\y_{i,(n_t-1)q} )\right\|^{2} \notag\\
	&+\sum_{r=\left(n_{t}-1\right) q+1}^{n_{t} q} \frac{4 L_{K}^{2}}{q} \mathbb{E}\left(\left\|\x_{r}-\x_{r-1}\right\|^{2}+\left\|\y_{r}-\y_{r-1}\right\|^{2}\right) \notag\\
	=
	&\sum_{r=\left(n_{t}-1\right) q+1}^{n_{t} q} \frac{4 L_{K}^{2}}{q} \mathbb{E}\left(\left\|\x_{r}-\x_{r-1}\right\|^{2}+\left\|\y_{r}-\y_{r-1}\right\|^{2}\right) 
\end{align}

Since $\mathbb{E}\left\|\p_{n_tq}-\mathbb{E}_{\bar\xi_{i,t}}\p_i(\x_{i,n_tq},\y_{i,n_tq})\right\|^{2}=\mathbb{E}\left\|\p_{(n_t-1)q}-\mathbb{E}_{\bar\xi_{i,t}}\p_i(\x_{i,(n_t-1)q},\y_{i,(n_t-1)q} )\right\|^{2} =0 $, and with eqs.(\ref{72}),we have

\begin{align}\label{75}
\sum_{t=1}^T	\|\p_t-\p_{t-1} \|^2
\leq \sum_{t=1}^T\big[ 4L_K^2 \mathbb{E} \left\|\x_{t}-\x_{t-1}\right\|^{2}+4L_K^2 \left\|\y_{t}-\y_{t-1}\right\|^{2}  \big]
\end{align}
With the results from Step 3, we have
\begin{align}	\|\x_t-\ot\bx_t\|^2& \le  (1+c_1)\lambda^2\|\x_{t-1} -\ot\bx_{t-1} \|^2 + (1+\frac{1}{c_1}) \alpha^2\|\u_{t-1}-\ot \bu_{t-1}\|^2,\notag\\
	\|\u_t-\ot\bu_t\|^2 & \le	\|\u_{t-1}-\ot\bu_{t-1}\|^2 + ((1+c_2)\lambda^2 -1)\|\u_{t-1} -\ot\bu_{t-1} \|^2 \notag\\&+ (1 + \frac{1}{c_2})4L_K^2  (\|\x_{t}- \x_{t-1} \|^2+\beta^2 \|\v_{t-1} \|^2). 
\end{align}

Then, we have
\begin{align}\label{l3'}
	&	\|\x_{T+1}-\ot\bx_{T+1}\|^2-	\|\x_{0}-\ot\bx_{0}\|^2 \notag\\ \le & ((1+c_1)\lambda^2-1) \sum_{t=1}^{T+1} \|\x_{t-1} -\ot\bx_{t-1} \|^2 + (1+\frac{1}{c_1}) \alpha^2\sum_{t=1}^{T+1}\|\u_{t-1}-\ot \bu_{t-1}\|^2.
\end{align}
\begin{align}\label{l4'}
	&	\|\u_{T+1}-\ot\bu_{T+1}\|^2  -	\|\u_{0}-\ot\bu_{0}\|^2 
	\notag\\
	\le& ((1+c_2)\lambda^2 -1)\sum_{t=1}^{T+1}\|\u_{t-1} -\ot\bu_{t-1} \|^2 +(1 + \frac{1}{c_2})4L_K^2 \sum_{t=1}^{T+1} (\|\x_{t}- \x_{t-1} \|^2+\beta^2 \|\v_{t-1} \|^2). 
\end{align}

Combing \eqref{eqs5}, \eqref{l3'} and \eqref{l4'}, we have
\begin{align}
	&
	\ell(\bx_{T+1}) - \ell(\bx_{0}) +  \frac{1-\lambda}{32(1+1/r)L_y^2 } \big[\|\by_{T+1} - \y_{T+1}^*\|^2  - \|\y_0^* - \by_0\|^2\big] \notag\\&+	[\|\x_{T+1}-\ot\bx_{T+1}\|^2-	\|\x_{0}-\ot\bx_{0}\|^2 ]+
	\alpha [	\|\u_{T+1}-\ot\bu_{T+1}\|^2  -	\|\u_{0}-\ot\bu_{0}\|^2 ]
	\notag\\
	\le
&-\frac{\alpha}{2} \sum_{t=0}^{T} \|\nabla \ell(\bx_t)\|^2 - (\frac{\alpha}{2} - \frac{L_{\ell}\alpha^2}{2}- 4(1+\frac{1}{r}) L_y^2\alpha^2m \frac{1-\lambda}{32(1+1/r)L_y^2 } )\sum_{t=0}^{T} \|\bu_{t}\|^2  \notag\\&+ (\frac{\alpha L_{\ell}  }{m}  +(1+\frac{1}{r}) L_y^2 \frac{1-\lambda}{4(1+1/r)L_y^2 } )\sum_{t=0}^{T} \| \x_{t} - \ot\bx_{t}  \|^2 \notag\\
&+  [\frac{2\alpha}{m}  L_f^2  +   \frac{1-\lambda}{32(1+1/r)L_y^2 } 3r -  \frac{1-\lambda}{32(1+1/r)L_y^2 }(1+3r)( 2\beta \frac{\mu_g L_g}{\mu_g+L_g }) ]\sum_{t=0}^{T}\|\y_{t} - \y_{t}^*\|^2\notag\\
&+ 2\alpha( \sum_{t=0}^{T} \| \frac{1}{m} \sum_{i=1}^{m} \bar\nabla f_i (\x_{i,t} ,\y_{i,t} ) - \bar\u_{ t} \|^2) \notag\\
&+ \frac{1-\lambda}{32(1+1/r)L_y^2 }(\delta-1)(1+r)^2 (2\beta \frac{1}{\mu_g+L_g } -\beta^2)\sum_{t=0}^{T}\| \v_{t} \|^2+ 4(1+\frac{1}{r}) L_y^2\alpha^2 \frac{1-\lambda}{32(1+1/r)L_y^2 } (\sum_{t=0}^{T}\|\u_{t}- \ot\bu_{t}\|^2 )\notag\\
& + \frac{1-\lambda}{32(1+1/r)L_y^2 }[(1+1/r)(1+r)\beta^2+(\frac{1}{\delta}-1 ) (1+r)^2 (2\beta \frac{1}{\mu_g+L_g } -\beta^2)]\sum_{i=1}^m\sum_{t=0}^{T} \| \v_{i,t} - \nabla_{\y} g_{i}(\x_{i,t},\y_{i,t})\|^2\notag\\
&+((1+c_1)\lambda^2-1) \sum_{t=1}^{T+1} \|\x_{t-1} -\ot\bx_{t-1} \|^2 + (1+\frac{1}{c_1}) \alpha^2\sum_{t=1}^{T+1}\|\u_{t-1}-\ot \bu_{t-1}\|^2\notag\\
&+\alpha((1+c_2)\lambda^2 -1)\sum_{t=1}^{T+1}\|\u_{t-1} -\ot\bu_{t-1} \|^2 +\alpha(1 + \frac{1}{c_2})4L_K^2 \sum_{t=1}^{T+1} (\|\x_{t}- \x_{t-1} \|^2+\beta^2 \|\v_{t-1} \|^2).
\end{align}	
	
Next, by plugging the Eqs.\eqref{73}, Eqs.\eqref{74}, we have
\begin{align}
		&
	\ell(\bx_{T+1}) - \ell(\bx_{0}) + \frac{1-\lambda}{32(1+1/r)L_y^2 }  \big[\|\by_{T+1} - \y_{T+1}^*\|^2  - \|\y_0^* - \by_0\|^2\big] \notag\\&+	[\|\x_{T+1}-\ot\bx_{T+1}\|^2-	\|\x_{0}-\ot\bx_{0}\|^2 ]+
	\alpha [	\|\u_{T+1}-\ot\bu_{T+1}\|^2  -	\|\u_{0}-\ot\bu_{0}\|^2 ]
	\notag\\
\leq
&-\frac{\alpha}{2} \sum_{t=0}^{T} \|\nabla \ell(\bx_t)\|^2 - (\frac{\alpha}{2} - \frac{L_{\ell}\alpha^2}{2}- 4(1+\frac{1}{r}) L_y^2\alpha^2m \frac{1-\lambda}{32(1+1/r)L_y^2 } )\sum_{t=0}^{T} \|\bu_{t}\|^2  \notag\\&+ (\frac{\alpha L_{\ell}  }{m}  +(1+\frac{1}{r}) L_y^2 \frac{1-\lambda}{4(1+1/r)L_y^2 } )\sum_{t=0}^{T} \| \x_{t} - \ot\bx_{t}  \|^2 \notag\\
&+  [\frac{2\alpha}{m}  L_f^2  +   \frac{1-\lambda}{32(1+1/r)L_y^2 } 3r -  \frac{1-\lambda}{32(1+1/r)L_y^2 }(1+3r)( 2\beta \frac{\mu_g L_g}{\mu_g+L_g }) ]\sum_{t=0}^{T}\|\y_{t} - \y_{t}^*\|^2\notag\\
&+ 2\alpha (L_K^2\sum_{t=0}^T (\| \x_{t+1}-\x_{t} \|^2+\| \y_{t+1}-\y_{t} \|^2) + ( \frac{C_{g_{x y}} C_{f_{y}}}{\mu_{g}}\left(1-\frac{\mu_{g}}{L_{g}}\right)^{K})^2\cdot (T+1)) \notag\\
&+ \frac{1-\lambda}{32(1+1/r)L_y^2 }(\delta-1)(1+r)^2 (2\beta \frac{1}{\mu_g+L_g } -\beta^2)\sum_{t=0}^{T}\| \v_{t} \|^2+ 4(1+\frac{1}{r}) L_y^2\alpha^2 \frac{1-\lambda}{32(1+1/r)L_y^2 } (\sum_{t=0}^{T}\|\u_{t}- \ot\bu_{t}\|^2 )\notag\\
& + \frac{1-\lambda}{32(1+1/r)L_y^2 }[(1+1/r)(1+r)\beta^2+(\frac{1}{\delta}-1 ) (1+r)^2 (2\beta \frac{1}{\mu_g+L_g } -\beta^2)][L_K^2\sum_{t=0}^T (\| \x_{t+1}-\x_{t} \|^2+\| \y_{t+1}-\y_{t} \|^2)]\notag\\
&+((1+c_1)\lambda^2-1) \sum_{t=1}^{T+1} \|\x_{t-1} -\ot\bx_{t-1} \|^2 + (1+\frac{1}{c_1}) \alpha^2\sum_{t=1}^{T+1}\|\u_{t-1}-\ot \bu_{t-1}\|^2\notag\\
&+\alpha((1+c_2)\lambda^2 -1)\sum_{t=1}^{T+1}\|\u_{t-1} -\ot\bu_{t-1} \|^2 +\alpha(1 + \frac{1}{c_2})4L_K^2 \sum_{t=1}^{T+1} (\|\x_{t}- \x_{t-1} \|^2+\beta^2 \|\v_{t-1} \|^2),
\end{align}

Next, with $r\in(0,1]$, we have 
\begin{align}\label{87}
&L_K^2[(1+1/r)(1+r)\beta^2+(\frac{1}{\delta}-1 ) (1+r)^2 (2\beta \frac{1}{\mu_g+L_g } -\beta^2)]\notag\\
&\leq  L_K^2 [\frac{(1+r)^2}{r}\beta^2+(\frac{1}{\delta}) (1+r)^2 (2\beta \frac{1}{\mu_g+L_g } )] \notag\\
&\leq  L_K^2  [\frac{4\beta^2}{r}+(\frac{1}{\delta}) (1+r)^2 (2\beta \frac{1}{\mu_g+L_g } )] \notag\\
&	\overset{(a)}{=} L_K^2  [\frac{12\beta (\mu_g+L_g) }{ \mu_g L_g}+(\frac{1}{\delta}) (1+r)^2 (2\beta \frac{1}{\mu_g+L_g } )].
\end{align}
where   (a) follows from $r=\frac{1}{3} \beta \frac{\mu_g L_g}{\mu_g+L_g }, \beta\leq \frac{3 (\mu_g+L_g) }{\mu_gL_g}$.

Thus, we can conclude that 
\begin{align}
	&
	\ell(\bx_{T+1}) - \ell(\bx_{0}) +  \frac{1-\lambda}{32(1+1/r)L_y^2 }  \big[\|\by_{T+1} - \y_{T+1}^*\|^2  - \|\y_0^* - \by_0\|^2\big] \notag\\&+	[\|\x_{T+1}-\ot\bx_{T+1}\|^2-	\|\x_{0}-\ot\bx_{0}\|^2 ]+
	\alpha [	\|\u_{T+1}-\ot\bu_{T+1}\|^2  -	\|\u_{0}-\ot\bu_{0}\|^2 ]
	\notag\\
	\leq
&-\frac{\alpha}{2} \sum_{t=0}^{T} \|\nabla \ell(\bx_t)\|^2 - (\frac{\alpha}{2} - \frac{L_{\ell}\alpha^2}{2}- 4(1+\frac{1}{r}) L_y^2\alpha^2m \frac{1-\lambda}{32(1+1/r)L_y^2 } )\sum_{t=0}^{T} \|\bu_{t}\|^2  \notag\\&+ (\frac{\alpha L_{\ell}  }{m}  +(1+\frac{1}{r}) L_y^2 \frac{1-\lambda}{4(1+1/r)L_y^2 } )\sum_{t=0}^{T} \| \x_{t} - \ot\bx_{t}  \|^2 \notag\\
&+  [\frac{2\alpha}{m}  L_f^2  +   \frac{1-\lambda}{32(1+1/r)L_y^2 } 3r -  \frac{1-\lambda}{32(1+1/r)L_y^2 }(1+3r)( 2\beta \frac{\mu_g L_g}{\mu_g+L_g }) ]\sum_{t=0}^{T}\|\y_{t} - \y_{t}^*\|^2\notag\\
&+ 2\alpha (L_K^2\sum_{t=0}^T (\| \x_{t+1}-\x_{t} \|^2+\| \y_{t+1}-\y_{t} \|^2) + ( \frac{C_{g_{x y}} C_{f_{y}}}{\mu_{g}}\left(1-\frac{\mu_{g}}{L_{g}}\right)^{K})^2\cdot (T+1)) \notag\\
&+ \frac{1-\lambda}{32(1+1/r)L_y^2 }(\delta-1)(1+r)^2 (2\beta \frac{1}{\mu_g+L_g } -\beta^2)\sum_{t=0}^{T}\| \v_{t} \|^2+ 4(1+\frac{1}{r}) L_y^2\alpha^2 \frac{1-\lambda}{32(1+1/r)L_y^2 } (\sum_{t=0}^{T}\|\u_{t}- \ot\bu_{t}\|^2 )\notag\\
& + L_K^2\frac{1-\lambda}{32(1+1/r)L_y^2 }[\frac{12\beta (\mu_g+L_g) }{ \mu_g L_g}+(\frac{1}{\delta}) (1+r)^2 (2\beta \frac{1}{\mu_g+L_g } )]\sum_{t=0}^T (\| \x_{t+1}-\x_{t} \|^2+\| \y_{t+1}-\y_{t} \|^2)]\notag\\
&+((1+c_1)\lambda^2-1) \sum_{t=1}^{T+1} \|\x_{t-1} -\ot\bx_{t-1} \|^2 + (1+\frac{1}{c_1}) \alpha^2\sum_{t=1}^{T+1}\|\u_{t-1}-\ot \bu_{t-1}\|^2\notag\\
&+\alpha((1+c_2)\lambda^2 -1)\sum_{t=1}^{T+1}\|\u_{t-1} -\ot\bu_{t-1} \|^2 +\alpha(1 + \frac{1}{c_2})4L_K^2 \sum_{t=1}^{T+1} (\|\x_{t}- \x_{t-1} \|^2+\beta^2 \|\v_{t-1} \|^2)
	\notag\\
\leq
&-\frac{\alpha}{2} \sum_{t=0}^{T} \|\nabla \ell(\bx_t)\|^2 - (\frac{\alpha}{2} - \frac{L_{\ell}\alpha^2}{2}- 4(1+\frac{1}{r}) L_y^2\alpha^2m \frac{1-\lambda}{32(1+1/r)L_y^2 } )\sum_{t=0}^{T} \|\bu_{t}\|^2  \notag\\&+ (\frac{\alpha L_{\ell}  }{m}  +(1+\frac{1}{r}) L_y^2 \frac{1-\lambda}{4(1+1/r)L_y^2 } )\sum_{t=0}^{T} \| \x_{t} - \ot\bx_{t}  \|^2 \notag\\
&+  [\frac{2\alpha}{m}  L_f^2  +   \frac{1-\lambda}{32(1+1/r)L_y^2 } 3r -  \frac{1-\lambda}{32(1+1/r)L_y^2 }(1+3r)( 2\beta \frac{\mu_g L_g}{\mu_g+L_g }) ]\sum_{t=0}^{T}\|\y_{t} - \y_{t}^*\|^2\notag\\
&+ 2\alpha (L_K^2\sum_{t=0}^T (8\|(\x_{t} - \ot\bx_{t}) \|^2 + 4\alpha^2 \|\u_{t} - \ot\bu_{t}\|^2 + 4\alpha^2m\|\bu_{t}\|^2 +\beta^2\| \v_{t} \|^2) + ( \frac{C_{g_{x y}} C_{f_{y}}}{\mu_{g}}\left(1-\frac{\mu_{g}}{L_{g}}\right)^{K})^2\cdot (T+1)) \notag\\
&+ \frac{1-\lambda}{32(1+1/r)L_y^2 }(\delta-1)(1+r)^2 (2\beta \frac{1}{\mu_g+L_g } -\beta^2)\sum_{t=0}^{T}\| \v_{t} \|^2+ 4(1+\frac{1}{r}) L_y^2\alpha^2 \frac{1-\lambda}{32(1+1/r)L_y^2 } (\sum_{t=0}^{T}\|\u_{t}- \ot\bu_{t}\|^2 )\notag\\
& +L_K^2 \frac{1-\lambda}{32(1+1/r)L_y^2 }[\frac{12\beta (\mu_g+L_g) }{ \mu_g L_g}+(\frac{1}{\delta}) (1+r)^2 (2\beta \frac{1}{\mu_g+L_g } )]\sum_{t=0}^T (8\|(\x_{t} - \ot\bx_{t}) \|^2 + 4\alpha^2 \|\u_{t} - \ot\bu_{t}\|^2 + 4\alpha^2m\|\bu_{t}\|^2+\beta^2 \| \v_{t} \|^2)]\notag\\
&+((1+c_1)\lambda^2-1) \sum_{t=0}^{T} \|\x_{t} -\ot\bx_{t} \|^2 + (1+\frac{1}{c_1}) \alpha^2\sum_{t=0}^{T}\|\u_{t}-\ot \bu_{t}\|^2\notag\\
&+\alpha((1+c_2)\lambda^2 -1)\sum_{t=0}^{T}\|\u_{t} -\ot\bu_{t} \|^2 +\alpha(1 + \frac{1}{c_2})4L_K^2 \sum_{t=0}^{T} (8\|(\x_{t} - \ot\bx_{t}) \|^2 + 4\alpha^2 \|\u_{t} - \ot\bu_{t}\|^2 + 4\alpha^2m\|\bu_{t}\|^2+\beta^2 \|\v_{t} \|^2)
\end{align}

where the last inequality is from eqs.(\ref{eqs30}). Choosing $c_1=c_2=\frac{1}{\lambda}-1$, we have
\begin{align}\label{l5'}
	&
	\ell(\bx_{T+1}) - \ell(\bx_{0}) + \frac{1-\lambda}{32(1+1/r)L_y^2 }  \big[\|\by_{T+1} - \y_{T+1}^*\|^2  - \|\y_0^* - \by_0\|^2\big] \notag\\&+	[\|\x_{T+1}-\ot\bx_{T+1}\|^2-	\|\x_{0}-\ot\bx_{0}\|^2 ]+
	\alpha [	\|\u_{T+1}-\ot\bu_{T+1}\|^2  -	\|\u_{0}-\ot\bu_{0}\|^2 ]
	\notag\\
	\le
	&-\frac{\alpha}{2} \sum_{t=0}^{T} \|\nabla \ell(\bx_t)\|^2 + C_1'\sum_{t=0}^{T} \|\bu_{t}\|^2  
	+C_2'\sum_{t=0}^{T} \| \x_{t} - \ot\bx_{t}  \|^2 +  C_3' \sum_{t=0}^{T}\|\y_{t} - \y_{t}^*\|^2 \notag\\
	&+C_4'\sum_{t=0}^{T}\| \v_{t} \|^2+ C_5'\sum_{t=0}^{T}\|\u_{t}- \ot\bu_{t}\|^2+2\alpha  ( \frac{C_{g_{x y}} C_{f_{y}}}{\mu_{g}}\left(1-\frac{\mu_{g}}{L_{g}}\right)^{K})^2\cdot (T+1),
\end{align}
where the constants are

\begin{align}
	&C_1'=-(\frac{\alpha}{2} -4\alpha^2m\frac{1}{1-\lambda} 4L_K^2	\alpha - \frac{L_{\ell}\alpha^2}{2}- 4(1+\frac{1}{r}) L_y^2\alpha^2m  \frac{1-\lambda}{32(1+1/r)L_y^2 } -8\alpha^3 L_K^2m\notag\\&\qquad-L_K^2[\frac{12\beta (\mu_g+L_g) }{ \mu_g L_g}+(\frac{1}{\delta}) (1+r)^2 (2\beta \frac{1}{\mu_g+L_g } )] 4\alpha^2m  \frac{1-\lambda}{32(1+1/r)L_y^2 } ),\\
	&C_2'=\big[\frac{1}{1-\lambda} 32L_K^2	\alpha +16\alpha L_K^2+\lambda-1+ \frac{\alpha L_{\ell}  }{m} +\frac{(1-\lambda)}{4} + \frac{1-\lambda}{4(1+1/r)L_y^2 }[\frac{12\beta (\mu_g+L_g) }{ \mu_g L_g}+(\frac{1}{\delta}) (1+r)^2 (2\beta \frac{1}{\mu_g+L_g } )]L_K^2 \big],\\
	&C_3'=[\frac{2\alpha}{m}  L_f^2  +  3r \frac{1-\lambda}{32(1+1/r)L_y^2 } - (1+3r)( 2\beta \frac{\mu_g L_g}{\mu_g+L_g })  \frac{1-\lambda}{32(1+1/r)L_y^2 }],\\
	&C_4'=\big[2\alpha L_K^2\beta^2+\frac{1}{1-\lambda}4L_K^2	\alpha \beta^2+(\delta-1)(1+r)^2 (2\beta \frac{1}{\mu_g+L_g } -\beta^2) \frac{1-\lambda}{32(1+1/r)L_y^2 }\notag\\&\qquad +\beta^2\cdot \frac{1-\lambda}{32(1+1/r)L_y^2 } L_K^2  [\frac{12\beta (\mu_g+L_g) }{ \mu_g L_g}+(\frac{1}{\delta}) (1+r)^2 (2\beta \frac{1}{\mu_g+L_g } )] \big] ,\\
	&C_5'=\big[4\alpha^2 \frac{1}{1-\lambda}4L_K^2 	\alpha + \frac{1}{1-\lambda} \alpha^2+ 4(1+\frac{1}{r}) L_y^2\alpha^2  \frac{1-\lambda}{32(1+1/r)L_y^2 }  +	\alpha(\lambda +8\alpha^2 L_K^2-1)\notag\\&\qquad+4\alpha^2 \cdot  \frac{1-\lambda}{32(1+1/r)L_y^2 }  L_K^2  [\frac{12\beta (\mu_g+L_g) }{ \mu_g L_g}+(\frac{1}{\delta}) (1+r)^2 (2\beta \frac{1}{\mu_g+L_g } )]\big].
\end{align}	

To ensure $C_1'\leq0$, we have

\begin{align}
	C_1'=&	-(\frac{\alpha}{2} -4\alpha^2m\frac{1}{1-\lambda} 4L_K^2	\alpha - \frac{L_{\ell}\alpha^2}{2}- 4(1+\frac{1}{r}) L_y^2\alpha^2m  \frac{1-\lambda}{32(1+1/r)L_y^2 } -8\alpha^3 L_K^2m\notag\\&\qquad-L_K^2[\frac{12\beta (\mu_g+L_g) }{ \mu_g L_g}+(\frac{1}{\delta}) (1+r)^2 (2\beta \frac{1}{\mu_g+L_g } )] 4\alpha^2m  \frac{1-\lambda}{32(1+1/r)L_y^2 } ),\notag\\
		\overset{(a)}{\leq}&	-(\frac{\alpha}{2} -4\alpha^2m\frac{1}{1-\lambda} 4L_K^2	\alpha - \frac{L_{\ell}\alpha^2}{2}- 4(1+\frac{1}{r}) L_y^2\alpha^2m  \frac{1-\lambda}{32(1+1/r)L_y^2 } -8\alpha^3 L_K^2m\notag\\&\qquad-(\frac{(1-\lambda)}{64}+\frac{(1-\lambda)}{64})4\alpha^2m  \frac{1-\lambda}{32(1+1/r)L_y^2 } ),\notag\\
	\overset{(b)}{\leq}&-\frac{\alpha}{2} +\frac{\alpha}{16}+\frac{\alpha}{16}+\frac{\alpha}{8} +\frac{\alpha}{8}+\frac{\alpha}{8}= 0,
\end{align}	

where   (a) follows from $\beta \leq \min\{\frac{3 (\mu_g+L_g) }{\mu_gL_g},  \frac{(1-\lambda)\mu_g L_g}{768 L_K^2 (\mu_g+L_g)} , \frac{(1-\lambda) (\mu_g+L_g)}{4096 L_K^2 }  \},  \delta =1/8,  r=\frac{1}{3} \beta \frac{\mu_g L_g}{\mu_g+L_g }$.
(b) follows from $\alpha \leq \min\{ \frac{1}{16 L_K} \sqrt{\frac{1-\lambda}{m}}$,

$ \frac{1}{8L_{\ell}},    \frac{r}{16m L_y^2(r+1)},\frac{1}{8L_K\sqrt{m}} ,\frac{1}{m(1-\lambda)}\}$.

To ensure $C_2'\leq-\frac{1-\lambda}{2}$, we have
\begin{align}
	C_2'&=\big[\frac{1}{1-\lambda} 32L_K^2	\alpha +\lambda-1+16\alpha L_K^2+ \frac{\alpha L_{\ell}  }{m} +\frac{(1-\lambda)}{4} + \frac{1-\lambda}{4(1+1/r)L_y^2 }[\frac{12\beta (\mu_g+L_g) }{ \mu_g L_g}+(\frac{1}{\delta}) (1+r)^2 (2\beta \frac{1}{\mu_g+L_g } )]L_K^2 \big]\notag\\
&	\leq \big[\frac{1}{1-\lambda} 32L_K^2	\alpha +\lambda-1+16\alpha L_K^2+ \frac{\alpha L_{\ell}  }{m} +\frac{(1-\lambda)}{4} + \frac{1-\lambda}{4L_y^2 }[\frac{12\beta (\mu_g+L_g) }{ \mu_g L_g}+(\frac{1}{\delta}) 4 (2\beta \frac{1}{\mu_g+L_g } )]L_K^2 \big]\notag\\
	&\overset{(a)}{\leq} \frac{1-\lambda}{4} -(1-\lambda)+ \frac{1-\lambda}{4}+ \frac{1-\lambda}{8}+ \frac{1-\lambda}{8}\leq -\frac{1-\lambda}{4},
\end{align}	
where $(a)$ follows from $\alpha \leq \min\{ \frac{(1-\lambda)^2}{128L_K^2 } , \frac{  (1-\lambda)}{4}(\frac{m}{L_{\ell}+16L_K^2m})\}$, $\beta\leq \min\{\frac{ L_y^2 \mu_g L_g}{24 L_K^2(\mu_g+L_g) }, \frac{(1-\lambda)(\mu_g+L_g)}{512L_K^2}  \}$.

To ensure $C_3'\leq- \frac{3  r^2(1-\lambda)}{32(1+r)L_y^2 }$, we have
\begin{align}
	C_3'&=[\frac{2\alpha}{m}  L_f^2  +   \frac{1-\lambda}{32(1+1/r)L_y^2 }3r -  \frac{1-\lambda}{32(1+1/r)L_y^2 }(1+3r)( 2\beta \frac{\mu_g L_g}{\mu_g+L_g }) ]\notag\\&\overset{(a)}{\leq} 3r \frac{1-\lambda}{32(1+1/r)L_y^2 }( 2\beta \frac{\mu_g L_g}{\mu_g+L_g })+\frac{1-\lambda}{32(1+1/r)L_y^2 }( \beta \frac{\mu_g L_g}{\mu_g+L_g })-\frac{1-\lambda}{32(1+1/r)L_y^2 }(1+3r)( 2\beta \frac{\mu_g L_g}{\mu_g+L_g })\notag\\ \leq& -\frac{1-\lambda}{32(1+r)L_y^2 }\beta r \frac{\mu_g L_g}{\mu_g+L_g } =-\frac{3  r^2(1-\lambda)}{32(1+r)L_y^2 } ,
\end{align}	
where $(a)$ follows from $\alpha \leq \frac{32(1+1/r)L_y^2 }{1-\lambda} \frac{3rm}{ L_f^2} ( \beta \frac{\mu_g L_g}{\mu_g+L_g })= \frac{288 r (1+r) m L_y^2 }{(1-\lambda)L_f^2} , \beta\leq \frac{3(\mu_g+L_g)}{ \mu_g L_g}, r= \frac{1}{3}\beta \frac{\mu_g L_g}{\mu_g+L_g } $.

To ensure $C_4'\leq0$, we have
\begin{align}
	C_4'&=\big[2\alpha L_K^2\beta^2+\frac{1}{1-\lambda}4L_K^2	\alpha \beta^2+(\delta-1)(1+r)^2 (2\beta \frac{1}{\mu_g+L_g } -\beta^2) \frac{1-\lambda}{32(1+1/r)L_y^2 }\notag\\&\qquad +\beta^2\cdot \frac{1-\lambda}{32(1+1/r)L_y^2 } L_K^2  [\frac{12\beta (\mu_g+L_g) }{ \mu_g L_g}+(\frac{1}{\delta}) (1+r)^2 (2\beta \frac{1}{\mu_g+L_g } )] \big]\notag\\
	&\overset{(a)}{\leq}  \big[2\alpha L_K^2\beta^2+\frac{1}{1-\lambda}4L_K^2	\alpha \beta^2+(\delta-1)(1+r)^2 (2\beta \frac{1}{\mu_g+L_g } -\beta^2)\frac{1-\lambda}{32(1+1/r)L_y^2 }+\beta^2\cdot \frac{(1-\lambda)}{32}\frac{1-\lambda}{32(1+1/r)L_y^2 }  \big] \notag\\
&	\leq\big[2\alpha L_K^2\beta^2+\frac{1}{1-\lambda}4L_K^2	\alpha \beta^2+(\delta)(1+r)^2 (2\beta \frac{1}{\mu_g+L_g } -\beta^2)\frac{1-\lambda}{32(1+1/r)L_y^2 }\notag\\
&\qquad-(1+r)^2 (2\beta \frac{1}{\mu_g+L_g } )\frac{1-\lambda}{32(1+1/r)L_y^2 }+(1+r)^2 (\beta^2)\frac{1-\lambda}{32(1+1/r)L_y^2 }+\beta^2\cdot \frac{(1-\lambda)}{32}  \frac{1-\lambda}{32(1+1/r)L_y^2 } \big]\notag\\
&\overset{(b)}{\leq} \frac{1-\lambda}{32(1+1/r)L_y^2 } [(1+r)^2\beta \frac{1}{4(\mu_g+L_g) }+(1+r)^2\beta \frac{1}{4(\mu_g+L_g) } +(1+r)^2\beta \frac{1}{4(\mu_g+L_g) } \notag\\&\qquad - (1+r)^2 2\beta \frac{1}{\mu_g+L_g } + (1+r)^2\beta \frac{1}{2(\mu_g+L_g) } + \beta \frac{1}{2(\mu_g+L_g) } ] <0,
\end{align}	
where (a) follows from $\beta \leq \min\{\frac{3 (\mu_g+L_g) }{\mu_gL_g},  \frac{(1-\lambda)\mu_g L_g}{768 L_K^2 (\mu_g+L_g)} , \frac{(1-\lambda) (\mu_g+L_g)}{4096 L_K^2 }  \},  \delta =1/8,  r=\frac{1}{3} \beta \frac{\mu_g L_g}{\mu_g+L_g }$, $(b)$ follows from $\alpha \leq\min\{  {\frac{r (1+r) (1-\lambda) }{256L_y^2 (\mu_g+L_g) L_K^2 \beta }}$,

$\frac{r(1+r)(1-\lambda)^2}{512L_y^2 (\mu_g+L_g)(L_K^2\beta)} \}, \beta\leq\min\{  \frac{1}{2(\mu_g+L_g)},  \frac{16}{(1-\lambda)(\mu_g+L_g)} \}, \delta =1/8$.

To ensure $C_5'\leq0$, we have
\begin{align}
	C_5'&=\big[4\alpha^2 \frac{1}{1-\lambda}4L_K^2 	\alpha + \frac{1}{1-\lambda} \alpha^2+ 4(1+\frac{1}{r}) L_y^2\alpha^2  \frac{1-\lambda}{32(1+1/r)L_y^2 }  +	\alpha(\lambda +8\alpha^2 L_K^2-1)\notag\\&\qquad+4\alpha^2 \cdot  \frac{1-\lambda}{32(1+1/r)L_y^2 }  L_K^2  [\frac{12\beta (\mu_g+L_g) }{ \mu_g L_g}+(\frac{1}{\delta}) (1+r)^2 (2\beta \frac{1}{\mu_g+L_g } )]\big]
	\notag\\
	&\overset{(a)}{\leq} \big[4\alpha^2 \frac{1}{1-\lambda}4L_K^2 	\alpha + \frac{1}{1-\lambda} \alpha^2+ 4(1+\frac{1}{r}) L_y^2\alpha^2  \frac{1-\lambda}{32(1+1/r)L_y^2 }+	\alpha(\lambda -1)+8\alpha^3 L_K^2+4\alpha^2 \frac{1-\lambda}{32(1+1/r)L_y^2 } \cdot \frac{(1-\lambda)}{32}\big]
	\notag\\
		&< \big[4\alpha^2 \frac{1}{1-\lambda}4L_K^2 	\alpha + \frac{1}{1-\lambda} \alpha^2+ 4(1+\frac{1}{r}) L_y^2\alpha^2  \frac{1-\lambda}{32L_y^2 }+	\alpha(\lambda -1)+8\alpha^3 L_K^2+4\alpha^2 \frac{1-\lambda}{32L_y^2 } \cdot \frac{(1-\lambda)}{32}\big]
	\notag\\
	&\overset{(b)}{\leq} 
	\frac{\alpha(1-\lambda )}{4}+\frac{\alpha(1-\lambda )}{4}+\frac{\alpha(1-\lambda )}{4}- \alpha(1-\lambda )+\frac{\alpha(1-\lambda )}{8}+\frac{\alpha(1-\lambda )}{8}= 0
\end{align}	
where (a) follows from $\beta \leq \min\{\frac{3 (\mu_g+L_g) }{\mu_gL_g},  \frac{(1-\lambda)\mu_g L_g}{768 L_K^2 (\mu_g+L_g)} , \frac{(1-\lambda) (\mu_g+L_g)}{4096 L_K^2 }  \},  \delta =1/8,  r=\frac{1}{3} \beta \frac{\mu_g L_g}{\mu_g+L_g }$, $(b)$ follows from  $ \alpha \leq \min\{ \frac{\sqrt{ 1-\lambda}}{8L_K},     \frac{(1-\lambda)^2}{4}$,

$ \frac{32L_y^2 }{16(1+\frac{1}{r})L_y^2},\sqrt{\frac{1-\lambda}{64L_K^2} }, \frac{32L_y^2}{(1-\lambda)}\}
$.

With the above conditions, we have
\begin{align}
	&
	\ell(\bx_{T+1}) - \ell(\bx_{0}) + \big[\|\by_{T+1} - \y_{T+1}^*\|^2  - \|\y_0^* - \by_0\|^2\big] \notag\\&+	[\|\x_{T+1}-\ot\bx_{T+1}\|^2-	\|\x_{0}-\ot\bx_{0}\|^2 ]+
	\alpha [	\|\u_{T+1}-\ot\bu_{T+1}\|^2  -	\|\u_{0}-\ot\bu_{0}\|^2 ]
	\notag\\
	\le
	&-\frac{\alpha}{2} \sum_{t=0}^{T} \|\nabla \ell(\bx_t)\|^2-\frac{1-\lambda}{4}\sum_{t=0}^{T} \| \x_{t} - \ot\bx_{t}  \|^2 - \frac{3  r^2(1-\lambda)}{32(1+r)L_y^2 } \sum_{t=0}^{T}\|\y_{t} - \y_{t}^*\|^2 + 2\alpha  ( \frac{C_{g_{x y}} C_{f_{y}}}{\mu_{g}}\left(1-\frac{\mu_{g}}{L_{g}}\right)^{K})^2 (T+1).
\end{align}

Thus,  we have the following convergence results:
\begin{align}
	\frac{1}{T+1}\sum_{t=0}^{T} \|\mathfrak{M}_t\|^2 \leq  	\frac{1}{(T+1)\min\{\frac{1-\lambda}{4},\frac{3  r^2(1-\lambda)}{32(1+r)L_y^2 },\frac{\alpha}{2}\}}[ \mathfrak{B}_0- \ell^*]+ C_{bias},
\end{align}
where $ \mathfrak{B}_t=	\ell(\bx_{t}) + \|\by_{t} - \y_{t}^*\|^2  +	\|\x_{t}-\ot\bx_{t}\|^2 +
\alpha 	\|\u_{t}-\ot\bu_{t}\|^2  , C_{bias}=\frac{2\alpha  ( \frac{C_{g_{x y}} C_{f_{y}}}{\mu_{g}}\left(1-\frac{\mu_{g}}{L_{g}}\right)^{K})^2}{\min\{\frac{1-\lambda}{4}, \frac{3  r^2(1-\lambda)}{32(1+r)L_y^2 },\frac{\alpha}{2}\}}$.

\textbf{Proof of Lemma \ref{lem:deterministic}:}

\begin{align} 
	& \left\|{\nabla} f\left( \ \x_{1}, \y_1 \right)-{\nabla} f\left( \ \x_{2}, \y_2 \right)\right\|^{2} \notag\\
\leq& 2\left\|\nabla_{\x} f\left( \ \x_{1}, \y_1 \right)-\nabla_{\x} f\left( \x_{2}, \y_2 \right)\right\|^{2} \notag\\& + 2\left\|\nabla_{\x\y}^2 g_{i}(\x_{1}, \y_1  ) \times 
	[ \nabla_{\y\y}^2 g_{i}(\x_{1}, \y_1  ) ]^{-1}   \nabla_{\y} f_{i}(\x_{1}, \y_1 )-
	\nabla_{\x\y}^2 g_{i}(\x_{2}, \y_2 ) \times 
	[ \nabla_{\y\y}^2 g_{i}(\x_{2}, \y_2) ]^{-1}   \nabla_{\y} f_{i}(\x_{2}, \y_2)\right\|^{2} \notag\\
\leq &2 L_{f_{x}}^{2}( \left\| \x_{1}- \x_{2}\right\|^{2} +\left\| \y_{1}- \y_{2}\right\|^{2}) +6C_{g y}^{2} \frac{1}{\mu_g^2} L_{f_{y}}^{2}( \left\| \x_{1}- \x_{2}\right\|^{2} +\left\| \y_{1}- \y_{2}\right\|^{2})+6 C_{f_{y}}^{2}\frac{1}{\mu_g^2}  L_{g_{x y}}^{2}( \left\| \x_{1}- \x_{2}\right\|^{2} +\left\| \y_{1}- \y_{2}\right\|^{2}) \notag\\&+6 C_{g_{x y}}^{2} C_{f_{y}}^{2}  \| 	[ \nabla_{\y\y}^2 g_{i}(\x_{1}, \y_1) ]^{-1}      \|^2 \cdot \| 	[ \nabla_{\y\y}^2 g_{i}(\x_{2}, \y_2) ] -	[ \nabla_{\y\y}^2 g_{i}(\x_{1}, \y_1) ]    \|^2 \cdot \| 	[ \nabla_{\y\y}^2 g_{i}(\x_{2}, \y_2) ]^{-1}      \|^2 \notag\\
\leq &2 L_{f_{x}}^{2}( \left\| \x_{1}- \x_{2}\right\|^{2} +\left\| \y_{1}- \y_{2}\right\|^{2}) +6C_{g y}^{2} \frac{1}{\mu_g^2} L_{f_{y}}^{2}( \left\| \x_{1}- \x_{2}\right\|^{2} +\left\| \y_{1}- \y_{2}\right\|^{2})+6 C_{f_{y}}^{2}\frac{1}{\mu_g^2}  L_{g_{x y}}^{2}( \left\| \x_{1}- \x_{2}\right\|^{2} +\left\| \y_{1}- \y_{2}\right\|^{2}) \notag\\&+6 C_{g_{x y}}^{2} C_{f_{y}}^{2}  \frac{1}{\mu_g^4} \cdot \| 	[ \nabla_{\y\y}^2 g_{i}(\x_{2}, \y_2) ] -	[ \nabla_{\y\y}^2 g_{i}(\x_{1}, \y_1) ]    \|^2 \notag\\
\leq &2 L_{f_{x}}^{2}( \left\| \x_{1}- \x_{2}\right\|^{2} +\left\| \y_{1}- \y_{2}\right\|^{2}) +6C_{g y}^{2} \frac{1}{\mu_g^2} L_{f_{y}}^{2}( \left\| \x_{1}- \x_{2}\right\|^{2} +\left\| \y_{1}- \y_{2}\right\|^{2})+6 C_{f_{y}}^{2}\frac{1}{\mu_g^2}  L_{g_{x y}}^{2}( \left\| \x_{1}- \x_{2}\right\|^{2} +\left\| \y_{1}- \y_{2}\right\|^{2}) \notag\\&+6 C_{g_{x y}}^{2} C_{f_{y}}^{2} L_{g_{y y}}^{2}\frac{1}{\mu_g^4} ( \left\| \x_{1}- \x_{2}\right\|^{2} +\left\| \y_{1}- \y_{2}\right\|^{2}).
\end{align}
Thus, we have 
\begin{align}
	L_{K_d}^{2}:=&2 L_{f_{x}}^{2}+6C_{g y}^{2} \frac{1}{\mu_g^2} L_{f_{y}}^{2} +6 C_{f_{y}}^{2}\frac{1}{\mu_g^2}  L_{g_{x y}}^{2} +6 C_{g_{x y}}^{2} C_{f_{y}}^{2} L_{g_{y y}}^{2}\frac{1}{\mu_g^4} .
\end{align}

\textbf{Proof of Lemma \ref{lem:stocahstic}:}
\begin{align} \label{88}
	&\mathbb{E}_k\left\|{\nabla} f\left( \ \x_{1}, \y_1 ; \bar{\xi}\right)-{\nabla} f\left( \ \x_{2}, \y_2 ; \bar{\xi}\right)\right\|^{2} \notag\\
	\stackrel{(a)}{\leq}& 2\left\|\nabla_{\x} f\left( \ \x_{1}, \y_1 ;\xi_i^0\right)-\nabla_{\x} f\left( \x_{2}, \y_2 ;\xi_i^0\right)\right\|^{2} +2 \| \nabla_{\x \y}^{2} g\left( \x_{1}, \y_1 ; \zeta_i^0\right)\left[\frac{1}{L_{g}}{K \prod_{p=1}^{k(K)}\left(\mathbf{I}-\frac{\nabla_{\mathbf{y y}}^{2} g\left(\mathbf{x}_{i, t}, \mathbf{y}_{i, t} ; \zeta_{i}^{p}\right)}{L_{g}}\right)}\right] \nabla_{\y} f\left( \x_{1}, \y_1 ;\xi_i^0\right) \notag\\
	&\quad-\nabla_{\x \y}^{2} g\left( \x_{2}, \y_2 ; \zeta_i^0\right)\left[\frac{1}{L_{g}} {K \prod_{p=1}^{k(K)}\left(\mathbf{I}-\frac{\nabla_{\mathbf{y y}}^{2} g\left(\mathbf{x}_{i, t}, \mathbf{y}_{i, t} ; \zeta_{i}^{p}\right)}{L_{g}}\right)} \right] \nabla_{\y} f\left( \x_{2}, \y_2 ;\xi_i^0\right) \|^{2} \notag\\
	\stackrel{(b)}{\leq} &2 L_{f_{x}}^{2}( \left\| \x_{1}- \x_{2}\right\|^{2} +\left\| \y_{1}- \y_{2}\right\|^{2}) +2 \| \nabla_{\x \y}^{2} g\left( \x_{1}, \y_1 ; \zeta_i^0\right)\left[\frac{1}{L_{g}} {K \prod_{p=1}^{k(K)}\left(\mathbf{I}-\frac{\nabla_{\mathbf{y y}}^{2} g\left(\mathbf{x}_{i, t}, \mathbf{y}_{i, t} ; \zeta_{i}^{p}\right)}{L_{g}}\right)}\right] \nabla_{\y} f\left( \x_{1}, \y_1 ;\xi_i^0\right) \notag\\
	&-\nabla_{\x \y}^{2} g\left( \x_{2}, \y_2 ; \zeta_i^0\right)\left[\frac{1}{L_{g}} {K \prod_{p=1}^{k(K)}\left(\mathbf{I}-\frac{\nabla_{\mathbf{y y}}^{2} g\left(\mathbf{x}_{i, t}, \mathbf{y}_{i, t} ; \zeta_{i}^{p}\right)}{L_{g}}\right)}\right] \nabla_{\y} f\left( \x_{2}, \y_2 ;\xi_i^0\right)\|^{2},\notag\\
\stackrel{(c)}{\leq} &2 L_{f_{x}}^{2}( \left\| \x_{1}- \x_{2}\right\|^{2} +\left\| \y_{1}- \y_{2}\right\|^{2}) +2( 3 C_{g y}^{2} \frac{K^{2}}{L_{g}^{2}}  \frac{1}{K}\left(\frac{L_{g}^{2}}{2 \mu_{g} L_{g}-\mu_{g}^{2}}\right) L_{f_{y}}^{2}( \left\| \x_{1}- \x_{2}\right\|^{2} +\left\| \y_{1}- \y_{2}\right\|^{2})\notag\\&+3 C_{f_{y}}^{2}\frac{K^{2}}{L_{g}^{2}} \frac{1}{K}\left(\frac{L_{g}^{2}}{2 \mu_{g} L_{g}-\mu_{g}^{2}}\right) L_{g_{x y}}^{2}( \left\| \x_{1}- \x_{2}\right\|^{2} +\left\| \y_{1}- \y_{2}\right\|^{2}) \notag\\
	&+3 C_{g_{x y}}^{2} C_{f_{y}}^{2}\frac{K^2}{L_g^2}\left\|\prod_{p=1}^{{k}(K)}\left(I-\frac{1}{L_{g}} \nabla_{\y \y}^{2} g_i\left( \x_{1}, \y_1 ; \zeta_i^p\right)\right)-\prod_{p=1}^{{k}(K)}\left(I-\frac{1}{L_{g}} \nabla_{\y \y}^{2} g_i\left( \x_{1}, \y_1 ; \zeta_i^p\right)\right)\right\|^{2})\notag\\
\stackrel{(d)}{\leq}  &2 L_{f_{x}}^{2}( \left\| \x_{1}- \x_{2}\right\|^{2} +\left\| \y_{1}- \y_{2}\right\|^{2}) +6 C_{g y}^{2} \frac{K^{2}}{L_{g}^{2}}  \frac{1}{K}\left(\frac{L_{g}^{2}}{2 \mu_{g} L_{g}-\mu_{g}^{2}}\right) L_{f_{y}}^{2}( \left\| \x_{1}- \x_{2}\right\|^{2} +\left\| \y_{1}- \y_{2}\right\|^{2})\notag\\&+6 C_{f_{y}}^{2}\frac{K^{2}}{L_{g}^{2}} \frac{1}{K}\left(\frac{L_{g}^{2}}{2 \mu_{g} L_{g}-\mu_{g}^{2}}\right) L_{g_{x y}}^{2}( \left\| \x_{1}- \x_{2}\right\|^{2} +\left\| \y_{1}- \y_{2}\right\|^{2}) \notag\\
	&+6 C_{g_{x y}}^{2} C_{f_{y}}^{2} \frac{K^2}{L_{g}^{2}} k(K) \sum_{p=1}^{k(K)} \left(1-\frac{\mu_{g}}{L_{g}}\right)^{2 (k(K)-1)}  \frac{1}{L_g^2}\left\|\nabla_{\y \y}^{2} g_i\left(\x_{1}, \y_1 ; \zeta_i^p\right)-\nabla_{\y \y}^{2} g_i\left(x_{2}, \y ; \zeta_i^p\right)\right\|^{2}\notag\\
\stackrel{(e)}{\leq}  &2 L_{f_{x}}^{2}( \left\| \x_{1}- \x_{2}\right\|^{2} +\left\| \y_{1}- \y_{2}\right\|^{2}) +6 C_{g y}^{2} \frac{K^{2}}{L_{g}^{2}}  \frac{1}{K}\left(\frac{L_{g}^{2}}{2 \mu_{g} L_{g}-\mu_{g}^{2}}\right) L_{f_{y}}^{2}( \left\| \x_{1}- \x_{2}\right\|^{2} +\left\| \y_{1}- \y_{2}\right\|^{2})\notag\\&+6 C_{f_{y}}^{2}\frac{K^{2}}{L_{g}^{2}} \frac{1}{K}\left(\frac{L_{g}^{2}}{2 \mu_{g} L_{g}-\mu_{g}^{2}}\right) L_{g_{x y}}^{2}( \left\| \x_{1}- \x_{2}\right\|^{2} +\left\| \y_{1}- \y_{2}\right\|^{2}) \notag\\
	&+6 C_{g_{x y}}^{2} C_{f_{y}}^{2} \frac{K^2}{L_{g}^{2}} k(K)^2 \left(1-\frac{\mu_{g}}{L_{g}}\right)^{2 (k(K)-1)}  \frac{1}{L_g^2}L_{g_{yy}}^{2}( \left\| \x_{1}- \x_{2}\right\|^{2} +\left\| \y_{1}- \y_{2}\right\|^{2}),
\end{align}
where (a) follows from triangle inequality and the definition of ${\nabla} f\left( \ \x, \y ; \bar{\xi}\right)$, (b) and (e) are follow from the gradient Liptichz assumption, (c) and (d) are follow from the triangle inequality and the Lemma A.1 in \cite{khanduri2021near}. 
Since we are aiming at finding constant $L_{K}$ which satisfied the Liptichz inequality for all $k$, eqs.\eqref{88} need to be hold with the  maximum value of $k(K)^2 \left(1-\frac{\mu_{g}}{L_{g}}\right)^{2 (k(K)-1)}  $. 

Thus, we have 
\begin{align}
	L_{K_s}^{2}:=&2 L_{f_{x}}^{2}+6 C_{g_{x y}}^{2} L_{f_{y}}^{2}\left(\frac{K}{2 \mu_{g} L_{g}-\mu_{g}^{2}}\right)+6 C_{f_{y}}^{2} L_{g_{x y}}^{2}\left(\frac{K}{2 \mu_{g} L_{g}-\mu_{g}^{2}}\right)\notag\\&+ 6 C_{g_{x y}}^{2} C_{f_{y}}^{2} \frac{K^2}{L_{g}^{2}} \max_{k(K)}\{k(K)^2 \left(1-\frac{\mu_{g}}{L_{g}}\right)^{2 (k(K)-1)} \} \frac{1}{L_g^2}L_{g_{yy}}^{2}\notag\\
	\leq& 2 L_{f_{x}}^{2}++6 C_{g_{x y}}^{2} L_{f_{y}}^{2}\left(\frac{K}{2 \mu_{g} L_{g}-\mu_{g}^{2}}\right)+6 C_{f_{y}}^{2} L_{g_{x y}}^{2}\left(\frac{K}{2 \mu_{g} L_{g}-\mu_{g}^{2}}\right)\notag\\&+ 6 C_{g_{x y}}^{2} C_{f_{y}}^{2} \frac{K^4}{L_{g}^{2}} \frac{1}{L_g^2}L_{g_{yy}}^{2}.
\end{align}

\end{document}